\documentclass{article}

\usepackage{microtype}
\usepackage{graphicx}
\usepackage{xcolor}
\usepackage{subfigure}
\usepackage{booktabs} % for professional tables
\usepackage{enumitem}
\usepackage{amsthm}
\usepackage{thm-restate}

\usepackage{hyperref}

\usepackage[style=alphabetic,backend=bibtex,maxbibnames=20,maxcitenames=6,giveninits=true,doi=false,url=false]{biblatex}
\newcommand*{\citet}[1]{\AtNextCite{\AtEachCitekey{\defcounter{maxnames}{2}}} \textcite{#1}}

\newcommand*{\citep}[1]{\cite{#1}}
\newcommand{\citeyearpar}[1]{\cite{#1}}
\bibliography{references}

\usepackage{fullpage}
\usepackage{multicol}

\usepackage{amsmath}
\usepackage{amsfonts}
\usepackage{amssymb}
\usepackage{bbm}

\allowdisplaybreaks

\newcommand{\lp}{\left}
\newcommand{\rp}{\right}

\newcommand{\me}{\mathrm{e}}
\renewcommand{\phi}{\varphi}

\newcommand{\Var}{\mathrm{Var}}
\newcommand{\R}{\mathbb{R}}
\newcommand{\E}{\mathbb{E}}

\def \Exp{{\mathbb{E}}}
\def \Pr{{\mathbb{P}}}
\def \R{{\mathbb{R}}}
\def \S{{\mathbb{S}}}

\newcommand{\va}{\mathbf{a}}
\newcommand\numberthis{\addtocounter{equation}{1}\tag{\theequation}}
\newcommand{\vtanh}{\mathbf{tanh}}
\newcommand{\fudge}{\xi}

\theoremstyle{plain}
\newtheorem{theorem}{Theorem}
\newtheorem*{theorem*}{Theorem}
\newtheorem{lemma}{Lemma}

\newtheorem{definition}{Definition}

\newtheorem{corollary}{Corollary}
\theoremstyle{definition}
\newtheorem{remark}{Remark}

\newtheorem{setting}{Setting}
\newtheorem{assumptions}{Assumptions}

\usepackage{babel}

\providetoggle{blx@lang@captions@english}

\begin{document}

\title{Statistical Estimation from Dependent Data}

\author{
    Yuval Dagan 
    %\thanks{TODO}\\
    \\EECS \& CSAIL, MIT\\
    \tt{dagan@mit.edu}
    \and
	Constantinos Daskalakis
	%\thanks{Supported by NSF Awards IIS-1741137, CCF-1617730 and CCF-1901292, by a Simons Investigator Award, by the DOE PhILMs project (No. DE-AC05-76RL01830), and by the DARPA award HR00111990021.}\\
	\\EECS \& CSAIL, MIT\\
	\tt{costis@csail.mit.edu}
	\and
	Nishanth Dikkala
	%\thanks{TODO}\\
	\\GOOGLE RESEARCH \\
	\tt{nishanthd@google.com}
	\and
	Surbhi Goel
	\\ Microsoft Research NYC\\
	\tt{surbgoel@microsoft.com}
	\and
	Anthimos Vardis Kandiros
	%\thanks{TODO}\\
	\\EECS \& CSAIL, MIT\\
	\tt{kandiros@mit.edu}
	}

\maketitle
\begin{abstract}
%attempt 4
We consider a general statistical estimation problem wherein binary labels across different observations are not independent conditioned on their feature vectors, but dependent, capturing  settings where e.g.~these observations are collected on a spatial domain, a temporal domain, or a social network, which induce dependencies. We model these dependencies in the language of Markov Random Fields and, importantly, allow these dependencies to be substantial, i.e.~do not assume that the Markov Random Field capturing these dependencies is in high temperature. As our main contribution we provide algorithms and statistically efficient estimation rates for this model,  giving several instantiations of our bounds in logistic regression, sparse logistic regression, and neural network settings with dependent data. Our estimation guarantees follow from novel results for estimating the parameters (i.e.~external fields and interaction strengths) of Ising models from a {\em single} sample. {We evaluate our estimation approach on real networked data, showing that it outperforms standard regression approaches that ignore dependencies, across three text classification datasets: Cora, Citeseer and Pubmed.}

\end{abstract}

% We consider the problem of logistic and non-linear regression. While it is common to assume that the observations are independent, this may be too restrictive. For instance, the observations may arise from a social network with peer effects. (+other examples)

% - Define the model
\section{Introduction} \label{sec:intro}

The standard supervised learning framework assumes access to a collection $(x_i,y_i)_{i=1}^n$ of observations, where the labels $y_1,\ldots,y_n \in {\cal Y}$ are  independent conditioning on the feature vectors $x_1,\ldots,x_n \in {\cal X}$. 
%A common goal is to use these observations to identify a model which, in some precise sense, estimates the distribution $\Pr[\vec{y} \mid \vec{x}]$. 
Further, it is common to assume that each label $y_i$ is  independent of $\{x_j\}_{j \ne i}$ conditioning on $x_i$,  i.e.~that
% the conditional distribution of $\vec{y}$ decomposes as 
$$\Pr[y_{1\ldots n} \mid x_{1\dots n}] = \prod_{i=1}^n \Pr[y_i\mid x_i],$$
and, moreover, that the observations share the same generative process $\Pr[y \mid x]$  sampling a label conditioning on a feature vector. Under these assumptions,
a common goal is to identify a model $\Pr_{\theta}[y \mid x]$ from some parametric class, which approximates the true generative process $\Pr[y \mid x]$ in some precise sense, or, under realizability assumptions, to estimate the  parameter $\theta$ of the true generative process. A special case of this problem is the familiar logistic regression problem, where each label lies in ${\cal Y}=\{\pm 1\}$, each feature vector lies in $\mathbb{R}^d$ and for some $\theta \in  \mathbb{R}^d$ it is assumed that
\begin{align}\Pr[y_{1\ldots n} \mid x_{1\dots n}] = \prod_{i=1}^n \frac{1}{1+\exp(-2(\theta^{\top}x_i) y_i)}. \label{eq:logistic regression model}
\end{align}

The standard assumptions outlined above are, however, too strong and almost never truly hold in practice. Indeed, they become especially prominent when it comes to observations collected in a temporal domain, a spatial domain or a social network, which naturally induce dependencies among the observations. 
% \nishanth{perhaps we can say the assumptions are too strong and almost never hold truly in practice. so it appears more general than just for specific types of data?} 
Such dependencies could arise from physical constraints, causal relationships among observations, or peer effects in a social network. They have been studied extensively in many practical fields, and from a theoretical standpoint in econometrics and statistical learning theory. See section~\ref{sec:related} for further discussion. 

In this paper we study such  dependencies  conforming to the following general class of models: 

\begin{equation}
\Pr[y_{1\ldots n} \mid x_{1\dots n}] \propto \exp(-\beta \cdot H(\vec{y})) \cdot \prod_{i=1}^n\exp(f_\theta(x_i,y_i))
\equiv
\exp\left(-\beta \cdot H(\vec{y}) + \sum_{i=1}^n f_\theta(x_i,y_i)\right) , \label{eq:very general}
\end{equation}

where $f_\theta$ is an (unknown) function from some parametric class, $H$ is a (known) function that captures the dependency structure and $\beta$ is an (unknown) parameter that captures the strengths of dependencies. It should be appreciated that Model~\eqref{eq:very general} is more general than the standard supervised learning problem without dependencies, which results from setting $\beta=0$. Once we allow $\beta\ne 0$, Model~\eqref{eq:very general} becomes more expressive in capturing the dependencies among the observations, which become stronger with higher values of $\beta$. The challenging estimation problem that arises, which motivates our work, is whether the model parameters $\theta$ and/or $\beta$ can  be identified, and at what rates, in the presence of the intricate dependencies arising from this model. Importantly, while the labels are intricately dependent, we do not have access to multiple independent samples from the conditional distribution~\eqref{eq:very general},  but a {\em single} sample from that distribution!

We focus here on a special case of Model~\eqref{eq:very general} wherein the labels are binary and the function $H$ is pairwise separable, studying models of the following form:

\begin{equation}
\Pr_{\theta,\beta}[y_{1\ldots n} \mid x_{1\dots n}] = \frac{\exp(\beta y^{\rm T} A y) \prod_{i=1}^n\exp(y_if_\theta(x_i))}{Z_{\theta,\beta}}\\ \equiv
\frac{1}{Z_{\theta,\beta}}\exp\left(\beta y^{\rm T} A y + \sum_i y_i f_\theta(x_i)  \right), \label{eq:general binary}
\end{equation}

where $f_\theta$ is an unknown function from some parametric class $\mathcal{F} = \{f_\theta: {\cal X}\rightarrow \mathbb{R}~|~ \theta \in \Theta\}$, $A$ is a known, symmetric \emph{interaction matrix} with zeros on the diagonal, $\beta$ is an unknown parameter, and $Z_{\theta,\beta}$ is the normalizing constant.
In other words, under~\eqref{eq:general binary}, the labels $y_1,\ldots,y_n$ are sampled from an $n$-variable \emph{Ising model} with \emph{external field} $f_\theta(x_i)$ on variable~$i$, \emph{interaction strength} $A_{ij}\equiv A_{ji}$ between variables $i$ and~$j$, and \emph{inverse temperature} $\beta$. 
Notice that $A_{ij}$ encourages $y_i$ and $y_j$ to have the same or opposite values, depending on its sign, however, this ``local encouragement'' can be overwritten by indirect interactions through other values of $y_k$. Such indirect interactions make this model rich in spite of the simple form of $H(y) = y^\top A y$ and as a consequence, it has found profound applications in a range of disciplines, including Statistical Physics, Computer Vision, Computational Biology, and the Social Sciences; see e.g.~\cite{GemanG86,Ellison93,Felsenstein04,chatterjee2005concentration,DaskalakisMR11,daskalakis2017concentration}.

It is clear that Model~\eqref{eq:general binary} generalizes~\eqref{eq:logistic regression model}, which can be obtained by setting $\beta=0$, ${\cal X}=\mathbb{R}^d$, and $f_\theta(x)=\theta^{\rm T}x$. It also generalizes the model studied by~\citet{daskalakis2019regression}, which results from setting $f_\theta(x)=\theta^{\rm T}x$ and $0 \le \beta \le O(1)$, as well as the model studied by~\citet{ghosal2018joint,bhattacharya2018inference,chatterjee2007estimation}, which results from taking $f_\theta(x)$ to be a constant function.

We study under what conditions on the function class $\mathcal{F}$, the interaction matrix $A$, and the feature vectors $(x_i)_i$, and at what rates can the parameters $\theta$ and/or $\beta$ of Model~\eqref{eq:general binary} be estimated given a collection $(x_i,y_i)_{i=1}^n$ of observations, where the labels $y_1,\ldots,y_n$ are sampled from~\eqref{eq:general binary} conditioning on the feature vectors $x_1,\ldots,x_n$. As explained earlier, in comparison to the standard supervised learning setting without dependencies, the statistical challenge that arises here is that, while the labels $y_1,\ldots,y_n$ are intricately dependent, we do not have access to multiple independent samples from the conditional distribution~\eqref{eq:general binary}, but a single sample from that distribution. Thus, it is not clear how to extract good estimates of the parameters from our observations and where to find statistical power to bound the error of these estimates from the true parameters. As a consequence, only limited theoretical results in this area are known.

% \surbhi{Would be good to make this a full section and move the related work before this}
\subsection{Overview of Results} \label{sec:results} We provide a general algorithmic approach which yields efficient statistical rates for the estimation of $\theta$ and/or $\beta$ of Model~\eqref{eq:general binary} for general function classes $\mathcal{F}$, in terms of the metric entropy of $\mathcal{F}$. We also prove information theoretic lower bounds, which combined with our upper bounds  characterize the min-max estimation rate of the problem up to a certain factor, discussed below. Before stating our general result as Theorem~\ref{thm:costis:general upper}, we present some corollaries of this theorem in more familiar settings. All the theorems that follow are also presented and proved in more detail in the Supplementary Material. Finally, in all statements below we use the following notation and assumptions, which summarize the already described setting.
\begin{assumptions}[and useful Notation]\label{assump}
We are given observations $(x_i,y_i)_{i=1}^n$, where $y_1,\ldots,y_n$ are sampled from~\eqref{eq:general binary} conditioning on $x_1,\ldots,x_n$, using some unknown parameters $\theta^*\in \Theta$ and $\beta^*\in [-B,B]$, and some known $A$, normalized such that $\|A\|_\infty = 1$. We further assume that $|f_\theta(x_i)| \le M$, for all $i$ and $\theta \in \Theta$. In all statements below, $\hat{\theta}$ and $\hat{\beta}$ refer to the estimates produced by the algorithm described in Section~\ref{sec:algorithm}, i.e.~the Maximum Pseudo-Likelihood Estimator (MPLE). Moreover, we let $\lesssim$ denote an inequality up to factors that are  singly-exponential in $M$ and $B$, a necessary dependence on those parameters when $\lesssim$ is used, and are independent of all other parameters. In particular, when $M,B=O(1)$, $\lesssim$ denotes inequality up to a constant. 
\end{assumptions}
Under the assumptions on our observations, and notation introduced above, we consider two settings to illustrate our general result (Theorem~\ref{thm:costis:general upper}), namely linear classes (Setting~\ref{setting:our maing setting}) and neural network classes (Setting~\ref{set2}).
\begin{setting}[Linear Classes] \label{setting:our maing setting}
Make Assumptions~\ref{assump}, suppose  $x_i \in \mathbb{R}^d$ and $\|x_i\|_2 \le M$, for all $i$, and suppose that $f_\theta$ is linear, i.e.~$f_\theta(x_i) = x_i^\top \theta$, for some $\theta \in \mathbb{R}^d$ and  $\|\theta\|_2 \le 1$.
% where $x_i,\theta \in \mathbb{R}^d$, $\|x_i\|_2 \le M$ and , and thus $|f_\theta(x_i)| \le M$, for all~$i$. 
Denote by $X$ the matrix whose rows are $x_1,\ldots,x_n$ and by $\kappa$ the minimum eigenvalue of $\frac{X^TX}{n}$, or its minimum restricted eigenvalue in the sparse setting of Theorem~\ref{thm:costis:sparse linear}. We suppress from our bounds of  Theorems~\ref{thm:costis:linear} and~\ref{thm:costis:sparse linear} a factor~of~$1/\kappa$.
 %Moreover, we assume that $\max_i \|x_i\|_2$ and $\|\theta^*\|$ are bounded by a constant $M$, unless stated otherwise. 
 %Thus, the bounds have in implicity dependence on $\exp(M)$, which is a constant and is thus omitted. Finally, let $\lesssim,\gtrsim$ denote inequalities up to constant factors.
\end{setting}

\begin{restatable}[Linear Class]{theorem}{linear} \label{thm:costis:linear}
Suppose Setting~\ref{setting:our maing setting}.
% , where
% $\Theta = \{\theta\in \mathbb{R}^d \colon \|\theta\|_2 \le 1 \}$. 
Then, with probability $\ge 1-\delta$,
\[
\|\hat{\theta}-\theta^*\|_2^2 + |\hat{\beta}-\beta^*|^2
\lesssim \frac{d \log n + \log(1/\delta)}{\|A\|_F^2}.
\]
\end{restatable} 

\begin{theorem}[Sparse Linear Class] \label{thm:costis:sparse linear}
Suppose Setting~\ref{setting:our maing setting} and additionally  that $\|\theta\|_1 \le s$.
% , where  
% $\Theta=\{\theta \in\mathbb{R}^d \colon \|\theta\|_1 \le s, \|\theta\|_2 \le 1\}$. 
Then, w.pr. $\ge 1-\delta$,

\[
\|\hat{\theta}-\theta^*\|_2^2 + |\hat{\beta}-\beta^*|^2
\lesssim  \frac{(n^2 s\log(d))^{1/3} + \log(1/\delta)}{\|A\|_F^2}.
\]
\end{theorem}
Theorem~\ref{thm:costis:linear} is proved in Section~\ref{sec:pr:linear} and Theorem~\ref{thm:costis:sparse linear} is proved in Section~\ref{sec:pr:sparse}.

% \footnote{\surbhi{Do we have a Theorem 2 analog for this that we could add?}\costis{i don't think we've worked it out...err...i think our analysis applies directly; I defer this to Varids/Yuval}
% \vardis{We could probably have a similar statement for sparse but the technicalities were already on a high level so we didnt pursue one.}}
Both bounds above are obtained by minimizing a convex function over a convex domain, which can be performed in polynomial time.
We note that the bound of Theorem~\ref{thm:costis:linear} generalizes the main result of~\citet{daskalakis2019regression}, which makes the additional assumption that $\|A\|_F = \Omega(\sqrt{n})$. We need no such assumption and our bound gracefully degrades as $\|A\|_F$ decreases.  Theorem~\ref{thm:costis:sparse linear} extends these results to the sparse linear model, for which no prior results exist. Note that our bound is non-vacuous as long as $\|A\|_F = \Omega(n^{1/3})$, which is a reasonable expectation, given that $A$ is $n \times n$. Moreover, it is possible to remove the appearance of $n^2$ from the bound of this theorem, if our model class satisfies $|\theta|_0 \le s$.
% \footnote{\yuval{This results in a non-efficient algorithm, with runtime $n^s$. However, we claimed efficiency.} \costis{we didn't claim efficiency in the theorems so technically there is no problem here. But it's ok by me if you want to make that point.}} 
Finally, we note that the factor $1/\|A\|_F^2$ which appears in our error bounds is tight, as per the following.
%for some worst-case instances:
\begin{theorem}[Lower bound]\label{thm:costis:lowerbound}
For any $n$ and $r\in [1,n]$ there exists an instance of a $d=1$-dimensional linear class that satisfies the assumptions of Theorems~\ref{thm:costis:linear} and \ref{thm:costis:sparse linear} and further $\|A\|_F^2 = r$, such that any estimator $(\theta',\beta')$ satisfies with probability $\ge 0.49$,
\[
|\theta'-\theta^*|^2 \gtrsim \frac{1}{\|A\|_F^2}, \quad |\beta'-\beta^*|^2 \gtrsim \frac{1}{\|A\|_F^2}.
\]
\end{theorem}
Theorem~\ref{thm:costis:lowerbound} is proved in Section~\ref{sec:pr:lower-Forb}.
While Theorem~\ref{thm:costis:lowerbound} shows that a dependence in $\frac{1}{\|A\|_F^2}$ is unavoidable in the worst case, under favorable assumptions we can remove such dependence as per the following theorem.
\begin{theorem}[Linear Class, Random Features] \label{thm:costis:linear theta only}
In the same setting as Theorem~\ref{thm:costis:linear}, remove all assumptions involving the feature vectors and suppose instead that $x_1,\ldots,x_n \overset{i.i.d.}{\sim} {\cal N}(0,I_d)$.  Then, with probability $\ge 1-\delta$,
\[
\|\hat{\theta}-\theta^*\|_2^2 
\lesssim \fudge(n, {1 /\delta}) \frac{d + \log(1/\delta)}{n}\left( 1 + \frac{d + \log(1/\delta)}{\|A\|_F^2/\|A\|_2^2}\right),
\]
where $\fudge(n, {\frac{1}{\delta}})$ is linear in $\log \log(\frac{1}{\delta})$ and sub-polynomial (i.e.~asymptotically smaller than any polynomial) in $n$.
\end{theorem} 

Theorem~\ref{thm:costis:linear theta only} is proved in Section~\ref{sec:pr:linear}.
Noticing that $\|A\|_F^2/\|A\|_2^2 \ge 1$, Theorem~\ref{thm:costis:linear theta only} shows that no lower bound on $\|A\|_F$ is necessary at all, if we are only looking to estimate $\theta^*$, which answers a main problem left open by~\citet{daskalakis2019regression}. Moreover, when $\|A\|_F^2/\|A\|_2^2 \ge d$, which is a reasonable expectation in our setting since $\|A\|_2 \le 1$ and $A$ is $n \times n$, our bound here essentially matches the estimation rates known for the familiar logistic regression problem, which corresponds to the case $\beta=0$, even though we make no such assumption, and hence our labels are dependent.

Beyond linear and sparse linear function classes, our main result (Theorem~\ref{thm:costis:general upper}) provides estimation rates for neural network regression, as in the following setting.

\begin{setting}[Neural Networks]\label{set2}
Make Assumptions~\ref{assump} and suppose that the function $f_\theta$ in~\eqref{eq:general binary} is a neural network parameterized by~$\theta$. We adopt the setting and terminology of \cite{bartlett2017spectrally}. In particular, we assume that the neural network takes the form: 
\begin{align}f_{\theta}(x)=\sigma_L(W_L\sigma_{L-1}(W_{L-1}\cdots\sigma_1(W_1x)\cdots)), \label{eq:DNN}
\end{align}
where the depth $L$ of the network is fixed, $\sigma_1,\ldots,\sigma_L: \mathbb{R} \rightarrow \mathbb{R}$ are some fixed non-linearities, and $W_1,\ldots,W_L$ are (unknown) weight matrices. In particular, $\theta=(W_1,\ldots,W_L)$.

We denote by $\rho_1,\ldots,\rho_L$ the Lipschitz constants of the non-linearities, and when, abusing  notation, we apply some non-linearity $\sigma_i$ to a vector $v$, the result $\sigma_i(v)$ is a vector whose $j$-th coordinate is $\sigma_i(v_j)$. We also adopt from~\cite{bartlett2017spectrally} the notion of \emph{spectral complexity} $R_\theta$ of a neural network $f_\theta$ with respect to reference matrices $M_1,\ldots,M_L$ (of the same dimensions as $W_1,\ldots,W_L$ respectively), defined in terms of different matrix norms as follows:
$$R_\theta = \left(\prod_{i=1}^L \rho_i \|W_i\|_2 \right)\left(\sum_{i=1}^L\frac{\|W_i^{\rm T} - M_i^{\rm T} \|_{2,1}^{2/3}}{\|W_i\|_2^{2/3}}\right)^{3/2},$$
where $\|M\|_{2,1} = \sum_{j=1}^n \sqrt{\sum_{i=1}^n M_{ij}^2}$.
Assuming a fixed bound on each matrix norm involved in the above expression, we take 
% \footnote{In~\cite{bartlett2017spectrally}, a neural network is denoted by $F_{\mathcal{A}}$ and its spectral complexity is denoted by $R_{\mathcal{A}}$.}
 $\mathcal{F}=\{f_\theta \colon \theta \in \Theta\}$ to be the collection of all neural networks of Form~\eqref{eq:DNN}, whose weight matrices satisfy those bounds. Suppose $R$ is the resulting bound on the spectral norm of all networks in our family, implied by our assumed bounds on the various matrix norms.
%  $R$ $\W_1,\ldots$ of a certain architecture with bounded matrix norms,\footnote{The class $\mathcal{F}$ defined above corresponds to the class $\mathcal{H}_X$ in Theorem~3.3 of \cite{bartlett2017spectrally}.} and let $R$ be an upper bound on the spectral complexities of all the networks in $\mathcal{F}$.
Finally, we assume that the widths of all networks $f_\theta \in \mathcal{F}$ are bounded by $d$. 
\end{setting}
\begin{restatable}{theorem}{neural}\label{t:neural_net}
Suppose Setting ~\ref{set2}, and let $K^2 = \frac{1}{n}\sum_i \|x_i\|_2^2$. Then, with probability $\geq 1 - \delta$,

\[
\frac{1}{n}\sum_{i=1}^n (f_{\hat{\theta}}(x_i)-f_{\theta^*}(x_i))^2 + |\hat{\beta}-\beta^*|^2
\quad\lesssim \frac{(n^2 K^2 R^2 \log d)^{1/3} +\log\lp(\frac{n}{\delta}\rp)  }{\|A\|_F^2}.
\]

\end{restatable}
Theorem~\ref{t:neural_net} is proved in Section~\ref{sec:pr:neural}.
Notice that, in this case, we do not provide guarantees for the estimation of $\theta$. Since these networks are often overparametrized, it might be impossible to recover $\theta$. 

All estimation results  above, namely Theorems~\ref{thm:costis:linear}--\ref{t:neural_net}, are corollaries of our general estimation result given below.

\begin{theorem}[General Estimation Result] \label{thm:costis:general upper}
Make Assumptions~\ref{assump}, where $f_{\theta^*}$ lies in some general class ${\cal F}=\{f_\theta\}_\theta$. Then, w.pr. $\ge 1-\delta$,

\begin{equation}
\label{eq:thm-gen-bnd}
\frac{1}{n} \sum_{i=1}^n (f_{\hat\theta}(x_i)-f_{\theta^*}(x_i))^2
\lesssim \mathcal{C}_1(\mathcal{F},X,\beta^*,\theta^*)\inf_{\epsilon \ge 0} \lp(\log \frac{n}{\delta} + \epsilon n + \log N(\mathcal{F},X,\epsilon)\rp),\notag
\end{equation}

where $X$ denotes the collection of feature vectors, $N\lp(\mathcal{F},X,\epsilon\rp)$ is the $\epsilon$-covering number of $\cal F$ under distance $d(f,f')=\sqrt{\sum_{i=1}^n (f(x_i)-f'(x_i))^2/n}$  and $\mathcal{C}_1 \le 1/\|A\|_F^2$ is a quantity that has a simple formula (both quantities are formally defined in Section~\ref{sec:formal-defs}). Further, if $\mathcal{F}$ is convex and closed under negation,\footnote{We say that $\mathcal{F}$ is convex if for any $f,f'\in \mathcal{F}$ and any $\lambda \in [0,1]$ the function $\tilde{f}(x) = (1-\lambda)f(x) +\lambda f'(x)$ belongs to $\mathcal{F}$. We say that $\mathcal{F}$ is closed under negation if $-f \in \mathcal{F}$ for all $f \in \mathcal{F}$.} for any estimator $(\theta',\beta')$ there exists $(\theta^*,\beta^*)$, s.t. w.pr. $\ge 0.49$,
\[
\frac{1}{n} \sum_{i=1}^n (f_{\theta'}(x_i)-f_{\theta^*}(x_i))^2
\gtrsim \mathcal{C}_1(\mathcal{F},X,\beta^*,\theta^*).
\]
Similar upper and lower bounds hold for estimating $\beta^*$, with $\mathcal{C}_1$ replaced with a different quantity $\mathcal{C}_2 \le 1/\|A\|_F^2$.
\end{theorem}
Theorem~\ref{thm:costis:general upper} is proved in Section~\ref{sec:ising-formulation}.
Theorem~\ref{thm:costis:general upper} is used to derive Theorems~\ref{thm:costis:linear}, \ref{thm:costis:sparse linear}, \ref{thm:costis:linear theta only},  and~\ref{t:neural_net} by bounding the covering numbers of linear, sparse linear and neural network classes. It is also used to derive Theorem~\ref{thm:costis:lowerbound} in a straight-forward way. It is worth emphasizing that we obtain separate general estimation rates for $\beta$ and $\theta$, which are tight or near-tight in a variety of settings. 

%Lower bounds in terms of covering numbers exist as well, although their form is generally more intricate. Further, this theorem asserts that the rate for linear regression equals $\mathcal{C}(\mathcal{F},\theta^*,\beta^*)$ up to a factor of $d\log n$.
%\costis{(Do you want to keep this paragraph or remove it? also what is ${\cal C}$? why isn't it not $1/{\cal C}$?)}

% \costis{Todos: Theorems 1, 2, 3 and 6 refer to theorems in the appendix but we currently have no numbers for those. do you want to remove those references and say these are special cases of our theorems presented in the  supplementary without refs to those? do you want to commit on the numbering of the theorems in the supplementary? or should the supplementary be just another copy of our paper including the main body here with those refs added and remove the references from the version we submit tomorrow?}
% \vardis{Since it's difficult to commit on the numbering, the second option seems better. Is it acceptable if we just say that all these theorems will be presented and proved in the appendix?} \surbhi{You can just add the appendix here in onecolumn format and then we can cut the file and not submit that part}

\subsection{Related Work} \label{sec:related}

Data dependencies are pervasive in many applications of Statistics and Machine Learning, e.g.~in financial, meteorological, epidemiological, and geographical applications, as well as social-network analyses, where peer effects have been studied in topics as diverse as criminal activity~\cite{glaeser1996crime}, welfare participation~\cite{bertrand2000network}, school achievement~\cite{sacerdote2001peer},  retirement plan participation~\cite{duflo2003role}, and obesity~\cite{christakis2013social,trogdon2008peer}. These applications have motivated substantial work in Econometrics (see~e.g.~\citet{manski1993identification,bramoulle2009identification} and their references), where identification results have been pursued and debated, mostly in linear auto-regressive models; see also~\citet{daskalakis2019regression}. In Statistical Learning Theory, learnability and uniform convergence bounds have been shown in the presence of sample dependencies; see e.g.~\citet{yu1994rates,gamarnik2003extension,berti2009rate,mohri2009rademacher,pestov2010predictive,mohri2010stability,ShaliziK13,london2013collective,kuznetsov2015learning,london2016stability,McDonald2017,dagan2019learning}. Those learnability frameworks are not applicable to our setting due to  exchangeability, fast-mixing, or weak-dependence properties that they are exploiting.

Close to our setting, recent work of~\citet{daskalakis2019regression} considers a special case of our problem, where function $f_\theta$ in Model~\eqref{eq:general binary} is assumed linear. We obtain stronger estimation bounds, under weaker assumptions, our bounds gracefully degrading with~$\|A\|_F$, as we have already discussed. Similarly, earlier work by~\citet{chatterjee2007estimation,bhattacharya2018inference,ghosal2018joint,dagan2020estimating}, motivated by single-sample estimation of Ising models, considers a special case of our problem where function~$f_\theta$ in Model~\eqref{eq:general binary} is assumed constant. Our bounds in this simple setting are as tight as the tightest bounds in that line of work. Overall, in comparison to these works, our general estimation result (Theorem~\ref{thm:costis:general upper}) covers arbitrary classes $\cal F$, characterizing the estimation rate up to a factor that depends on the metric entropy of $\cal F$. We thus obtain rates for sparse linear classes (Theorem~\ref{thm:costis:sparse linear}), neural network classes (Theorem~\ref{t:neural_net}), and Lipschitz classes (discussed in the Supplementary Material), which had not been shown before. Finally, our bounds disentangle our ability of estimating $\theta$ and $\beta$, allowing for the estimation of $\theta$ even when the estimation of $\beta$ is impossible, as shown in Theorem~\ref{thm:costis:linear theta only} for linear classes, answering a main open problem left open by~\cite{daskalakis2019regression}.

% First, we provide optimal statistical rates for the problem of estimating $(\beta,\theta)$ making fewer assumptions about the model, namely properties of the interaction matrix $A$. Second, our bounds disentangle  Third, we go beyond linear models to arbitrary models, instantiating our results in terms of the metric entropy of the function class $\{f_\theta\}_{\theta}$. 

At a higher level, single-sample statistical estimation is both a classical and an emerging field~\cite{besag1974spatial,bresler2018optimal,valiant2019worstcase,dagan2020estimating} with intimate connections to Statistical Physics, Combinatorics, and High-Dimensional Probability.

% \cite{kontorovich2008concentration, kontorovich2017concentration}

\textbf{Roadmap.} 
% We discuss related work in Section~\ref{sec:related}. 
We present the estimator used to derive all our upper bounds in Section~\ref{sec:algorithm}. We present a sketch of our proof of Theorem~\ref{thm:costis:general upper} in Section~\ref{sec:proof sketch}. We do this in two steps. First we present a sketch for the toy case of Theorem~\ref{thm:costis:linear}, i.e.~the single-dimensional case. This illustrates some of the main ideas of the proof. We then provide the modifications necessary for the multi-dimensional case, which naturally lead us to the formulation of Theorem~\ref{thm:costis:general upper}. While the main technical ideas are already illustrated in Section~\ref{sec:proof sketch} in sufficient detail, the complete details can be found in the supplementary material. We conclude with experiments in Section~\ref{sec:experiments}, where we apply our estimator on citation datasets and compare its prediction accuracy to supervised learning approaches that do not take into account label dependencies. 

% \surbhi{We have already said this before, we should probably have other related work that handles this setting. Also, move it to the introduction}

\section{The Estimation Algorithm} \label{sec:algorithm}
In all our theorems, the estimator we use is the Maximum Pseudo-Likelihood Estimator (MPLE), first proposed by \citet{besag1974spatial} and defined as follows

\begin{equation}
\label{eq:MPLE}
(\hat\theta,\hat \beta) := \arg\max_{\theta,\beta} \prod_{i=1}^n \Pr_{\theta,\beta}[y_i |x,y_{-i}]
\equiv \prod_{i=1}^n \frac{\exp\lp(y_i \lp(f_{\theta}(x_i) + \beta \sum_{j=1}^n A_{ij}y_j\rp)\rp)}{2\cosh\lp(f_{\theta}(x_i) + \beta \sum_{j=1}^n A_{ij}y_j\rp)}\enspace,
\end{equation}

where $\Pr_{\theta,\beta}$ is defined in \eqref{eq:general binary}, $x=(x_1,\dots,x_n)$ and $y_{-i} = (y_1,\dots, y_{i-1},y_{i+1},\dots,y_n)$.
We optimize the MPLE over $\theta \in \Theta$ and $\beta\in [-B,B]$, for $\Theta,B$ as per Assumption~\ref{assump}. 

In comparison to MPLE, the more common Maximum Likelihood Estimator (MLE) optimizes $\Pr_{\theta,\beta}[y_{1\dots n}\mid x_{1\dots n}]$. Notice that the MPLE coincides with the MLE in the case $\beta = 0$, which corresponds to $y_1,\dots,y_n$ being independent conditioned on $x_{1\dots n}$. When $\beta \ne 0$, this conditional independence ceases to hold and the two methods target different objectives. In this case, the objective function of MLE, which is~\eqref{eq:general binary}, involves the normalizing factor $Z_{\theta,\beta}$, which is in general computationally hard to approximate~\cite{sly2014counting}. In contrast, the MPLE is efficiently computable in many cases.
% optimize different it is apparent how the dependencies in the matrix $A$ are taken into account in \eqref{eq:MPLE}. In this dependent setting, it is possible to execute the MLE instead of the MPLE, optimizing \eqref{eq:general binary} directly. However, this requires to compute the normalizing factor 
%On the other hand, the pseudo-likelihood has a closed form, since the conditional distribution of a node conditioned on the rest of the nodes is supported on $2$ values and is easy to compute. Furthermore, the pseudo-likelihood function carries a lot of information about the distribution of the nodes. To see that, notice that since $\|A\|_\infty$ is bounded and the external field is also bounded, the total influence of $\sigma_{-i}$ on the distribution of $\sigma_i|\sigma_{-i}$ is bounded. This means that a change in the value of $\theta$ or $\beta$ would also alter the probability $\Pr_{\theta,\beta}[\sigma_i |\sigma_{-i}]$ enough to help us recover the parameters. This 
%ensures that the function we are trying to optimize is not "flat". 
%The MPLE also has other desirable properties. In many interesting settings, it is a convex function of $\theta,\beta$. 
For example, in the linear case where $f_\theta(x_i) = x_i^\top\theta$, the logarithm of~\eqref{eq:MPLE} is a convex function of $\theta$ and $\beta$.
%function is convex, if $\theta$ belongs to a convex domain $\Theta$.
Hence, we can use a variety of convex optimization algorithms to find the optimal solution. Even in cases where it is not a convex function, we can always use generic optimization techniques such as gradient-based methods to find a local optimum fast, since the derivative is easy to compute. 
Thus, the MPLE is a very appealing choice for various models. 
In all the results that follow, both theoretical and practical, the algorithm used will be the MPLE.

\section{Proof overview} \label{sec:proof sketch}
In this section, we will briefly describe the most important contributions of this work at the technical level. 
We start by discussing the case where $f_\theta$ is linear and $\theta$ is a one-dimensional  parameter.  We describe in detail the obstacles that had to be overcome to obtain tight rates for the estimation of $\theta$ and $\beta$ in this case and highlight some of the most important features of the proof.
In particular, we use the \emph{mean field approximation}, a tool from statistical physics, to derive the bounds. Later, we sketch the proof of the general Theorem~\ref{thm:costis:general upper}.

% In Section~\ref{sec:formal-defs} we introduce the formal definitions that are used in Theorem~\ref{thm:costis:general upper} and sketch its proof in Section~\ref{sec:proof-general}. Lastly, in Section~\ref{sec:pr-apply}
% describe how to obtain the main theorems of this paper using Theorem~\ref{thm:costis:general upper}.

\textbf{Notation: Matrix Norms.}
We use the Forbenius norm $\|A\|_F$, the spectral norm $\|A\|_2$ and the infinity-to-infinity norm $\|A\|_\infty$. In our setting $A$ is symmetric, so one has $\|A\|_2 \le \|A\|_\infty = 1$ and $\|A\|_F \le \sqrt{n}\|A\|_2 \le \sqrt{n}$.
% For a matrix $A$, we denote its Frobenius norm by $\|A\|_F = \sqrt{\sum_{i=1}^n A_{ij}^2}$, its spectral norm by $\|A\|_2 = \max_{v\ne 0} \|Av\|_2/\|v\|_2$ and its infinity norm by $\|A\|_\infty = \max_{v\ne 0} \|Av\|_\infty/\|v\|_\infty = \max_{i=1}^n \sum_{j=1}^n |A_{ij}|$. Any $n\times n$ matrix satisfies $\|A\|_F \le \sqrt{n} \|A\|_2$ and any symmetric matrix satisfies $\|A\|_\infty \le \|A\|_2$. 

% In our setting, $A$ will be symmetric with $\|A\|_\infty \le 1$, which implies that $\|A\|_2 \le 1$ and $\|A\|_F \le \sqrt{n}$.

\subsection{Single-dimensional linear classes}\label{sec:single-dim-upper}
%First, suppose $\theta \in \mathbb{R}$ is a single dimensional parameter, assume that $x_1,\dots,x_n \in \mathbb{R}$ and $f_{\theta}(x_i) = \theta x_i$, denote $x=(x_1,\dots x_n)$, $y = (y_1,\dots,y_n)$ and assume that $\|A\|_\infty,|\theta^*|,|\beta^*|,|x_1|,\dots,|x_n| \le O(1)$ and $\|x\|_2 \ge \Omega(\sqrt{n})$. Let $\hat\theta,\hat\beta$ be the MPLE estimates obtained by optimizing \eqref{eq:MPLE} over $\{(\theta,\beta) \colon |\theta|,|\beta|\le O(1) \}$.
We consider the setting of Theorem~\ref{thm:costis:linear}, when the dimension is $d=1$. We denote $x=(x_1,\dots x_n)$ and $y = (y_1,\dots,y_n)$. To simplify the presentation, we assume $\kappa\ge \Omega(1)$, which implies that $\|x\|_2\ge \Omega(\sqrt{n})$, and further that $M,B = O(1)$. 
In this sketch we focus on estimating $\theta$ while the bound on $\beta$ is similarly obtained, and our goal is to show the special case of Theorem~\ref{thm:costis:linear} for dimension $d=1$, namely, that with probability $\ge 1-\delta$:
\begin{equation}\label{eq:bnd-oneparam-combined}
|\hat{\theta}-\theta^*| \lesssim \frac{\sqrt{\log \frac n\delta}}{\|A\|_F}.
\end{equation}
In fact, we will show the tighter bound of:
\begin{align}\label{eq:bnd-one-theta}
|\hat\theta - \theta^*| \lesssim\sup_{\lambda \in \R}\frac{\sqrt{\log \frac{n}{\delta}}}{\|\lambda A\|_F + \left\|x - \lambda A\vtanh\left(\frac{\beta^*x}{\lambda} + \theta^* x\right)\right\|_2}
\end{align}
where $\vtanh(z_1,\dots,z_n) = (\tanh(z_1),\dots,\tanh(z_n))$.
We note that this bound is tight up to the factor of $\sqrt{\log\frac n\delta}$ (after a small tweak to these bounds that we omit for simplicity), and it can be obtained from our general bound of Theorem~\ref{thm:costis:general upper} with respect to the quantity $\mathcal{C}_1$ (see Section ~\ref{sec:formal-defs}.

Before establishing~\eqref{eq:bnd-one-theta}, we note that it is stronger than the right hand side of \eqref{eq:bnd-oneparam-combined}. This follows from a simple exercise,
considering cases for $\lambda$ and
utilizing the fact that under the assumptions stated above, $\|\lambda A\vtanh((\beta^*/\lambda) x + \theta^* x)\|_2 \le O(\lambda \sqrt{n})$, while $\|x\|_2 \ge \Omega(\sqrt{n})$.

% and~\eqref{eq:bnd-one-beta}, let us first argue that they are stronger bounds than the right hand side of \eqref{eq:bnd-oneparam-combined}, as we have claimed. Bounding \eqref{eq:bnd-one-beta} by the RHS of \eqref{eq:bnd-oneparam-combined} is trivial. So we proceed to bound \eqref{eq:bnd-one-theta} by the RHS of \eqref{eq:bnd-oneparam-combined}. For any $\lambda$ smaller than some small enough constant $\lambda_0=\Theta(1)$, one has by the triangle inequality,
% \begin{multline}\label{eq:bnd-gen-by-forb}
% \|\lambda A\|_F + \|x - \lambda A\vtanh((\beta^*/\lambda) x + \theta^* x)\|_2\\
% \ge \|\lambda A\|_F + \|x\|_2
% - \|\lambda A\vtanh((\beta^*/\lambda) x + \theta^* x)\|_2\\
% \ge \|\lambda A\|_F + \|x\|_2
% - |\lambda| \|A\|_2 \|\vtanh((\beta^*/\lambda) x + \theta^* x)\|_2\\
% \ge \|\lambda A\|_F + \Omega(\sqrt{n}) - |\lambda| O(\sqrt{n})
% \ge \Omega(\|A\|_F),
% \end{multline}
% where we use that $\|x\|_2 \ge \Omega(\sqrt{n})$, $\|A\|_2 \le O(1)$, $\|A\|_F \le \sqrt{n}$, and $|\tanh(z)| \le 1$ for all $z \in \mathbb{R}$. On the other hand, for any $\lambda$ greater than $\lambda_0$, the LHS of \eqref{eq:bnd-gen-by-forb} is lower bounded by $\Omega(\|A\|_F)$. Thus, in both cases the denominator of~\eqref{eq:bnd-one-theta} is lower bounded by $\Omega(\|A\|_F)$, establishing that~\eqref{eq:bnd-one-theta} is bounded by the RHS of~\eqref{eq:bnd-oneparam-combined}.
%The right hand side of \eqref{eq:bnd-one-beta} is clearly bounded by this quantity. For the right hand side of \eqref{eq:bnd-one-beta}, it suffices to show that for all l

We proceed with sketching the proof of \eqref{eq:bnd-one-theta}.
Let $\phi(\theta,\beta)$ be the negative pseudo log-likelihood for the pair $(\theta,\beta)$, namely, minus the log of the quantity in \eqref{eq:MPLE}. This is a convex function whose minimum equals $(\hat\theta,\hat\beta)$ and our goal is to show that $(\theta^*,\beta^*)$ lies in proximity to this minimum. In order to show this, it suffices to prove that the gradient of $\phi$ at $(\theta^*,\beta^*)$ is small, while the function is strongly convex in its neighborhood. For a more rigorous proof, we  write $\phi(\hat\theta,\hat\beta)$ using a Taylor sum around $(\theta^*,\beta^*)$. Denoting  $v= (v_\theta,v_\beta)=(\hat \theta - \theta^*, \hat \beta - \beta^*)$, we get:
\begin{equation*}
    \varphi(\hat\theta,\hat\beta)
    = \varphi(\theta^*,\beta^*)
    + v^\top \nabla \varphi(\theta^*,\beta^*) + \frac{1}{2}v^\top \nabla^2 \varphi(\theta',\beta') v,
\end{equation*}
for some $(\theta',\beta')$ in the segment connecting $(\theta^*,\beta^*)$ and $(\hat \theta,\hat \beta)$. Since $(\hat\theta,\hat\beta)$ is the minimizer of the MPLE, one has $\varphi(\hat\theta,\hat\beta) \le \varphi(\theta^*,\beta^*)$, which implies that
\begin{equation}\label{eq:321}
\frac{1}{2} v^\top \nabla^2 \varphi(\theta',\beta') v
\le - v^\top \nabla \varphi(\theta^*,\beta^*)
\le |v^\top \nabla \varphi(\theta^*,\beta^*)|.
\end{equation}
Using concentration inequalities from \cite{dagan2020estimating}, we can show that w.pr. $\ge 1-\delta$ (w.r.t.~the randomness of the $y_1,\ldots,y_n$ which are implicit arguments of $\phi$), any $u \in \mathbb{R}^2$ satisfies
\begin{equation}\label{eq:322}
\frac{|u^\top \nabla \varphi(\theta^*,\beta^*)|}{u^\top \nabla^2 \varphi(\theta',\beta') u}
\lesssim \frac{\sqrt{\log n/\delta}}{\|u_\beta A\|_F + \|u_\theta x + u_\beta Ay\|}.
\end{equation}
After substituting $u=v$, it follows from \eqref{eq:321} that the left hand side of \eqref{eq:322} is lower bounded by $1/2$. We derive that
\[
1 \lesssim \frac{\sqrt{\log n/\delta}}{\|v_\beta A\|_F + \|v_\theta x + v_\beta Ay\|}.
\]
Multiplying by $v_\theta$, and writing $\lambda=-v_\beta/v_\theta$, we have

\begin{equation}
\label{eq:random-bnd}
|\hat\theta-\theta^*|
= |v_\theta| \le \frac{\sqrt{\log n/\delta}}{\|\lambda A\|_F + \|x - \lambda Ay\|}
\le \sup_{\lambda \in \mathbb{R}} \frac{\sqrt{\log n/\delta}}{\|\lambda A\|_F + \|x - \lambda Ay\|}.
\end{equation}

At this point, we have bounded the rate by the solution to an optimization problem. However, notice that the right hand side contains $y$ which is a random variable. We would like to show that the whole expression is bounded by a nonrandom quantity and, in particular, by \eqref{eq:bnd-one-theta}.
This statement requires new insights and, as a result, a significant part of the proof is devoted to it. Here, we first sketch the main idea and then give a more technical explanation for it. 

We would like to bound the optimization problem in  \eqref{eq:random-bnd} by that in
\eqref{eq:bnd-one-theta}, which corresponds to showing
\begin{equation}\label{eq:bnd-rand-nonrand}
\|\lambda A\|_F + \|x - \lambda Ay\| \gtrsim
\|x - \lambda A\vtanh((\beta^*/\lambda) x + \theta^* x)\|.
\end{equation}
%We would like to understand how the quantity $\|x - \lambda Ay\|$ behaves. Since this quantity appears in the denominator of \eqref{eq:random-bnd}, we would like to obtain a lower bound for this quantity. 
We start by describing a rough and informal intuition for proving \eqref{eq:bnd-rand-nonrand}, and later proceed with a more formal derivation. %Notice that since it is in the denominator, we are trying to find a $\lambda$ which minimizes $\|x - \lambda Ay\|$, i.e. the projection of $x$ to $Ay$. 
We use an approach from statistical physics that is called \emph{mean-field approximation}: we
can substitute each $y_i$ with $\mathbb{E}[y_i \mid x, y_{-i}] = \tanh(\beta^* \sum_j A_{ij}y_j + \theta^* x)$. Applying this substitution for all $i$, we obtain that 
\begin{equation}\label{eq:333}
y\approx \vtanh(\beta^* Ay + \theta^* x).
\end{equation}
We assume towards contradiction that \eqref{eq:bnd-rand-nonrand} does not hold, and in this case we make the (false) substitution $\|x-\lambda Ay\|_2 \approx 0$, which implies that
$Ay \approx x/\lambda$. Substituting this in the right hand side of \eqref{eq:333}, we obtain that
% \begin{equation}\label{eq:mean-field}
$y \approx \vtanh(\beta^* x/\lambda + \theta^* x).$
% \end{equation}
Making this substitution in $\|x-\lambda Ay\|$, we obtain \eqref{eq:bnd-rand-nonrand}.

% Namely, we can prove that $\lambda Ay \approx \lambda A\tanh(\beta^* Ay + \theta^* h)$. This means that if $h$ is far from $\lambda Ay$, then $h$ is far from satisfying the mean field approximation. Thus, in this case, we should expect the error of $h$ in the mean-field approximation, which is 
% $$
% \|h - \lambda A\tanh((\beta^*/\lambda) h + \theta^* h\|
% $$
% , to be large. Thus, it would be reasonable to substitute this deterministic quantity in the optimization problem to determine the rate.

Now, we will argue more formally about the previous claims to derive \eqref{eq:bnd-rand-nonrand}.
Using the triangle inequality, we get
\begin{multline}\label{eq:three-terms}
\|x - \lambda A\vtanh(\beta^*x/\lambda + \theta^* x)\| \leq
+\|x - \lambda A y\|+
 \|\lambda A y - \lambda A \vtanh(\beta^* A y + \theta^* x)\|\\
+ \|\lambda A \vtanh(\beta^* A y + \theta^* x) - \lambda A\vtanh(\beta^*x/\lambda + \theta^* x)\|.
\end{multline}
We would like to bound each of the three terms on the right hand side by a constant times the left hand side of \eqref{eq:bnd-rand-nonrand}. For the first term, this is trivial. Further, we can show that the third term on the right hand side of \eqref{eq:three-terms} is bounded by the first term, using the Lipschitzness of $\tanh$:
\begin{multline*}
 \|\lambda A \vtanh(\beta^* A y + \theta^* x) - \lambda A\vtanh((\beta^*/\lambda) x + \theta^* x)\|
 \leq \|\lambda A \beta^* Ay - A \beta^* x\| \leq \|\beta^* A\|_2 \|x - \lambda A y\|\\
 \le O(\|x-\lambda Ay\|),
\end{multline*}
where $\|\beta^*A\|_2 \le O(1)$ using the assumptions of this paper.
%Thus, the third term is of the same order as the first term. 
As for the second term, it represents the error of the mean field approximation for $y$, which corresponds to the substitution in \eqref{eq:333}. In order to bound this error term, we use the method of
exchangeable pairs developed in \cite{chatterjee2005concentration}, which provides a strong and general concentration inequality for non-independent random variables. We can show that with high probability, this term will be  $O(\|\lambda \beta^* A\|_2) \le O(\lambda\|A\|_2) \le O(\lambda\|A\|_F)$, since $B = O(1)$.
Combining the above bounds we derive \eqref{eq:bnd-rand-nonrand}, as required.

% In order to obtain a high probability bound, we need to derive that which concludes the proof of \eqref{eq:bnd-one-theta}, as required.

% Thus, we have managed to prove that with high probability
% $$
% \|h - \lambda A y\| = \Omega(\|h - \lambda A\tanh((\beta^*/\lambda) h + \theta^* h)\|)
% $$
% , which means we can upper bound the rate for $\theta$ by 
% $$
% \frac{1}{\inf_{\lambda \in \R} \lp(\|\lambda A\|_F + \|h - \lambda A\tanh((\beta^*/\lambda) h + \theta^* h)\|\rp)}.
% $$
% Of course, in doing so we have to establish the high probability statement for all $\lambda$, which results in the loss of a log factor in the rate. 
% One might wonder whether using the mean-field approximation results in a tight bound. In Section REF, we show that this is indeed the case in high temperature. 

\subsection{Definitions of the terms in Theorem~\ref{thm:costis:general upper}}\label{sec:formal-defs}

We now sketch the proof of our general upper bound of Theorem~\ref{thm:costis:general upper}. We first define the notions of covering numbers and the quantities $\mathcal{C}_1$ and $\mathcal{C}_2$ in the theorem statement.

\begin{definition}\label{def:cover}
Given a metric space $(\Omega,d)$ and $\epsilon >0$, a subset $\Omega'\subseteq \Omega$ is an \emph{$\epsilon$-net} for $\Omega$ if for any $\omega\in \Omega$ there exists $\omega'\in \omega$ such that $d(\omega,\omega') \le \epsilon$. The covering number at scale $\epsilon$, $N(\Omega,\epsilon)$, is the smallest size of an $\epsilon$-net.
\end{definition}
For a function class $\mathcal{F}$ and collection of feature vectors $X=(x_1,\ldots,x_n)$, we denote by $N(\mathcal{F},X,\epsilon)$  the covering number at scale $\epsilon$ of $\cal F$ w.r.t.~the distance $d(f,g) = \sqrt{\|f(X)-g(X)\|_2^2/n}$, where we use the convenient notation $f(X)=(f(x_1),\ldots,f(x_n))$ and similarly for $g(X)$.

Next, we define the quantities $\mathcal{C}_1$ and $\mathcal{C}_2$. We start by defining the following  as a function of $\beta,\beta' \in \mathbb{R}$ and $h,h' \in \mathbb{R}^n$:

\begin{equation}\label{eq:defpsi}
\psi(h,\beta;h',\beta') = (\beta-\beta')^2\|A\|_F^2 + 
\lp\|h - h' + (\beta-\beta') A \vtanh\lp(\frac{\beta'}{\beta-\beta'}(h' - h) + h'\rp) \rp\|_2^2
\end{equation}

where $\vtanh((z_1,\dots,z_n)) = (\tanh(z_1),\dots,\tanh(z_n))$.
% Also, define $f_\theta(X) = (f_{\theta}(x_1),\dots,f_\theta(x_n))$. 
Now, for some universal constant $c \ge 0$, we define
\begin{equation}\label{eq:def-C1}
\mathcal{C}_1(\mathcal{F},X,\theta^*,\beta^*)
:=
\sup_{(\theta,\beta) \in \Theta \times [-B,B]}
\min\lp(\frac{\|f_\theta(X)-f_{\theta^*}(X)\|_2^2/n}{\psi(f_\theta(X),\beta;f_{\theta^*}(X),\beta^*)}, \|f_\theta(X)-f_{\theta^*}(X)\|_2^2/n\rp).
\end{equation}
Similarly, $\mathcal{C}_2$ is defined in an analogous way, by replacing  $\|f_\theta(X)-f_{\theta^*}(X)\|_2^2/n$ with $(\beta-\beta^*)^2$. 
Conveniently, we can use the following upper bound on $\mathcal{C}_1$:
\begin{equation}\label{eq:def-C1prime}
\mathcal{C}_1'(\mathcal{F},X,\theta^*,\beta^*):= \sup_{(\theta,\beta) \in \Theta\times[-B,B]} \frac{\|f_\theta(X)-f_{\theta^*}(X)\|_2^2/n}{\psi(f_\theta(X),\beta;f_{\theta^*}(X),\beta^*)}.
\end{equation}
At this point, we can explain how the rate in \eqref{eq:bnd-one-theta} for $d=1$ is derived from the bound of Theorem~\ref{thm:costis:general upper}. %First of all, notice that the complexity $\mathcal{C}_1$ bounds 
In this case, $(x_{1},\ldots, x_n)$ is simply a vector $x\in \R^n$ and $f_\theta(x_i) = \theta x_i$. Substituting $\mathcal{C}_1\le \mathcal{C}'_1$ into \eqref{eq:thm-gen-bnd}, substituting $f_\theta(x_i) = \theta x_i$ and substituting $\lambda = -(\beta - \beta^*)/(\theta - \theta^*)$, \eqref{eq:bnd-one-theta} follows.

\subsection{Sketch of the upper bound in Theorem~\ref{thm:costis:general upper}}\label{sec:proof-general}

Here, we  sketch the proof of the upper bound in Theorem~\ref{thm:costis:general upper}, but a weaker one where $\mathcal{C}_1$ is replaced by its upper bound $\mathcal{C}_1'$ defined in \eqref{eq:def-C1}.
In particular, we sketch that w.pr. $\ge 1-\delta$,

\begin{equation}
\label{eq:toprove-general}
\frac{1}{n}  \|f_{\hat\theta}(X)-f_{\theta^*}(X)\|_2^2
\lesssim \mathcal{C}_1'(\mathcal{F},X,\beta^*,\theta^*)\inf_{\epsilon \ge 0} \lp(\log \frac{n}{\delta} + \epsilon n + \log N(\mathcal{F},X,\epsilon)\rp).
\end{equation}

It is possible to prove that $\mathcal{C}_1'\le O(1/\|A\|_F^2)$, similarly to the corresponding argument in Section~\ref{sec:single-dim-upper} and we focus below on proving \eqref{eq:toprove-general}.

Notice that in the definition of $\mathcal{C}_1$ and $\mathcal{C}_1'$, we do not need the set $\mathcal{F}$ itself, but only the vectors $f_\theta(X)$ for every $\theta$ in the class $\mathcal{F}$. Hence, if we define the set $\mathcal{H} = \{f_\theta(X): f_\theta \in \mathcal{F}\}$, we immediately observe that $\mathcal{C}_1$ is in fact a function of $\mathcal{H}$.
In this setting, we can similarly define $h^* = f_{\theta^*}(X)$ and $\hat{h} = f_{\hat\theta}(X)$ and define the covering numbers $N(\mathcal{H},\epsilon)$ with respect to the distance $d(h,h') = \sqrt{\|h-h'\|_2^2/n}$. In this language, \eqref{eq:toprove-general} translates to

\begin{equation}
\label{eq:generalH}
\frac{1}{n}  \|\hat{h}-h^*\|_2^2 \lesssim
\mathcal{C}_1'(\mathcal{H},h^*,\beta^*)\inf_{\epsilon \ge 0} \lp(\log \frac{n}{\delta} + \epsilon n + \log N(\mathcal{H},\epsilon)\rp)\enspace.
\end{equation}

In the remainder of the proof, we will focus on proving \eqref{eq:generalH}, dividing the proof to multiple steps.

\paragraph{Step 1: A single dimensional $\mathcal{H}$.}
In this case, $\mathcal{H}$ is a single dimensional subspace of $\R^n$, namely, there exists $v\in \mathbb{R}^n$ such that $\mathcal{H} = \{h^*+t v \colon t \in \mathcal{T}\subseteq \mathbb{R}\}$. This is clearly reminiscent of the setting on a one-dimensional function-class discussed in Section~\ref{sec:single-dim-upper}.
Hence, using the exact same approach and using the calculation of Section~\ref{sec:formal-defs}, we can prove that w.pr. $1-\delta'$
$$
\frac{1}{n}  \|\hat{h}-h^*\|_2^2 \lesssim
\mathcal{C}_1'(\mathcal{H},h^*,\beta^*) \log\frac{n}{\delta '}.
$$
%  Further, assume that $\theta^* = 0$. Our goal is to bound $\|f_{\hat\theta}(X)-f_{\theta^*}(X)\|_2^2/n$.
% Notice that the linear form of $f_\theta$ is almost identical to the one in Section~\ref{sec:single-dim-upper}, with the difference being that $(x_1,\dots,x_n)$ is replaced with $(g(x_1),\ldots,g(x_n))$. 
% Hence, we can essentially use the same arguments (albeit with a slightly modified notation), to prove that w.pr. $1-\delta'$,
% \begin{equation*}
% \|f_{\hat\theta}(X)-f_{\theta^*}(X)\|_2^2/n
% \le
% \log\frac{n}{\delta'} \mathcal{C}_1(\mathcal{F},X,\theta^*,\beta^*)
% \end{equation*}
%Notice that the supremum over $\beta \in \mathbb{R}$ corresponds to the supremum over $\lambda$ in \eqref{eq:bnd-one-theta}.

\paragraph{Step 2: A union of single-dimensional classes.}
Now, suppose that we have a finite set of directions (unit vectors) $v_1,\ldots, v_N$ and denote
$\mathcal{H}_i = \{h^* + tv_i: h^* + tv_i \in \mathcal{H}\}$. In other words, $\mathcal{H}_i$ is the restriction of $\mathcal{H}$ on a specific line passing through $h^*$ with direction $v_i$. 
Suppose we run MPLE on each direction, producing an output $\hat h_i$ for each direction. 
The calculations of Step 1 suggest that for all $i \in [N]$, w.pr. $1  - \delta'$:
$$
\frac{1}{n}  \|\hat{h}_i-h^*\|_2^2 \lesssim
\mathcal{C}_1'(\mathcal{H},h^*,\beta^*) \log\frac{n}{\delta'}.
$$
With a simple union bound over these $N$ events, we can set $\delta ' = \delta/ N$ and obtain that w.pr.~$\ge 1 - \delta$, for all $i \in [N]$,
\begin{equation}\label{eq:directions}
\frac{1}{n}  \|\hat{h}_i-h^*\|_2^2 \lesssim
\mathcal{C}_1'(\mathcal{H},h^*,\beta^*) \lp(\log\frac{n}{\delta}+\log N\rp).
\end{equation}
This essentially means that, if we run MPLE on the original set $\mathcal{H}$ and it ends up lying in any of the ${\cal H}_i$'s, it will lie close to the optimal point $h^*$. 

Since we don't know in which direction the MPLE will lie, we have to establish a statement like \eqref{eq:directions} for all directions in $\mathcal{H}$. The problem is that usually there are infinity directions, so the union bound approach doesn't automatically work. 

However, we can approximate the set of directions by a finite subset  of directions that form an $\epsilon$-net. Since any point $h \in \mathcal{H}$ defines a direction $h -h^*$, we can take an $\epsilon$-net $\mathcal{U}$ with respect to $\mathcal{H}$, which has size $N=N(\mathcal{H},\epsilon)$, which corresponds to the covering number defined in Definition~\ref{def:cover} of Section~\ref{sec:formal-defs}. Due to Lipschitzness of the optimization target, one can prove that the MPLE over $\mathcal{U}$ is close to the MPLE over $\mathcal{H}$. By selecting $\epsilon$ appropriately and substituting $N=N(\mathcal{H},\epsilon)$ in \eqref{eq:directions}, we derive \eqref{eq:generalH}.

\begin{figure*}[!ht]
\begin{center}
\begin{tabular} {ccc}
  \includegraphics[width=0.3\textwidth]{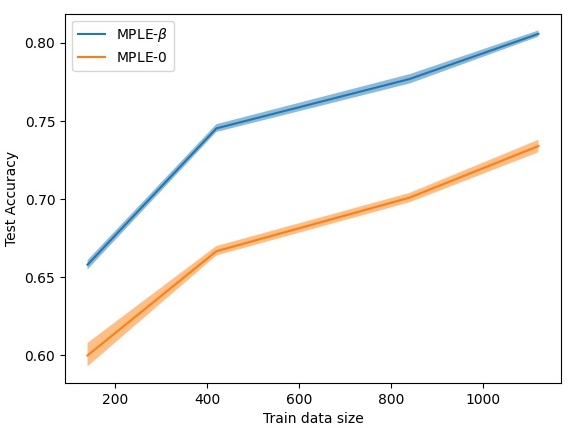} &   \includegraphics[width=0.3\textwidth]{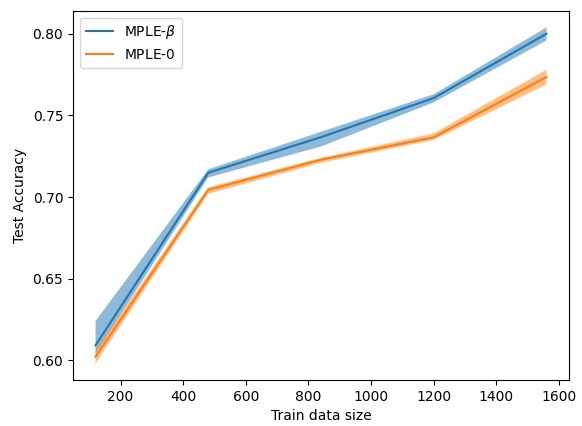} &    \includegraphics[width=0.3\textwidth]{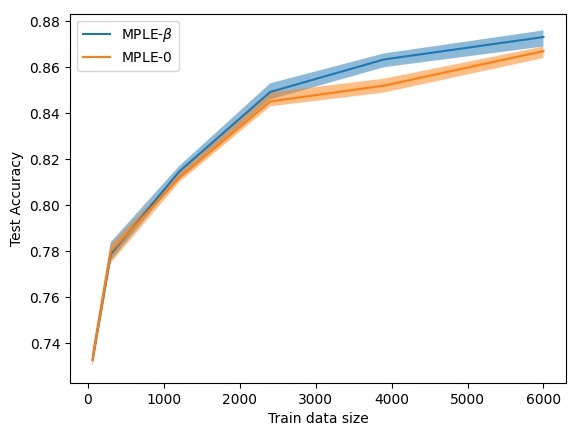} \\
  \end{tabular}
\end{center}
  \caption{From Left to Right: Plots of the accuracy of MPLE-$\beta$ (\textit{blue}) vs MPLE-$0$ (\textit{orange}) for Cora, Citeseer, Pubmed respectively as we increase the training data size gradually while maintaining the class probabilities.}
  \label{fig:increasing-training}
\end{figure*}

\section{Experiments} \label{sec:experiments}
When there is network information about dependencies between samples, we can use it to significantly boost the performance of supervised learning approaches. We demonstrate such improvements of MPLE-$\beta$ from including the dependency structure compared to assuming the data is i.i.d. We call this MPLE-$0$ (i.e. setting $\beta = 0$). We observe that MPLE-$\beta$ consistently outperforms MPLE-$0$ by a significant margin.

\noindent \textbf{Datasets.} We utilize three public citation datasets - Cora, Citeseer and Pubmed \cite{yang2016revisiting}. These datasets consist of a network where each node corresponds to a publication and the edges correspond to citation links. Each node contains a bag-of-words representation of the publication and a corresponding label indicating the area of the publication. Table \ref{tab:dataset} gives the statistics of the datasets.

\begin{table}[t]
\caption{Datasets: Cora and Citeseer have probability vectors as features. Pubmed has TF-IDF frequencies as features.}\label{tab:dataset}
\label{sample-table}
\vskip 0.15in
\begin{center}
\begin{small}
\begin{sc}
\begin{tabular}{lcccc}
\toprule
Dataset & Classes & Nodes  & Edges & Features\\
\midrule
% \hline
Cora & 7 & 2708 & 5429 & 1433 \\
Citeseer & 6 & 3327 & 4732 & 3703 \\
Pubmed & 3 & 19717 & 44338 & 500 \\
\bottomrule
\end{tabular}
\end{sc}
\end{small}
\end{center}
\vskip -0.1in
\end{table}

\noindent\textbf{Experimental Setup.} The datasets we use are common benchmarks used for semi-supervised and fully-supervised learning on graph structured data. The state of the art on a lot of these is graph neural network (GNN) \cite{chen2020simple} based approaches. The setups considered in prior literature on these datasets differ from ours in the following sense: these works consider the {\em transductive} setting, that is, they assume access to the adjacency matrix of the entire graph as well as the features of the entire dataset (including those in the test set) at train time. In contrast, we work in the {\em inductive} setting, where we do not assume access to any information about the test set. However, at test time, our hypothesis uses the labels in the validation set (not the features). 

We perform three different experiments on each dataset where we measure the accuracy of prediction on the test labels. We run each experiment with 10 fixed random seeds and report the average and standard deviation. 

% \begin{enumerate}[nosep]
1. \textit{Sparse-data}: Following the semi-supervised setup of \cite{kipf2016semi, feng2020graph} and others, we compare performance of MPLE-$0$ and MPLE-$\beta$ over a public split which includes only 20 nodes per class as training, 500 nodes for validation and 1000 nodes for testing.

2. \textit{Increasing training data}: We compare the gap in performance of the two methods when training data is gradually increased from the semi-supervised setting towards the full-supervised setting.

3. \textit{Full-supervised}: We consider the fully-supervised setup from \cite{pei2020geom}. In this setup, we consider 10 random splits of the entire dataset. Each split maintains class distribution by splitting the set of nodes of each class into 60\%(train)-20\%(val)-20\%(test). For this experiment, we compare against an inductive variant of GCNII we denote GCNII-In. We disable access to the test set features during training in order to have a fair comparison with our inductive setting. 
% \end{enumerate}

\textbf{Model Details.} Since our classification task is multi-class, we extend the MPLE-$\beta$ algorithm for Ising models to its natural Pott's model generalization. For number of classes $K$, the probability of label $y_i = k^*$ conditioned on the other data and labels is computed as follows:

\[
\Pr_{\theta,\beta}[y_i = k^* |x,y_{-i}]
= \frac{\exp\lp(f_{\theta}(x_i)_{k^*} + \beta \sum_{j=1}^n A_{ij}\mathbbm{1}[y_j = k^*]\rp)}{\sum_{k=1}^K\exp\lp(f_{\theta}(x_i)_k + \beta \sum_{j=1}^n A_{ij}\mathbbm{1}[y_j = k]\rp)}.
\]

Using this we compute the MPLE-$\beta$ objective.

For both MPLE-$0$ and MPLE-$\beta$, our underlying model $f_{\theta}:\mathbb{R}^{\mathsf{\#features}} \rightarrow \mathbb{R}^{\mathsf{\#classes}}$ is a 2-layer neural network with 32 units in the hidden layer and ReLU activations. The difference between the two models is just the use of $\beta$. For comparison with the graph neural networks (GNNs), we use the GCNII \cite{chen2020simple} model which is a state-of-the-art GNN with depth 64 and hidden layer size of 64. We run our code on a GPU and use Adam to train all our models. We use the tuned hyper-parameters for GCNII however for our algorithms we do not perform a hyper-parameter search but use the parameters used in prior work \cite{feng2020graph}.

% \yuval{Are you using a deep net or a logistic regression for the MPLE? What is NG-CNI using?}\\
% \yuval{note the unfinished potts model discussion}\\
% \yuval{what are the three different studies? Perhaps say "the first", "the second"....}\\
% \yuval{"Following the semi supervised..." Is this the second study?}

% \paragraph{Comparison to Graph Neural Networks.} 
% \yuval{What of the benchmarks that you use are GNNs? Is GCNII a GNN?}
% \nishanth{yes}
% \nishanth{add that its not fair to compare to GNNs for experiment 1 because they are transductive}

% \yuval{are some of our experiments compare to GNNs?}

\textbf{Results.}
On the sparse-data experiment, for Cora MPLE-$\beta$ gives an accuracy of $65.8\pm 0.09\% $ vs $60\pm0.4\%$ given by MPLE-$0$. For Citeseer, MPLE-$\beta$ gets $60.9\pm 0.7\%$ vs MPLE-$0$ which gets $60.2\pm 0.3\%$. For Pubmed, both approaches get $73.3\pm 0.2\%$. As we increase the train data size as shown in Figure~\ref{fig:increasing-training} our gains also tend to increase. Finally for the fully-supervised setting we again outperform MPLE-$0$ {\em and} GCNII-In.
On Pubmed, our gains are smaller as the TF-IDF feature vector already implicitly encodes some network information from the neighbors.
Moreover, MPLE-$\beta$ runs much faster than any of the GNN approaches and is simpler with a low overhead of a scalar parameter on any given model, while remaining competitive in performance. However, it should be noted
that we do not compare performance in the transductive setting, in which GCNII was probably intended to run. 
Finally, our experiments are based on an approach with provable end-to-end guarantees, in contrast with the GNN approaches. 
%In contrast to GNN approaches for semi-supervised learning which use test features at train time, our approach doesn't require the use of test features at train time.
%\nishanth{add the comparision for sparse-data experiment here in English}

\begin{table}[t]
\caption{Accuracy comparison between MPLE-$0$, MLPE-$\beta$ and GCNII-In for full-supervised experiment.}
% for a 2-layer MLP\yuval{What is an MLP? Doesn't logistic regression correspond to linear classification?}. The results are averaged over 10 runs each of 10 random splits (60\% train, 20\% validation and 20\% test). We also compare GCNII-In (GCNII run in the inductive setting) against our approach. Note that GCNII is originally trained in the transductive setting.}
\label{fig:results}
\vskip 0.15in
\begin{center}
\begin{small}
\begin{sc}
\begin{tabular}{lccc}
\toprule
Dataset & MPLE-0  & MPLE-$\beta$ &GCNII-In\\
\midrule
% \hline
Cora & $74.5\pm 1.8 $ & $\textbf{85.3}\pm 1.7$ & $\textbf{85.3}\pm 1.3$\\
Citeseer & $72.3\pm1.7 $ & $\textbf{76.3}\pm1.0$ & $68.6 \pm 0.3$\\
Pubmed & $87.3\pm 0.2$ & $\textbf{89.0}\pm 0.2$ & $83.3 \pm 0.6$ \\
\bottomrule
\end{tabular}
\end{sc}
\end{small}
\end{center}
\vskip -0.2in
\end{table}

\subsection*{Acknowledgements.}
Costis Daskalakis was supported by NSF Awards IIS-1741137, CCF-1617730, and CCF-1901292, by a Simons Investigator Award, by the Simons Collaboration on the Theory of Algorithmic Fairness, by a DSTA grant, and by the DOE PhILMs project (No. DE-AC05-76RL01830). Vardis Kandiros was supported by the Onassis Foundation-Scholarship ID: F ZP 016-1/2019-2020.

\appendix
\section{Preliminary: formulation as an Ising model}\label{sec:ising-formulation}
We start by repeating some of the definitions from Section~\ref{sec:formal-defs}, and proceed by presenting a modified notation, that we will use in the proof of Theorem~\ref{thm:costis:general upper}, as it is easier to handle.

\paragraph{Central definitions from Section~\ref{sec:formal-defs}.}
\begin{definition}
Let $h,h' \in \R^n$ and $\beta,\beta' \in \R$. 
 We define
 $$
 \psi(h,\beta;h',\beta') = (\beta-\beta')^2\|A\|_F^2 + 
\lp\|h - h' + (\beta-\beta') A \vtanh\lp(\frac{\beta'}{\beta-\beta'}(h' - h) + h'\rp) \rp\|_2^2
 $$
\end{definition}
Next, we define the quantities $\mathcal{C}_1,\mathcal{C}_2$ that appear in the rate of Theorem~\ref{thm:costis:general upper}.
\begin{definition}
For a function class $\mathcal{F} = \{f_\theta: \theta \in \Theta\}$ and a collection of feature vectors $X \in \R^{n\times d}$, where the $i$-th row of $X$ is the feature $x_i$, we define
\begin{align*}
&\mathcal{C}_1(\mathcal{F},X,\theta^*,\beta^*)
:=
\sup_{(\theta,\beta) \in \Theta \times [-B,B]}
\min\lp(\frac{\|f_\theta(X)-f_{\theta^*}(X)\|_2^2/n}{\psi(f_\theta(X),\beta;f_{\theta^*}(X),\beta^*)}, \|f_\theta(X)-f_{\theta^*}(X)\|_2^2/n\rp)\\
&\mathcal{C}_2(\mathcal{F},X,\theta^*,\beta^*)
:=
\sup_{(\theta,\beta) \in \Theta \times [-B,B]}
\min\lp(\frac{(\beta - \beta^*)^2}{\psi(f_\theta(X),\beta;f_{\theta^*}(X),\beta^*)}, (\beta - \beta^*)^2\rp)
\end{align*}
In some parts of the proof, it will be more convenient to work with the following simplified upper bounds to $\mathcal{C}_1$ and $\mathcal{C}_2$:
\begin{align*}
&\mathcal{C}_1'(\mathcal{F},X,\theta^*,\beta^*)
:=
\sup_{(\theta,\beta) \in \Theta \times [-B,B]}
\frac{\|f_\theta(X)-f_{\theta^*}(X)\|_2^2/n}{\psi(f_\theta(X),\beta;f_{\theta^*}(X),\beta^*)}\\
&\mathcal{C}_2'(\mathcal{F},X,\theta^*,\beta^*)
:=
\sup_{(\theta,\beta) \in \Theta \times [-B,B]}
\frac{(\beta - \beta^*)^2}{\psi(f_\theta(X),\beta;f_{\theta^*}(X),\beta^*)}
\end{align*}

\end{definition}

\begin{definition}
Given $\epsilon > 0$, $\mathcal{F}$ and $X$, denote by $\mathcal{N}(\epsilon,\mathcal{F})$ the $\epsilon$-covering number of $\mathcal{F}$ with respect to the distance 
\[
d(f,f') = ||f(X)-f'(X)||/\sqrt n
= \sqrt{\frac{1}{n} \sum_{i=1}^n (f(x_i)-f'(x_i))^2}.
\]
In other words, $\mathcal{N}(\epsilon,\mathcal{F})$ is the minimal cardinality of a set $\mathcal{G}\subseteq \mathcal{F}$, such that for any $f \in \mathcal F$ there exists $g \in \mathcal G$ such that
$
d(f,f') \le \epsilon.
$
\end{definition}

\paragraph{A modified notation for the proof.}
To prove Theorem~\ref{thm:costis:general upper}
we use a slightly modified setting, replacing the family of vectors $\{f_\theta(X) \colon \theta \in \Theta\}$ with a family $\mathcal{H}$ of elements of $\mathbb{R}^n$.
This comes from the realization that the only way a function $f_\theta$ influences the outcome is thought the vector $(f_\theta(x_1),\ldots,f_\theta(x_n)) \in \R^n$. Hence, it is more convenient to consider the set of all these vectors that are produced for various $\theta$ as our main object of interest. 
We have a fixed matrix $A$ with $\|A\|_\infty = 1$. We also replace $(y_1,\dots,y_n)$ by a vector $\sigma \in \{-1,1\}^n$ and for any $h \in \mathbb{R}^n$ and $\beta \in \mathbb{R}$, denote
\[
\Pr_{h,\beta}[\sigma] = 
\frac{1}{Z_{h,\beta}}\exp\left(\beta \sigma^{\rm T} A \sigma + \sum_{i=1}^n \sigma_i h_i  \right).
\]
We can also write the MPLE Eq.~\eqref{eq:MPLE} in this language. Define the negative log of the optimized quantity by
\[
\phi(h,\beta;\sigma) := -\log \prod_{i=1}^n \Pr_{h,\beta}[\sigma_i |\sigma_{-i}]
\]
and by $(\hat h,\hat\beta)$ the MPLE, namely,
\[
(\hat h,\hat\beta)
= \arg\min_{h \in \mathcal{H},\beta \in [-B,B]}
\phi(h,\beta;\sigma).
\]
We define $\mathcal{C}_1$ and $\mathcal{C}_2$ analogously with respect to $\mathcal{H}$ instead of $\mathcal{F}$:
\begin{definition}
Suppose $\mathcal{H} \subseteq \R^n$ and let $\beta^* \in \R,h^*\in \R^n$ be fixed. We define
\begin{align*}
&\mathcal{C}_1(\mathcal{H},h^*,\beta^*)
:=
\sup_{(h,\beta) \in \mathcal{H} \times [-B,B]}
\min\lp(\frac{\|h - h^*\|_2^2/n}{\psi(h,\beta;h^*,\beta^*)}, \|h - h^*\|_2^2/n\rp)\\
&\mathcal{C}_2(\mathcal{H},h^*,\beta^*)
:=
\sup_{(h,\beta) \in \mathcal{H} \times [-B,B]}
\min\lp(\frac{(\beta - \beta^*)^2}{\psi(h,\beta;h^*,\beta^*)}, (\beta - \beta^*)^2\rp)
\end{align*}
\end{definition}
Similarly, we define $\mathcal N(\epsilon, \mathcal H)$:
\begin{definition}
Given $\mathcal{H}$ and $\epsilon > 0$, denote by $\mathcal{N}(\epsilon,\mathcal H)$ the $\epsilon$-covering number of $\mathcal{H}$ with respect to the distance $d(h,h') = \|h-h'\|/\sqrt{n}$.
\end{definition}
In various parts of the proof, we will use symbols like $c,C,C'$. These denote constants that possibly depend exponentially in $M$ and $B$ and are independent of the other parameters.

\paragraph{Re-writing the main theorem in the modified notation.}
We can rewrite Theorem~\ref{thm:costis:general upper} in the modified notation. We split it to Theorem~\ref{t:general} and Theorem~\ref{thm:general-lower}, proving the upper and lower bounds, respectively.

The following theorem is proved in Section~\ref{sec:proof-upper}:
\begin{theorem}\label{t:general}
Let $A$ be a matrix of $\|A\|_\infty = 1$, let $M,B>0$, let $\mathcal{H} \subseteq [-M,M]^n$ and let $\beta^* \in [-B,B]$. 
Let $(\hat h, \hat{\beta})$ denote the MPLE over $\mathcal{H}\times [-M,M]$. Then, there is a constant $C(M,B)$ that depends singly exponentially on $M,B$, such that for any $\delta \in (0,1/2)$, with probability at least $1-\delta$, we have
\begin{align*}
\frac{\|\hat h - h^*\|^2}{n} \leq  C(M,B)\inf_{\epsilon\ge 0}\lp(\log \frac n\delta + \epsilon n  + \log \mathcal{N}\lp(\epsilon,  \mathcal{H}\rp)\rp)\mathcal{C}_1(\mathcal{H},h^*,\beta^*)
%\inf_{\lambda \in \R, h \in \mathcal{H}}\lp(\|\lambda A\|_F + \lp\|\frac{h^* - h}{\|h^* - h\|} - \lambda A \tanh\lp(\frac{\beta^*}{\lambda}\frac{h^* - h}{\|h^* - h\|} + h^*\rp)\rp\|\rp)
\end{align*}
and 
\begin{align*}
(\hat \beta - \beta^*)^2 \leq  C(M,B)\inf_{\epsilon\ge 0}\lp(\log \frac n\delta + \epsilon n  + \log \mathcal{N}\lp(\epsilon, \mathcal{H}\rp)\rp)\mathcal{C}_2(\mathcal{H},h^*,\beta^*).
%\inf_{\lambda \in \R, h \in \mathcal{H}}\lp(\|A\|_F + \lp\|\frac{1}{\lambda}\frac{h^* - h}{\|h^* - h\|} - A \tanh\lp(\frac{\beta^*}{\lambda}\frac{h^* - h}{\|h^* - h\|} + h^*\rp)\rp\|\rp)
\end{align*}
\end{theorem}

The following theorem is proved in Section~\ref{sec:lowerbound}:
\begin{theorem}\label{thm:general-lower}
Let $A$ be a matrix of $\|A\|_\infty = 1$, let $M,B>0$, let $\mathcal{H} \subseteq [-M,M]^n$ and let $\beta^* \in [-B,B]$. Assume that $\mathcal{H}$ is convex and closed under negation, namely, for any $h,h' \in \mathcal{H}$ and $\lambda \in (0,1)$ it holds that $\lambda h + (1-\lambda) h' \in \mathcal{H}$ and $-h \in \mathcal{H}$.
Then, for any estimator $h'$ there exists $(h^*,\beta^*) \in \mathcal{H} \times [-B,B]$ s.t. with probability $\ge 0.49$ over $\sigma \sim \Pr_{h^*,\beta^*}$,
\[
\frac{\|h'-h^*\|^2}{n}
\ge c(M) \mathcal{C}_1(\mathcal{H},h^*,\beta^*)
\]
where $C(M)$ depends exponentially on $M$.
\end{theorem}

We also have the following simple Lemma that upper bounds $\mathcal{C}_1,\mathcal{C}_2$ by a simpler quantity, and proved in Section~\ref{sec:ub-simpler}.

\begin{lemma}\label{lem:C-to-Forb}
It holds that
\[
\mathcal{C}_1(\mathcal{H},h^*,\beta^*)
\le \frac{1}{4\|A\|_F^2}.
\]
\end{lemma}

\paragraph{Proof of Theorem~\ref{thm:costis:general upper}.}
We now explain how we can easily use Theorem~\ref{t:general}, Theorem~\ref{thm:general-lower} and Lemma~\ref{lem:C-to-Forb} to get Theorem~\ref{thm:costis:general upper}.

\begin{proof}[Proof of Theorem~\ref{thm:costis:general upper}]
For the upper bound, we will apply Theorem~\ref{thm:costis:general upper} by setting $\mathcal{H} = \{(f_\theta(x_1),\ldots,f_\theta(x_n)) : \|\theta\| \leq M\}.$
First of all, we show that $N(\mathcal{F}, X, \epsilon) = N(\mathcal{H},\epsilon)$. Indeed, by the definition of $\mathcal{H}$ we can see how the elements of $\mathcal{F}$ can be in one-to-one correspondence with the elements of $\mathcal{F}$. Specifically, some $f_\theta \in \mathcal{F}$ corresponds to $h = (f_\theta(x_1),\ldots,f_\theta(x_n)) \in \mathcal{H}$.
As for the metric, if $h_1,h_2$ correspond to $f_{\theta_1},f_{\theta_1}$, we have
\begin{align*}
d(f_{\theta_1},f_{\theta_2}) = \sqrt{\sum_{i=1}^n (f_{\theta_2}(x_i)-f_{\theta_1}(x_i))^2/n} = \frac{\|h_2 - h_1\|}{\sqrt n}
\end{align*}
Thus, we see that the metric $d$ corresponds exactly to the metric used to define $N(\mathcal{H},\epsilon)$ in Section~\ref{sec:proof-general}. It follows that $N(\mathcal{F}, X, \epsilon) = N(\mathcal{H},\epsilon)$. 

Since optimizing in the space $\mathcal{H}$ is equivalent to optimizing in the space of $\theta$'s, we have $\hat h = f_{\hat\theta}(X)$ and $h^* = f_{\theta^*}(X)$. To conclude the proof of the general rate, we need to argue that $\mathcal{C}_1(\mathcal{H},h^*,\beta^*)= \mathcal{C}_1(\mathcal{F}, X, \theta^*,\beta^*)$. Since in the definition of $\mathcal{C}_1(\mathcal{H},h^*,\beta^*)$ we just substituted $h = f_\theta(X)$ and to create the set $\mathcal{H}$ we did the same thing, these two quantities are clearly the same. A similar argument can be made for $\mathcal{C}_2(\mathcal{H},h^*,\beta^*)$.

To conclude the proof of the upper bound in Theorem~\ref{thm:costis:general upper}, we simply use Lemma~\ref{lem:C-to-Forb} to bound the quantities $\mathcal{C}_1, \mathcal{C}_2$. 
The lower bound similarly follows from Theorem~\ref{thm:general-lower}.
\end{proof}

\section{Proof of the upper bound in Theorem~\ref{thm:costis:general upper}}\label{sec:proof-upper}
\subsection{Proof Overview}
As was shown in Section~\ref{sec:ising-formulation}, it suffices to prove Theorem~\ref{t:general}.
The proof can be broken down into several lemmas. We first present a lemma that shows that for a single $h \in \mathcal{H}$, if $h$ is far from $h^*$ it is unlikely that the MPLE will return $h$.

\begin{lemma}\label{l:taylor}
Let $h \in [-M,M]^n$ be a fixed vector and $\beta \in [-B,B]$. Then, there is a constant $C=C(M,B)>0$ such that for all $u>0$, with probability at least $1 - \log n\me^{-u^2}$, if 
\begin{equation}\label{eq:taylor-cond}
\|(\beta-\beta^*) A\sigma + h - h^*\|_2 \geq Cu
\end{equation}
then 
\begin{equation}\label{eq:greater-than}
\phi(h,\beta;\sigma) > \phi(h^*,\beta^*;\sigma) + cu^2
\end{equation}
\end{lemma}

\begin{proof}
First, denote 
$$
 \mathbf{a} = (h - h^*,\beta-\beta^*)
$$
where $\mathbf{a}\in \mathbb{R}^{n+1}$.
Define the function $g:[0,1] \to \R$ as
$$
g(t) = \phi(h^* + t (h-h^*) , \beta^* + t(\beta-\beta^*); \sigma) = \phi((h^*,\beta^*) + t \va; \sigma)
$$
For simplicity, from now on  we suppress the dependence of $\phi$ on the sample $\sigma$ since it is obvious.

Consider the Taylor expansion of $g$ around $0$. Notice also that $g(0) = \phi(h^*,\beta^*)$.
We have
\begin{align*}
    g(t) &= g(0) + t g'(0) + \frac{t^2}{2} g''(t')\\
    &=\phi(h^*,\beta^*) + t \va^\top \nabla\varphi(h^*,\beta^*)
    + \frac{t^2}{2} \va^\top \nabla^2\varphi(h',\beta')\va,
\end{align*}
% \begin{align*}
% g(t,\beta) &= g(0,0) + t \frac{\partial g}{\partial t}(0,0) + \beta \frac{\partial g}{\partial \beta}(0,0) + \frac{1}{2} (t,\beta)^\top D^2g(t',\beta') (t,\beta)\\
% &= \phi(h^*,\beta^*) + (t\mathbf{a},\beta)^\top \nabla \phi(h^*,\beta^*) + \frac{1}{2} (t\mathbf{a},\beta)^\top \nabla^2 \phi(t',\beta')(t\mathbf{a},\beta)
% \end{align*}
where $t' \in [0,t]$ and $(h',\beta') = (h^*,\beta^*)+t'\va$.
% where $(t\mathbf{a})$ denotes a $d+1$ dimensional vector and $t',\beta'$ are the mean value points. 
It becomes clear that in order to control the quantity on the right, we need to understand the relation of the first and the second order terms of the Taylor. To do that, we extend some concentration bounds from \cite{dagan2020estimating} involving the first and second derivatives, as follows (proof in Section~\ref{sec:from-prior}):
\begin{lemma}\label{lem:from-previous}
Let $t \ge 0$, $\beta\in\mathbb{R}$ and $h\in \mathbb{R}^n$. Then, with probability at least $1-\log n\exp(-ct^2)$, if
\[
\|\beta A\sigma+h\|_2 \ge C t|\beta| \|A\|_2
\]
then
\[
\frac{\left|(h,\beta)^\top \nabla\varphi(h^*,\beta^*)\right|}{\min_{\beta' \in [-M,M],h'\in [-M,M]^n}(h,\beta)^\top \nabla^2 \phi(h',\beta')(h,\beta)} \le \frac{c't}{\|\beta A\sigma+h\|_2}
\]
\end{lemma}
Using the result of this lemma, we get that with probability $1 - \log n \me^{-u^2}$, if
\begin{align}\label{eq:22}
    \|(\beta-\beta^*)A\sigma + h-h^* \|_2 \ge Cu\|A\|_2|\beta-\beta^*|,
\end{align}
then
\begin{align}\label{eq:23}
    \frac{|\va^\top \nabla \varphi(J^*,h^*)|}{\va^\top \nabla^2 \varphi(h',\beta')\va} \le \frac{Cu}{\|(\beta-\beta^*)A\sigma + h-h^* \|_2}.
\end{align}
Notice that \eqref{eq:taylor-cond} implies \eqref{eq:22} (if the constant $C>0$ in \eqref{eq:22} is sufficiently large). Then, with probability $1-\log n \me^{-u^2}$, \eqref{eq:taylor-cond} implies \eqref{eq:23}.
Hence, it suffices to prove that \eqref{eq:23} implies \eqref{eq:greater-than}. For the remainder of the proof, assume that \eqref{eq:23} holds.
% \begin{align*}
% \frac{(t\mathbf{a},\beta)^\top \nabla^2 \phi(t',\beta')(t\mathbf{a},\beta)}{ |(t\mathbf{a},\beta)^\top \nabla \phi(h^*,\beta^*)|} &\geq C \frac{\|(\beta-\beta^*)  A\sigma + t \mathbf{a}\|}{u}
% \end{align*}
Substituting back to the Taylor expression, we get that
\begin{align*}
\phi(h, \beta)- \phi(h^*,\beta^*)
= g(1) + g(0)
\ge \frac{1}{2}\va^\top \nabla^2\varphi(h',\beta')\va
- \lp|\va^\top \nabla\varphi(h^*,\beta^*)\rp| \\
\ge \va^\top \nabla^2\varphi(h',\beta')\va \lp(\frac{1}{2} - \frac{C'u}{\|(\beta-\beta^*)  A\sigma + h-h^*\|}\rp)
\ge \frac{1}{4}\va^\top \nabla^2\varphi(h',\beta')\va,
\end{align*}
where the last inequality holds if the constant $C$ in \eqref{eq:taylor-cond} is sufficiently large.
To bound the right hand side, we use the following lemma
that will be proven in Section~\ref{sec:from-prior} using the techniques of \cite{dagan2020estimating}:
\begin{lemma}\label{lem:trivial-hessian}
Let $\beta' \in [-B,B]$, $h'\in [-M,M]^n$, $\beta \in \mathbb{R}$ and $h \in \mathbb{R}^n$.
Then,
\[
(h,\beta)^\top \nabla^2 \varphi(h',\beta') (h,\beta)
\ge c \|\beta A\sigma+h\|_2.
\]
\end{lemma}

As a consequence of Lemma~\ref{lem:trivial-hessian} and \eqref{eq:taylor-cond},
\[
\va^\top \nabla^2\varphi(h',\beta')\va \ge c \|(\beta-\beta^*)A\sigma + h-h^*\|_2^2 \ge cu^2,
\] 
which proves \eqref{eq:greater-than} and concludes the proof.
% This implies that if 
% $$
% \|\beta  A\sigma + t \mathbf{a}\| = \Omega(u)
% $$
% then  with probability $1 - \me^{-u^2}$
% $$\phi(h^* + t\mathbf{a}, \beta^* + \beta) > \phi(h^*,\beta^*)
% $$
% \vardis{this requires the second derivative to be positive, which holds with high probability}
% Setting $t = \|h - h^*\|$ concludes the claim.
\end{proof}

After establishing Lemma~\ref{l:taylor}, it is useful to examine what guarantees we get from it. If we take the contrapositive of the statement, it means that if the MPLE algorithm returns the vector $h$ as an estimate, then with high probability Eq.~\eqref{eq:taylor-cond} holds. However, this inequality is a random quantity that depends on the particular instance of $\sigma$ that we sample. Thus, we want to replace it with a nonrandom inequality, which is the content of the next lemma. We use the quantity $\psi(h,\beta;h',\beta')$ defined in \eqref{eq:taylor-cond} and write use the shorthand notation 
$$\psi(h,\beta) := \psi(h,\beta;h^*,\beta^*).
$$

\begin{lemma}\label{l:substitute}
Let $h\in \mathbb{R}^n$ and $\beta \in \mathbb{R}$.
% Denote 
% \[
% \psi(h,\beta) := |\beta-\beta^*|\|A\|_F + 
% \lp\|h^* - h - (\beta-\beta^*) A \tanh\lp(\frac{\beta^*}{\beta-\beta^*}(h^* - h) + h^*\rp) \rp\|_2\]
% Then, there are constants $C,c$ depending on $M$, such that 
% with probability
With probability
$1 - C \log n \me^{-c\psi(h,\beta)/(\beta-\beta^*)\|A\|_2^2}$,
$$
\|h^* - h - (\beta-\beta^*) A\sigma \|_2^2 \ge c \psi(h,\beta).
$$
\end{lemma}
The proof of 
Lemma~\ref{l:substitute} is divided into two parts. Since $\psi(h,\beta)$ involves two terms, each part consists of showing that the normed quantity is greater than each of the two terms. 
For the first term, we have the following Lemma.
\begin{lemma} \label{lem:bnd-by-Forb}
% Suppose $\sigma$ is sampled from \eqref{eq:model}. Let $h \in \R^n$ be a fixed vector and $\beta \in \R$, such that $|\beta|,\|h\|_\infty \le 2M$.
Let $h\in \mathbb{R}^n$ and $\beta \in \mathbb{R}$.
    Then, with probability at least $1-\log n\exp(-c \|A\|_F^2/\|A\|_2^2)$
    \[
    \|h^* - h + (\beta^*-\beta) A\sigma \|_2 \ge c |\beta-\beta^*|\|A\|_F.
    \]
\end{lemma}
This is proved in Section~\ref{sec:from-prior}, and follows from arguments from \cite{dagan2020estimating}.
The idea is to use concentration bounds about second degree polynomials of an Ising model. These immediately yield the required dependence on $\|A\|_F$.

For the second part of the proof of Lemma~\ref{l:substitute}, we need the following Lemma. 
\begin{lemma} \label{lem:bnd-by-mean}
Let $h\in \mathbb{R}^n$ and $\beta \in \mathbb{R}$.
    Then, with probability at least
    $$
    1 - \exp\lp(- c \frac{\lp\|h^* - h + (\beta^*-\beta) A \tanh\lp(\frac{\beta^*}{\beta-\beta^*}(h^* - h) + h^*\rp) \rp\|_2^2}{(\beta-\beta^*)^2\|A\|_2^2} \rp)
    $$ 
    we have that
    \[
    \|h^* - h + (\beta^*-\beta) A\sigma \|_2 
    \ge c \lp\|h^* - h + (\beta^*-\beta) A \tanh\lp(\frac{\beta^*}{\beta-\beta^*}(h^* - h) + h^*\rp) \rp\|_2.
    \]
\end{lemma}

Before proving Lemma~\ref{lem:bnd-by-mean}, we need some auxiliary lemma about the concentration of a function of the Ising model. This lemma is proved in Section~\ref{sec:chatterjee} using the technique of exchangeable pairs \cite{chatterjee2005concentration}.  
\begin{lemma}\label{lem:conc-from-cha}
Let $\sigma \sim \Pr_{h^*,\beta^*}$, let $v \in \mathbb{R}^n$ and let
\[
f(\sigma)
= \sum_{i=1}^n v_i(\sigma_i - \tanh(\beta^* (A\sigma)_i + h^*_i).
\]
Then, for all $t\ge 0$,
\[
\Pr[|f(\sigma)|>t]
\le 2 \exp\lp(-\frac{t^2}{8 \|v\|_2^2 \lp(1 + |\beta^*| \|A\|_\infty \rp)}\rp).
\]
\end{lemma}

\begin{proof}[Proof of Lemma~\ref{lem:bnd-by-mean}]
Denote $\tilde{\beta}=\beta-\beta^*$.
First of all, notice that we can rewrite the quantity of interest as 
$$
\| h^* - h - \tilde{\beta} A \sigma\|
$$
The strategy for bounding this quantity will be to use the mean field approximation for the Ising model. Of course, we have to show that this is a good enough approximation in our case. 
Let $\mathbf{u} \in \R^n$ be an arbitrary vector with $\|\mathbf{u}\| = 1$.
We start by using the triangle inequality to write
\begin{align*}\label{eq:meanfield}
\lp|\mathbf{u}^\top \lp(h^* - h - \tilde{\beta} A \tanh\lp(\frac{\beta^*}{\tilde{\beta}}(h^* - h) + h^*\rp)\rp) \rp| &\leq |\mathbf{u}^\top (h^* - h - \tilde{\beta} A \sigma)| + |\mathbf{u}^\top\lp(\tilde{\beta} A \sigma - \tilde{\beta} A \tanh\lp(\beta^* A \sigma + h^*\rp)\rp)| +\\
&\quad\quad\quad + \lp|\mathbf{u}^\top\lp(\tilde{\beta} A \tanh\lp(\beta^* A \sigma + h^*\rp)- \tilde{\beta} A \tanh\lp(\frac{\beta^*}{\tilde{\beta}}(h^* - h) + h^*\rp)\rp)\rp|
\end{align*}
We will bound each of these three terms separately. 
For the first term, using the Cauchy-Schwarz inequality, we obtain
\begin{equation}
|\mathbf{u}^\top (h^* - h - \tilde{\beta} A \sigma)| \leq \|\mathbf{u}\| \|h^* - h - \tilde{\beta} A \sigma\| = 
\|h^* - h - \tilde{\beta} A \sigma\|
\end{equation}

We proceed with bounding the second term. Define $\mathbf{v} = \tilde{\beta} A^\top \mathbf{u}$. Then this term can be written as
$$
|\mathbf{v}^\top (\sigma - \tanh(\beta^* A\sigma +h^*))|
= \lp|\sum_{i=1}^n v_i (\sigma_i - \tanh(\beta^* (A\sigma)_i +h^*_i ))\rp|.
$$
We use Lemma~\ref{lem:conc-from-cha} to derive that for any $t\ge 0$, this term can be bounded by $Ct\|v\|_2\sqrt{1+|\beta^*|\|A\|_\infty}$ with probability $e^{-ct^2}$. Using the facts that $|\beta^*|,\|A\|_\infty \le O(1)$ and \[
\|v\|_2 = \|\tilde{\beta} u^\top A\|_2 \le |\tilde{\beta}|\|u\|_2\|A\|_2 = |\tilde{\beta}|\|A\|_2
\]
to derive that w.p. $1-e^{-t^2}$,
\[
|\mathbf{v}^\top (\sigma - \tanh(\beta^* A\sigma +h^*))|
\le Ct |\tilde{\beta}|\|A\|_2.
\]

For the third term, we use Cauchy Schwarz inequality and
the lipschitzness of $\tanh$ to conclude that
\begin{align*}
\lp|\mathbf{u}^\top\lp(\tilde{\beta} A \tanh\lp(\beta^* A \sigma + h^*\rp)- \tilde{\beta} A \tanh\lp(\frac{\beta^*}{\tilde{\beta}}(h^* - h) + h^*\rp)\rp)\rp| 
&\leq \beta^* \|A\|_\infty \|\mathbf{u}\|\|h^* - h - \tilde{\beta} A\sigma\|\\
&\le C \|h^* - h - \tilde{\beta} A\sigma\|,
\end{align*}
using the fact that $\|u\|_2=1$ and that $|\beta^*|,\|A\|_\infty \le 1$.
Putting everything together, we conclude that
for every $\mathbf{u} \in \S^{n-1}$, with probability at least $1 - \me^{-t^2}$ 
$$
\lp|\mathbf{u}^\top \lp(h^* - h - \tilde{\beta} A \tanh\lp(\frac{\beta^*}{\tilde{\beta}}(h^* - h) + h^*\rp)\rp) \rp| 
\le C't|\tilde{\beta}|\|A\|_2 + C'\|h^* - h - \tilde{\beta} A\sigma\|
$$
Setting
$$
\mathbf{u} = \frac{h^* - h - \tilde{\beta} A \tanh\lp(\frac{\beta^*}{\tilde{\beta}}(h^* - h) + h^*\rp)}{\lp\|h^* - h - \tilde{\beta} A \tanh\lp(\frac{\beta^*}{\tilde{\beta}}(h^* - h) + h^*\rp) \rp\| },
$$
which is a deterministic vector,
we get that with probability $1 - \me^{-t^2}$
\begin{align*}
&\lp\|h^* - h - \tilde{\beta} A \tanh\lp(\frac{\beta^*}{\tilde{\beta}}(h^* - h) + h^*\rp) \rp\| 
\le C\|h^* - h - \tilde{\beta} A\sigma\| + Ct|\tilde{\beta}|\|A\|_2\implies\\
&\|h^* - h - \tilde{\beta} A\sigma\|
 \ge c_1 \lp\|h^* - h - \tilde{\beta} A \tanh\lp(\frac{\beta^*}{\tilde{\beta}}(h^* - h) + h^*\rp) \rp\| - C_2 t|\tilde{\beta}|\|A\|_2 
\end{align*}
where $c_1,C_2$ are constants. By substituting $t = c \lp\|h^* - h - \tilde{\beta} A \tanh\lp(\frac{\beta^*}{\tilde{\beta}}(h^* - h) + h^*\rp) \rp\| / |\tilde{\beta}|\|A\|_2$ for a sufficiently small constant $c>0$, the proof follows.
% At the same time,
% by REF we have that with  probability at least $1 - \me^{-ct^2}$
% we have that
% $$
% \|h^* - h - \beta A\sigma \| \ge C_3 \|\beta A\|_F
%  - t C_4\left(\|A\|_F + \|\E[h^* - h - \beta A\sigma \mid x_{-I_j}]\|_2\right)
% $$
% for some constants $C_3, C_4$. By taking a suitable linear
% combination of the above two inequalities, we conclude that there exists a constant $K$ such that with probability $1 - C \log n \me^{-ct^2}$
% $$
% \|h^* - h - \beta A\sigma \| \ge K \lp(|\beta|\|A\|_F + 
% \lp\|h^* - h - \beta A \tanh\lp(\frac{\beta^*}{\beta}(h^* - h) + h^*\rp) \rp\|\rp)
%  - t C_4\left(\|A\|_F + \|\E[h^* - h - \beta A\sigma \mid x_{-I_j}]\|_2\right).
% $$
% The result follows.
\end{proof}

Given Lemmas~\ref{lem:bnd-by-Forb} and ~\ref{lem:bnd-by-mean}, we can finally prove Lemma~\ref{l:substitute}.
\begin{proof}[Proof of Lemma~\ref{l:substitute}]
We divide into cases: if \[|\beta-\beta^*|\|A\|_F \ge \lp\|h^* - h - (\beta-\beta^*) A \tanh\lp(\frac{\beta^*}{\beta-\beta^*}(h^* - h) + h^*\rp) \rp\|_2\]
then the proof follows directly from Lemma~\ref{lem:bnd-by-Forb}, otherwise it follows directly from Lemma~\ref{lem:bnd-by-mean}.
\end{proof}

Using Lemma~\ref{l:substitute} we can show that if $\psi(h,\beta)$ is large, then the pseudo-likelihood value for $(h,\beta)$ will be far from optimal. This is the first indication that maximizing the likelihood might give us meaningful information about the parameters.
\begin{lemma}\label{lem:one-direction-final}
%Suppose $\sigma$ is sampled from \eqref{eq:model}. Let $h \in \R^n, \beta \in \R$ be fixed, such that $|\beta|,|\beta^*|,\|h\|_\infty,\|h^*\|_\infty \le M$.
Let $h \in [-M,M]^n$ and $\beta \in [-B,B]$.
Then, with probability $1-\log n\, \me^{-c\psi(h,\beta)}$,
\[
\varphi(h,\beta) \ge \varphi(h^*,\beta^*)+c \psi(h,\beta).
\]
It follows that for any $u>0$ with $u^2 \leq \psi(h,\beta)$ we have with probability at least $1 - \log n\, \me^{-cu^2}$
\[
\varphi(h,\beta) \ge \varphi(h^*,\beta^*)+c u^2.
\]
\end{lemma}
\begin{proof}
First of all, we have from Lemma~\ref{l:substitute} that with probability at least
\[
1 - C \log n\, \me^{-c\psi(h,\beta)/(\beta-\beta^*)^2\|A\|_2^2} 
\ge 1 - C \log n\, \me^{-c'\psi(h,\beta)},
\]
it holds that
\begin{equation}\label{eq:42}
\|h^* - h - (\beta-\beta^*) A\sigma \|_2 \ge c \psi(h,\beta).
\end{equation}
Further, from Lemma~\ref{l:taylor}, with probability $1-\log n\, \exp(-c''\psi(h,\beta))$, if \eqref{eq:42} holds then $\varphi(h,\beta) \ge \varphi(h^*,\beta^*) + c'' \psi(h,\beta)$. This concludes the proof of the first claim. For the second claim, since $u^2 \leq \psi(h,\beta)$, we have $1 - \log n\, \exp(-c\psi(h,\beta)) \geq 1 - \log n\, \exp(-cu^2)$, while at the same time 
$$
\varphi(h,\beta) \ge \varphi(h^*,\beta^*)+c \psi(h,\beta) \implies  
\varphi(h,\beta) \ge \varphi(h^*,\beta^*)+c u^2.
$$
\end{proof}

So far, we have focused on a specific $(h,\beta)$ in the space of possible solutions to the MPLE problem. Next, we would like to show that the MPLE will satisfy this inequality. In order to do so, we would like to prove a high probability statement for all such pairs from $\mathcal{H}\times [-B,B]$. To prove such a statement, we need to make use of the properties of the metric space induced by $\mathcal{H}$ and $\|\cdot\|$, as stated in the following lemma:
%We have shown that a large value for $\psi(h,\beta)$ implies that $(h,\beta)$ will be far from the optimal solution. However, we haven't showed what it means for $\psi(h,\beta)$ to have a large value. Ideally, we would like $\psi(h,\beta)$ when the parameters are far in distance from the optimal ones. 
%We would also like to prove a high probability statement for all such directions. To prove such a statement, we need to make use of the properties of the metric space induced by $\mathcal{H}$ and $\|\cdot\|$. The following lemma makes use of the special structure of $\mathcal{H}$ to derive a bound for the MPLE. 
\begin{lemma} \label{lem:upperbd-kappa}
% Suppose $\sigma$ is sampled from \eqref{eq:model}. 
% Let $(\hat{\beta},\hat{h})$ denote the MPLE over $\mathcal{H}\times [-M,M]$.
With probability $1-\delta$,
\begin{equation}\label{eq:940}
\psi(\hat{h},\hat{\beta}) \le C \inf_{\epsilon \ge 0}\lp(\epsilon n + \log \frac{n}{\delta}  + \log \mathcal{N}\lp(\epsilon, \mathcal{H}\rp)\rp).
\end{equation}
%where $c>0$ is a universal constant and $C(M)>0$ depends only on $M$.
\end{lemma}
\begin{proof}
Lemma~\ref{lem:one-direction-final} states that for a single pair $(h,\beta)$, if $\psi(h,\beta)$ is large then with high probability, $\phi(h,\beta)>\phi(h^*,\beta^*)$. We need to argue that this holds for all pairs $(h,\beta)$ such that $\psi(h,\beta)$ is larger than the right hand side of \eqref{eq:940} and this will imply that the MPLE satisfies that $\psi(\hat{h},\hat{\beta})$ is smaller that that quantity, as required. 
For convenience, let $u$ denote the right hand side of \eqref{eq:940}.

%To do this, we need to argue that for all points that are far from the true model, their value of $\phi$ is significantly larger.
A simple way to do this would be to use a union bound over all the points such that $\psi(h,\beta) \ge u$.
Unfortunately, the set of these points is potentially uncountably infinite, hence this approach is infeasible.
A common tool to bypass this obstacle is to define an $\epsilon$-net over this set. If this net has a finite number of points, we can take the union bound over these points. This will mean that all the points in the net have a large value of $\phi$. If the radius $\epsilon$ is chosen sufficiently small and the function $\phi$ is Lipschitz, then this implies that for all the points that are far, their value of $\phi$ is large. 

We now start to execute on that strategy. First of all, we need to know the size of the $\epsilon$-net of the set of points $\mathcal{H}_u := \{(h,\beta) \in \mathcal{H}\times [-M,M] \colon \psi(h,\beta) \ge u^2\}$. We can bound this in terms of the covering numbers for $\mathcal{H}\times [-M,M]$ using the following lemma, easily proven using the definition of covering numbers.
\begin{lemma}
Let $(U,d)$ be a metric space with domain $U$ and metric function $d$ and let $V \subseteq U$. Then, for any $\epsilon>0$
\[
\mathcal{N}(\epsilon, V,d)
\le \mathcal{N}(\epsilon/2,U,d).
\]
\end{lemma}
%This is at most the size of the $\epsilon$-net of the whole parameter space, which is $[-B,B]\times \mathcal{H}$. This is where the metric entropy of the set comes into play. 
For reasons that will become apparent shortly, for $\mathcal{H}$ we will use the distance 
$$
d(h_1, h_2) = \frac{\|h_2 - h_1\|_2}{\sqrt{n}},
$$
which was defined in Section~\ref{sec:proof-general}.
Hence, the metric for the whole space $\mathcal{H}\times [-B,B]$ will be
$$
d_1((h_1,\beta_1),(h_2,\beta_2)) := d(h_1,h_2) + |\beta_2 - \beta_1|
$$
The reason we define this metric is explained by the difference in the Lipschitz constants of $\phi$ with respect to the two parameters. Specifically, we have the following Lemma.

\begin{lemma}
Let $\phi:\mathcal{H}\times [-M,M] \to \R$ be the negative log pseudo-likelihood function. Then, for a fixed $h$, $\phi$ is $2\|A\|n$-Lipschitz with respect to $\beta$ and for a fixed $\beta$, $\phi$ is $2 \sqrt{n}$-Lipschitz with respect to $h$ in $l_2$-norm. 
As a result, $\phi$ is $2\|A\|n$-Lipschitz with respect to both $h,\beta$ and distance $d_1$.
\end{lemma}
\begin{proof}

For a fixed $\beta$, let $h_1,h_2 \in \mathcal{H}$. Define the function $g:[0,1]\to \R$ as
$$
g(t) = \phi(t h_2 + (1-t)h_1,\beta)
$$
and denote $h(t) = t h_2 + (1-t)h_1$. 
we have that
\begin{align*}
|g'(t)| &= \lp|\sum_{i=1}^n ((h_2)_i - (h_1)_i)(\sigma_i - \tanh(\beta A_i \sigma + h(t)_i))\rp|\\
&\leq \|h_2 - h_1\| \sqrt{\sum_{i=1}^n \lp(\sigma_i - \tanh(\beta A_i \sigma + h(t)_i)\rp)^2}\\
&\leq 2\sqrt{n}\|h_2 - h_1\|
\le 2 n \|h_2-h_1\|_\infty,
\end{align*}
where we have used the Cauchy-Schwarz inequality in the above calculations. Notice that this bound holds for all $t \in [0,1]$. Hence, by the mean value Theorem, we have
\begin{align*}
|\phi(h_2,\beta) - \phi(h_1,\beta)| &= |g(1) - g(0)| = |g'(\xi)| \leq  2\sqrt{n}\|h_2 - h_1\|
\end{align*}
for some $\xi \in (0,1)$. This shows that $\phi$ is $2\sqrt{n}$ Lipschitz for a fixed $\beta$. Now, suppose $h$ is fixed. We define $r:[0,1]\to \R$ as 
$$
r(t) = \phi(h, \beta(t))
$$
where $\beta(t) = t\beta_2 + (1-t)\beta_1$. 
Similar to the previous computation, we have
\begin{align*}
r'(t) &= \lp|(\beta_2 - \beta_1) \sum_{i=1}^n A_i\sigma(\sigma_i - \tanh(\beta(t) A_i \sigma + h_i))\rp|\\
&\leq |\beta_2 - \beta_1| \|A\sigma\| 2\sqrt{n}\\
&\leq 2\|A\|n |\beta_2 - \beta_1|
\end{align*}
where we again used Cauchy-Schwarz. Finally, by the mean value Theorem
$$
|\phi(h,\beta_2) - \phi(h,\beta_1)| = |r(1) - r(0)| \leq 2\|A\|n |\beta_2 - \beta_1|
$$
which establishes the Lipschitz constant for $\beta$. To conclude the proof, notice that
\begin{align*}
|f(h_2,\beta_2) - f(h_1,\beta_1)| &\leq
|f(h_2,\beta_2) - f(h_2,\beta_1)| + |f(h_2,\beta_1) - f(h_1,\beta_1)|\\
&\leq 2\|A\|n |\beta_2 - \beta_1| + 2\|A\|\sqrt{n}\|h_2 - h_1\|\\
&= 2\|A\|n d_1((h_2,\beta_2),(h_1,\beta_1))
\end{align*}

\end{proof}
Thus, we need to cover $\mathcal{H}_u\times [-B,B]$ with respect to the distance $d_1$. 
Denote by $N(\epsilon) = \mathcal{N}(\epsilon, \mathcal{H}_u\times [-B,B], d_1)$ the covering number of this set. We can simplify this
expression by noticing that if we choose an $\epsilon/2$ cover of $[-B,B]$ w.r.t. the absolute value distance and an $\epsilon/2$ cover of $\mathcal{H}_u$ w.r.t. $d$, their product gives an $\epsilon$ cover for $[-B,B]\times \mathcal{H}_u$ w.r.t. $d_1$. Also, it is clear that we can choose an $\epsilon/2$ cover of $[-B,B]$ of size $4B/\epsilon$. Thus, we have that
\begin{equation}\label{eq:prodcover}
N(\epsilon)  \leq C\frac{1}{\epsilon} \mathcal{N}\lp(\frac{\epsilon}{2}, \mathcal{H}_u, d\rp)
\le C\frac{1}{\epsilon}\mathcal{N}\lp(\frac{\epsilon}{4}, \mathcal{H},d\rp).
\end{equation}

Taking the union bound over this $\epsilon$-net, we get that with probability at least $1 - C\log n\, N(\epsilon)\, \me^{-u^2}$, for any $h,\beta$ in the net,
$\phi(h,\beta) - \phi(h^*,\beta^*) \ge cu^2$. Now we would like to show that this implies the claim for all $h$ that are far apart, not just the ones in the net. 
Since $\phi$ is $n$-Lipschitz with respect to $d_1$, if we choose $\epsilon = u^2/(2n)$, then this implies that with probability at least $1 - C\log n\, N(\epsilon)\, \me^{-u^2}$, for any $(h,\beta) \in  \mathcal{H}_u\times [-B,B]$,
$$
\phi(h,\beta) > \phi(h^*,\beta^*).
$$
In order for this to hold with probability $1 - \delta$, we need 
\begin{align*}
C N\lp(\frac{u^2}{2n}\rp) \log n \me^{-u^2} < \delta \implies u > C(M) \sqrt{\log \log n  + \log N\lp(\frac{u^2}{2n}\rp) + \log\frac{1}{\delta}}
\end{align*}

Since $u$ will end up being an upper bound for $\psi(\hat h,\hat\beta)$, 
we should select
\begin{align*}
\inf\lp\{u: u > C(M) \sqrt{\log \log n  + \log N\lp(\frac{u^2}{2n}\rp) + \log\frac{1}{\delta}}\rp\} &\leq C
\inf\lp\{\sqrt{n\epsilon}: n\epsilon > \log \log n  + \log N\lp(\epsilon\rp) + \log\frac{1}{\delta}\rp\} 
\end{align*}
By continuity arguments, it is easy to see that the infimum of the latter expression is achieved when
$$
n\epsilon = \log \log n  + \log N\lp(\epsilon\rp) + \log\frac{1}{\delta}
$$
Substituting the bound of $N(\epsilon)$, it is enough to consider
$$
n\epsilon = \log \log n  + \log \frac{1}{\epsilon }+ \log \mathcal{N}\lp(\epsilon,\mathcal{H}\rp) + \log\frac{1}{\delta}
$$
Now notice that the right hand side of this inequality will be $>1$ for $\delta$ smaller than a constant. This means that $\epsilon > 1/n$, which implies that $\log(1/\epsilon) < \log n$. Hence, the critical value of $\epsilon$ can only increase if we solve instead
$$
n\epsilon = \log n + \log \mathcal{N}\lp(\epsilon,\mathcal{H}\rp) + \log\frac{1}{\delta}
$$
This critical value of $\epsilon$ is the same as in the problem
$$
\inf_{\epsilon\geq 0}\lp\{ n\epsilon + \log n  + \log \mathcal{N}\lp(\epsilon,\mathcal{H}\rp) + \log\frac{1}{\delta}\rp\} 
$$
Hence, we conclude that 
$$
\inf\lp\{\sqrt{n\epsilon}: n\epsilon > \log \log n  + \log N\lp(\epsilon\rp) + \log\frac{1}{\delta}\rp\}  \leq \sqrt{\inf_{\epsilon\geq 0}\lp\{ n\epsilon + \log n  + \log \mathcal{N}\lp(\epsilon,\mathcal{H}\rp) + \log\frac{1}{\delta}\rp\} }
$$
Let $\epsilon*$ be the point where the infimum is reached. 
Then, with probability at least $1-\delta$, for all $(h,\beta) \in \mathcal{H}_{u^*}\times [-B,B]$, with $u = \sqrt{n\epsilon^*}$, we have 
$$
\phi(h,\beta) > \phi(h^*,\beta^*)
$$
However, we know by definition of MPLE that $\phi(\hat h,\hat \beta) \leq \phi(h^*,\beta^*)$, which implies that with probability $1 - \delta$,

$$
\psi(\hat{h},\hat{\beta}) \le C(M) (u^*)^2 = C(M) \sqrt{n\epsilon^*}
= C(M) \sqrt{
\inf_{\epsilon\geq 0}\lp\{ n\epsilon  + \log \mathcal{N}\lp(\epsilon,\mathcal{H}\rp) + \log\frac{n}{\delta}\rp\} }
$$
Now we use inequality~\eqref{eq:prodcover} and the proof is complete.

\end{proof}

We are now ready to complete the proof of Theorem~\ref{t:general}. So far, we have proved that for the MPLE estimates we have that $\psi(\hat h,\hat \beta)$ will be small. However, we still need to show that this implies that the estimation of the parameters will be good enough. Thus, we now present the complete proof, which shows exactly the connection of $\psi$ with the estimation error of the parameters.

\begin{proof}[Proof of Theorem~\ref{t:general}]
Assume that the high probability event of Lemma~\ref{lem:upperbd-kappa} holds and let $R$ denote the right hand side of \eqref{eq:940}. Our goal is to prove that $\|\hat h-h^*\|_2^2/n \le R\mathcal{C}_1(\mathcal{H},h^*,\beta^*)$. We can assume that $R\ge 1$, and we divide into cases. 
First, if $\psi(\hat h,\hat\beta) \le 1$, then by definition of $\mathcal{C}_1$ in \eqref{eq:def-C1},
\begin{align*}
\frac{\|\hat h-h^*\|_2^2}{n} &= 
\min\lp(\frac{\|\hat h - h^*\|_2^2/n}{\psi(\hat h, \hat\beta)}, \frac{\|\hat h-h^*\|_2^2}{n}\rp)\\
&\le \mathcal{C}_1(\mathcal{H},h^*,\beta^*)
\le R\mathcal{C}_1(\mathcal{H},h^*,\beta^*).
\end{align*}
Otherwise,
\begin{equation}
\frac{\|\hat h-h^*\|_2^2}{n}
= \psi(\hat h, \hat\beta) \cdot \frac{\|\hat h - h^*\|_2^2/n}{\psi(\hat h, \hat\beta)}
\le R \cdot\mathcal{C}_1(\mathcal{H},h^*,\beta^*).
\label{eq:chain-h-bnd}
\end{equation}
Similarly, we can bound $(\hat{\beta}-\beta^*)^2$ in terms of $\mathcal{C}_2$.

\end{proof}

%Corollary ~\ref{cor:simpler} is of course weaker than the general Theorem, since it has an explicity dependence on the Frobenius norm. It will, however, be useful later in some applications of this Theorem. 

\subsection{Exchangeable pairs}\label{sec:chatterjee}

We use the following theorem, proven in \cite{chatterjee2005concentration}, that guarantees  concentration of a function $f(X)$, of a random variable $X$:
\begin{theorem}\label{thm:exchangeable-pairs}
Let $X$ be a random variable over some domain $\mathcal{X}$ and let $f\colon \mathcal{X}\to \mathbb{R}$ be a measurable map. Let $X'$ be another $\mathcal{X}$-valued random variable, jointly distributed with $X$, such that $(X,X')$ is an exchangeable pair, namely, $(X,X')$ has the same distribution as $(X',X)$. Let $F\colon \mathcal{X}^2 \to \mathbb{R}$ be a measurable map that is antisymmetric, namely,
\[
F(x,x')-F(x',x),
\]
and further
\[
\E[F(X,X')|X] = F(X).
\]
Define for $x \in X$,
\[
v(x)
= \frac{1}{2}\E[|(f(X)-f(X'))F(X,X')|\mid X=x].
\]
Let $M>0$ and assume that $|v(X)|\le M$ almost surely. Then, for all $t>0$,
\[
\Pr[|f(X)-\mathbb{E}[f(X)]|>t]
\le 2\me^{-t^2/(2M)}.
\]
\end{theorem}

We prove Lemma~\ref{lem:conc-from-cha}:
\begin{proof}[Proof of Lemma~\ref{lem:conc-from-cha}]
First, notice that $\E f(\sigma)=0$. Indeed, let $\sigma_{-i}$ denote the vector obtained from $\sigma$ by omitting $\sigma_i$ and notice that $\E[\sigma_i \mid \sigma_{-i}] = \tanh(\beta^* (A \sigma)_i + h^*_i)$. Taking expectation, one obtains that
\[
\E[\sigma_i - \tanh(\beta^* (A \sigma)_i + h^*_i)]
= \E[\E[\sigma_i\mid \sigma_{-i}] - \tanh(\beta^* (A \sigma)_i + h^*_i)]=0,
\]
which implies that $\E f(\sigma)=0$ as required.

Next, we prove concentration around the expectation. Define $\sigma'$ in a joint distribution with $\sigma$ as follows: given $\sigma$, an index $j \in [n]$ is drawn uniformly at random. Then, $\sigma'$ is obtained by $\sigma'_i =\sigma_i$ for $i\ne j$ and $\sigma'_i$ is drawn from the conditional distribution of $\sigma_i$ conditioned on $\sigma_{-i}$. One can indeed verify that $(\sigma,\sigma')$ is an exchangeable pair.

Define the function:
\[
F(\sigma,\sigma') = n\sum_{i=1}^n v_i(\sigma_i - \sigma_i')
\]
and notice that $F$ is antisymmetric. Further notice that
\[
\E[F(\sigma,\sigma')\mid \sigma,j]
= n v_j (\sigma_j - \E[\sigma_j'\mid \sigma,j])
= nv_j (\sigma_j - \E[\sigma_j \mid \sigma_{-j}])
= n v_j (\sigma_j - \tanh(\beta^* (A \sigma)_i + h^*_i)).
\]
Taking expectation over $j$, one derives that
$
\E[F(\sigma,\sigma')\mid \sigma]
= f(\sigma)$ as required.

Lastly, we bound $v(\sigma)= \E[|(f(\sigma)-f(\sigma'))F(\sigma,\sigma')|\mid \sigma]$. One has
\[
v(\sigma)
= \frac{1}{n} \sum_{j=1}^n\E[|(f(\sigma)-f(\sigma'))F(\sigma,\sigma')|\mid \sigma,j].
\]
Condition on $j$. Then,
\[
|F(\sigma,\sigma')|=n|v_j(\sigma_j-\sigma'_j)|\le 2 nv_j.
\]
Further, using the 1-Lipschitzness of $\tanh$,
\begin{align*}
|f(\sigma)-f(\sigma')|
&\le |v_j(\sigma_j - \sigma_j')| +  \sum_{i=1}^n |v_i(\tanh(\beta^* (A \sigma)_i + h^*_i) - \tanh(\beta^* (A \sigma')_i + h^*_i))|\\
&\le 2|v_j| + \sum_{i=1}^n |v_i \beta^* ((A\sigma)_i - (A\sigma')_i)|
= 2|v_j| + \sum_{i=1}^n |v_i \beta^* A_{ij}(\sigma_j-\sigma'_j)|\\
&\le 2 |v_j|+2\beta^*\sum_{i=1}^n |v_i| |A_{ij}|.
\end{align*}
We derive that, conditioned on $j$, one has
\[
|F(\sigma,\sigma')(f(\sigma)-f(\sigma'))|
\le 4 n|v_j| \lp(|v_j| + \beta^* \sum_{i=1}^n |v_i||A_{ij}|\rp).
\]
Hence, for all $\sigma$,
\begin{align*}
v(\sigma) &= \E[|F(\sigma,\sigma')(f(\sigma)-f(\sigma'))|\mid \sigma]
= \frac{1}{n} \sum_{j=1}^n \E[|F(\sigma,\sigma')(f(\sigma)-f(\sigma'))|\mid \sigma,j]\\
&\le \frac{1}{n}\sum_{j=1}^n 4 n|v_j| \lp(|v_j| + |\beta^*| \sum_{i=1}^n |v_i||A_{ij}|\rp).
\end{align*}
Denote $\tilde{v} = (|v_1|,\dots,|v_n|)$ and similarly, $\tilde{A} = (|A_ij|)_{i,j\in [n]}$, and we derive that
\[
v(\sigma) \le 4 \tilde{v}^\top \tilde{v} + 4|\beta^*| \tilde{v}^\top \tilde{A} \tilde{v}
\le 4 \|\tilde{v}\|_2^2 + 4|\beta^*| \|\tilde{v}\|_2 \|\tilde{A}\|_2 \|\tilde{v}\|_2
\le 4 \|\tilde{v}\|_2^2 \lp(1 + |\beta^*| \|\tilde{A}\|_\infty \rp)
= 4 \|v\|_2^2 \lp(1 + |\beta^*| \|A\|_\infty \rp).
\]
Applying Theorem~\ref{thm:exchangeable-pairs} with $M=4 \|v\|_2^2 \lp(1 + |\beta^*| \|A\|_\infty \rp)$, the result follows.
\end{proof}
\subsection{Proofs of Auxiliary lemmas} \label{sec:from-prior}
The proofs below follow from similar arguments as in \cite{dagan2020estimating}. Below, we elaborate on how to modify the arguments in that paper.

\subsubsection{Proof of Lemma~\ref{lem:from-previous}}

% The following lemma was key in \cite{dagan2020estimating}:

% \begin{lemma}[\cite{dagan2020estimating}, Lemma~2] \label{lem:ising-subsample}
% 	%Let $\sigma=(\sigma_1,\dots,\sigma_n)$ be an $(M,\gamma)$-Ising model, and fix $\eta \in (0,M]$. 
% 	Let $\beta^*$
% 	Then, there exist subsets $I_1,\dots,I_\ell \subseteq [n]$ with $\ell \le CM^2 \log n/\eta^2$ such that:
% 	\begin{enumerate}
% 		\item For all $i \in [n]$, 
% 		$$|\{j \in [\ell] \colon i \in I_j\}| =\lp\lceil \frac{\eta\ell}{8M} \rp\rceil
% 		$$.
% 		\item For all $j \in [\ell]$ and any value of $x_{-I_j}$, the conditional distribution of $x_{I_j}$ conditioned on $x_{-I_j}$ is an $(\eta,\gamma)$-Ising model.
% 	\end{enumerate}
% 	Furthermore, for any non-negative vector $\theta \in \mathbb{R}^n$ there exists $j \in [\ell]$ such that 
% 	$$
% 	\sum_{i \in I_j} \theta_i \ge \frac{\eta}{8M} \sum_{i=1}^n \theta_i.
% 	$$
% \end{lemma}

We use the index sets $I_1,\dots,I_\ell \subseteq [n]$ created in \cite{dagan2020estimating}, where $\ell=O(\log n)$.
We start by bounding the first derivative of $\varphi$. The following is a small modification of Lemma~3 in \cite{dagan2020estimating}:
% \begin{lemma} \label{lem:derivative-calc}
% 	Let $x$ be a $(1/2,\gamma)$ Ising model over $\{-1,1\}^m$ with interaction matrix $J$ and external field $h$. Let $A$ be a symmetric real matrix of dimension $m\times m$ with zeros on the diagonal, let $b \in \mathbb{R}^m$ be a vector and let
% 	\[
% 	f(x) = \sum_{i \in [m]}(A_ix + b_i) (x - \tanh(J_ix + h)).
% 	\]
% 	Then, for any $t > 0$,
% 	\[
% 	\Pr[|f(x)| \ge t]
% 	\le \exp\left(-c\min\left(\frac{t^2}{\|\E Ax+b\|_2^2}, \frac{t^2}{\|A\|_F^2}, \frac{t}{\|A\|_2}\right)\right),
% 	\]
% 	where $c>0$ is lower bounded by a constant constant whenever $\|A\|_\infty$, $\|b\|_\infty$, $\|A'\|_\infty$ and $\|b'\|_\infty$ are bounded from above by a constant.
% \end{lemma}

% \begin{restatable}{lemma}{lemmaderivativeconcentration} \label{lem:derivative-concentration}
% 	For any symmetric matrix $A$ with zeros on the diagonal, we have
% 	\[
% 	\Pr\left[|\psi_j(x;A)| \ge t ~\middle|~ x_{-I_j} \right]
% 	\le \exp\left(-c\min\left(\frac{t^2}{\|\E\lp[Ax~\middle|~x_{-I_j}\rp]\|_2^2+\|A\|_F^2}, \frac{t}{\|A\|_2}\right)\right).
% 	\]
% \end{restatable}

\begin{lemma}\label{lem:derivative-one-concentration}
% Let $\sigma$ be drawn from the distribution $\Pr_{J^*,h^*}$ and $A \in \R^{n\times n}$. Assume $\|J^*\|_\infty,\|h^*\|_\infty \leq M$. Let $I_1, \dots, I_\ell$ be the subsets
% 	obtained by Lemma~\ref{lem:ising-subsample} for
% 	$\eta = 1/2$.
For any $h \in \mathbb{R}^n$, $\beta \in \mathbb{R}$ and $t > 0$ we have
	\[
	\left|(h,\beta)^\top\nabla\varphi(\beta^*,h^*)\right|\leq C\max\left(t\lp(\|\beta A\|_F + \max_{j \in [l]} \|\E[\beta A\sigma+h|\sigma_{-I_j}]\|_2\rp),t^2\|\beta A\|_2\right)
	\]
with probability at least
\[
1-C \log n\, \exp\lp(-ct^2\rp).
\]
\end{lemma}
\begin{proof}
This follows from a similar proof as of the proof of Lemma~3 in \cite{dagan2020estimating}. We note the differences. First, we replace $\varphi(J)$ with $\varphi(h,\beta)$. Second, we replace $\psi_j(A;\sigma) = \left|\sum_{i \in I_j} A_i \sigma(\sigma_i - \tanh(J^*_i \sigma))\right|$ with $\psi_i(h,\beta;\sigma) = \left|\sum_{i \in I_j} (\beta A_i \sigma+h_i)(\sigma_i - \tanh(\beta^* A_i x))\right|$. 
Thirdly, in the proof of Lemma~3 in \cite{dagan2020estimating}, we replace the inequality
	\[
	\Pr\left[|\psi_j(\sigma;A)| \ge t \lp(\|A\|_F + \lp\|\E[A\sigma|\sigma_{-I_j}]\rp\|_2\rp)\right]
	\le \exp\lp(-c \min\lp( t^2, \frac{t \|A\|_F}{\|A\|_2}\rp)\rp).
	\]
with 
\begin{align*}
\Pr\left[|\psi_j(h,\beta;\sigma)| \ge \max\left(t\lp(\|\beta A\|_F + \max_{j \in [\ell]} \|\E[\beta A\sigma+h|\sigma_{-I_j}]\|_2\rp),t^2\|\beta A\|_2\right)\right]\le \exp\lp(-c t^2\rp),
\end{align*}
which follows from Lemma~5 in \cite{dagan2020estimating}.
\end{proof}

Next, we bound the second derivative of $\varphi$. We use Lemma~\ref{lem:trivial-hessian} which claims that for all $h' \in [-M,M]^n$ and $\beta'\in [-B,B]$,
\[
(h,\beta)^\top \nabla^2 \varphi(h',\beta') (h,\beta)
\ge c \|\beta A\sigma+h\|_2
\]
and bound $\|\beta A\sigma+h\|_2$.
We state a modified variant of Lemma~4 in \cite{dagan2020estimating}:

% \begin{lemma}\label{lem:conc-mainpart}
% 	Fix symmetric matrix $A$ with zeros on the diagonal and $j \in [\ell]$. Then, for any fixed value of $x_{-I_j}$ and
% 	any $t>0$
% 	\[
% 	\Pr\left[\lp\|Ax\rp\|_2^2 < \E\left[\|Ax\|_2^2~\middle|~x_{-I_j}\right] - t ~\middle|~ x_{-I_j}\right]
% 	\le \exp\left(-\frac{c}{\|A\|_2^{2}}\min\left(\frac{t^2}{\|A\|_F^2+\|\E\lp[Ax~\middle|~x_{-I_j}\rp]\|_2^2},t\right)\right).
% 	\]
% \end{lemma}

\begin{restatable}{lemma}{secondder}\label{lem:hessian-oneparam}
%Let $\sigma$ be drawn from the distribution parameterized by $J^*$ and $A \in \R^{n\times n}$ be a symmetric matrix with zeros on the diagonal. Assume $\|J^*\|_\infty \leq M$. 
%Let $A,h$ such that $\|A\|_\infty,\|h\|_\infty \le M$.
%Let $I_1, \dots, I_\ell$ be the subsets obtained by Lemma~\ref{lem:ising-subsample} for $\eta = 1/2$. Then, there exist constants $C,c,c'$ depending only on $M$ such that 
For any $\beta \in \mathbb{R}$, $h \in \mathbb{R}^n$ and $t > 0$, with probability at least $1-\log n \me^{-ct^2}$ we have that either
	\begin{align}\label{eq:932}
	%\max_{J,h'\colon \|J\|_\infty\le M,\|h'\|_\infty \le M}\frac{\partial^2 \phi(J,h')}{\partial^2 (A,h)}
	\|\beta A\sigma+h\|_2^2
	\ge c'\|\beta A\|_F^2 + c'\max_{j \in [\ell]} \|\E[\beta A\sigma+h\mid \sigma_{-I_j}]\|_2^2,
	%\frac{\|A\|_F^2+\max_{j \in [l]} \|\E[Ax+h\mid x_{-I_j}]\|_2^2}{\|A\|_2^2}\right).
	\end{align}
	or, that 
	\begin{equation}\label{eq:33}
	\|\beta A\|_F + \max_{j \in [\ell]} \|\E[\beta A\sigma+h\mid \sigma_{-I_j}]\|_2
	\le C t \|\beta A\|_2.
	\end{equation}
\end{restatable}
\begin{proof}
We explain the changes with respect to the proof of Lemma~4 in \cite{dagan2020estimating}. 
\begin{itemize}
\item 
We start by replacing Lemma~7 in \cite{dagan2020estimating} with 
\begin{equation}\label{eq:modified-lem-7}
	\E\lp[\|\beta A\sigma+h\|_2^2 ~\middle|~ \sigma_{-I_j}\rp]
	\ge c_e \|\beta A_{\cdot I_j}\|_F^2  + \lp\|\E\lp[\beta A\sigma+h~\middle|~\sigma_{-I_j}\rp]\rp\|_2^2,
\end{equation}
where $c_e>0$ is a constant.

\item
Next, we replace Lemma~8 in \cite{dagan2020estimating} with the claim that for any $t>0$,
\begin{multline}\label{eq:modified-lem8}
	\Pr\left[\lp\|\beta A\sigma+h\rp\|_2^2 < \E\left[\|\beta A\sigma+h\|_2^2~\middle|~\sigma_{-I_j}\right] - t ~\middle|~ \sigma_{-I_j}\right]\\
	\le \exp\left(-\frac{c}{\|\beta A\|_2^{2}}\min\left(\frac{t^2}{\|\beta A\|_F^2+\|\E\lp[\beta A\sigma+h~\middle|~\sigma_{-I_j}\rp]\|_2^2},t\right)\right).
\end{multline}
This follows from the same proof as that of Lemma~8 in \cite{dagan2020estimating}.

\item
Finally, we explain how to derive  Lemma~\ref{lem:hessian-oneparam}. For any $j\in [\ell]$, let $E_j$ denote the event that
\[
\|\beta A\sigma+h\|_2^2 > c_e \|\beta A_{\cdot I_j}\|_F^2  + \lp\|\E\lp[\beta A\sigma+h~\middle|~\sigma_{-I_j}\rp]\rp\|_2^2 - t \|\beta A\|_2\left(\|\beta A\|_F + \|\E[\beta A\sigma+h\mid \sigma_{-I_j}]\|_2\right),
\]
where $c_e$ is the constant from \eqref{eq:modified-lem-7}.
By \eqref{eq:modified-lem8},
\[
	\Pr[E_j \mid \sigma_{-I_j}]
	\ge 1- \exp\left(-c \min\lp(t^2,t(\|\beta A\|_F+\|\E[\beta A\sigma+h\mid \sigma_{-I_j}]\|_2)/\|\beta A\|_2\rp)\right)
	\ge 1- \exp\left(-c t^2\right),
\]
using the fact that $E_j$ holds trivially for $t \ge \Omega((\|\beta A\|_F+\|\E[\beta A\sigma+h\mid \sigma_{-I_j}]\|_2)/\|\beta A\|_2)$. We conclude that
\[
\Pr[E_j] \ge 1-\exp(-c t^2).
\]
Taking a union bound over $j\in [\ell]$ and that $\ell = O(\log n)$, one has
\[
\Pr[\cap_j E_j] \ge 1 - C\log(n) \exp(-c t^2).
\]
Assume that $\cap_j E_j$ holds. Then, one can easily verify that if \eqref{eq:33} does not hold then \eqref{eq:932} holds, if the constants in \eqref{eq:932} and \eqref{eq:33} are chosen appropriately.
\end{itemize}
\end{proof}

\begin{proof}[Proof of Lemma~\ref{lem:from-previous}]
Let $t\ge 0$ and assume that the high probability events of both Lemma~\ref{lem:derivative-one-concentration} and Lemma~\ref{lem:hessian-oneparam} hold. As discussed above, due to Lemma~\ref{lem:trivial-hessian} it suffices to show that under these high probability events, either
\begin{equation}\label{eq:40}
\left|(h,\beta)^\top\varphi(h^*,\beta^*)\right| / \|\beta A\sigma+h\|_2^2 \le \frac{C't}{\|\beta A\sigma+h\|_2}
\end{equation}
holds or
\[
\|\beta A\sigma+h\|_2 \le C t \|\beta A\|_2.
\]
% This follows since 
% \[
% \max_{J,h'\colon \|J\|_\infty\le M,\|h'\|_\infty \le M}\frac{\partial^2 \phi(J,h')}{\partial^2 (\beta A,h)}
% \ge c \|\beta A\sigma+h\|_2^2,
% \]
% which holds from a similar calculation as Equation~(15) in \cite{dagan2020estimating}.
We divide into cases. 
\begin{itemize}
    \item 
First, assume that
\[
\|\beta A\|_F + \max_{j \in [l]} \|\E[\beta A\sigma+h\mid \sigma_{-I_j}]\|_2
	\le C t \|\beta A\|_2
\]
for a sufficiently large constant $C>0$. Then, by Lemma~\ref{lem:derivative-one-concentration},
\[
\left|(h,\beta)^\top\varphi(h^*,\beta^*)\right| \le t^2 \|\beta A\|_2.
\]
This implies that if $\|\beta A\sigma+h\|_2 \ge t\|\beta A\|_2$ then \eqref{eq:40} holds, which is what we wanted to prove.
\item Next, assume that 
\[
\|\beta A\|_F + \max_{j \in [l]} \|\E[\beta A\sigma+h\mid \sigma_{-I_j}]\|_2
	\ge C t \|\beta A\|_2.
\]
Since we assumed that the high probability events of Lemma~\ref{lem:derivative-one-concentration} and Lemma~\ref{lem:hessian-oneparam} hold, one has that
\[
	\left|(h,\beta)^\top\varphi(h^*,\beta^*)\right|\leq Ct\lp(\|\beta A\|_F + \max_{j \in [l]} \|\E[\beta A\sigma+h|\sigma_{-I_j}]\|_2\rp),
\]
whereas
\[
\|\beta A\sigma+h\|_2^2
	\ge c'\|\beta A\|_F^2 + c'\max_{j \in [l]} \|\E[\beta A\sigma+h\mid \sigma_{-I_j}]\|_2^2,
\]
which implies that 
\[
	\left|(h,\beta)^\top\varphi(h^*,\beta^*)\right|
	\le C't \|\beta A\sigma+h\|,
\]
which derives \eqref{eq:40} as required.
\end{itemize}
\end{proof}

\subsubsection{Proof of Lemma~\ref{lem:trivial-hessian}}

This holds from a simple calculation as Equation~(15) in \cite{dagan2020estimating}.

\subsubsection{Proof of Lemma~\ref{lem:bnd-by-Forb}}
This follows directly from the following lemma:
\begin{lemma}\label{lem:bnd-by-Forb-general}
    For any real matrix $A$ and any $h \in \mathbb{R}^n$,
    \[
    \Pr[\|A\sigma+h\|_2 \ge c \|A\|_F]
    \ge 1-\exp(-c \|A\|_F^2/\|A\|_2^2).
    \]
\end{lemma}
\begin{proof}[Sketch.]
This is analogous to the second part of Lemma~4 in \cite{dagan2020estimating}.
% Lemma~4 in \cite{dagan2020estimating} proves that
% \[
% 	\le C \log n \exp\left(-c\frac{\|A\|_F^2}{\|A\|_2^2}\right).
% \]
% In order to prove this, one first proves that
% \[
% 	\Pr\left[\|A\sigma\| < c'\|A\|_F^2 + c'\max_{j \in [l]} \|\E[A\sigma\mid \sigma_{-I_j}]\|_2^2\right]
% 	\le C \log n \exp\left(-c\frac{\|A\|_F^2}{\|A\|_2^2}\right).
% \]
% Using essentially the same proof, one can prove that
% \[
% 	\Pr\left[\|A\sigma+h\| < c'\|A\|_F^2 + c'\max_{j \in [l]} \|\E[A\sigma+h\mid \sigma_{-I_j}]\|_2^2\right]
% 	\le C \log n \exp\left(-c\frac{\|A\|_F^2}{\|A\|_2^2}\right),
% \]
% which concludes our proof.
\end{proof}

\subsection{Proof of Lemma~\ref{lem:C-to-Forb}} \label{sec:ub-simpler}

Notice the following:
first,
\[
\sqrt{\psi(h,\beta;h^*,\beta^*)}
\ge |\beta-\beta'|\|A\|_F
\]
and secondly,
\begin{multline*}
\sqrt{\psi(h,\beta;h^*,\beta^*)}
\ge \lp\|h - h^* + (\beta-\beta^*) A \vtanh\lp(\frac{\beta^*}{\beta-\beta^*}(h^* - h) + h^*\rp) \rp\|\\
\ge \|h-h^*\| - |\beta-\beta^*| \|A\|_2 \lp\| \vtanh\lp(\frac{\beta'}{\beta-\beta'}(h^* - h) + h^*\rp)\rp\|_2\\
\ge \|h-h^*\| - |\beta-\beta^*| \|A\|_2 \sqrt{n}
\ge \|h-h^*\| - |\beta-\beta^*|\sqrt{n}\enspace,
\end{multline*}
using the fact that $|\tanh(a)| \le 1$ for all $a \in \mathbb{R}$ and that $\|A\|_2 \le \|A\|_\infty \le 1$, where the first inequality is due to the symmetricity of $A$ and the second is due to the assumption in this theorem that $\|A\|_\infty \le 1$.

Divide into cases: if
\[
\sqrt{n} |\beta-\beta^*| \le \|h-h^*\|_2,
\]
then
\[
\sqrt{\psi(h,\beta;h^*,\beta^*)} \ge \|h-h^*\|/2,
\]
which implies that
\[
\frac{\|h-h^*\|^2/n}{\psi(h,\beta;h^*,\beta^*)}
\ge \frac{1}{4n} \ge \frac{1}{4\|A\|_F^2},
\]
using the fact that $\|A\|_F \le \sqrt{n} \|A\|_2 \le \sqrt{n}$.
For the second case, if 
\[
\sqrt{n} |\beta-\beta^*| \ge \|h-h^*\|/2,
\]
then
\[
\sqrt{\psi(h,\beta;h^*,\beta^*)} \ge |\beta-\beta^*| \|A\|_F
\ge \|h-h^*\|_2\|A\|_F /2\sqrt{n}.
\]
Hence,
\[
\frac{\|h-h^*\|^2/n}{\psi(h,\beta;h^*,\beta^*)} \ge \frac{1}{4\|A\|_F^2}.
\]
This holds for any $h,\beta$, hence
\[
\mathcal{C}_1(\mathcal{H},h^*,\beta^*)
\le \mathcal{C}_1'(\mathcal{H},h^*,\beta^*)
\le \frac{1}{4\|A\|_F^2},
\]
as required.
\section{Lower bound}\label{sec:lowerbound}
We start with proving the lower bound in Theorem~\ref{thm:costis:general upper} and then prove Theorem~\ref{thm:costis:lowerbound}.

\subsection{Proof of the lower bound in Theorem~\ref{thm:costis:general upper}}

As stated in Section~\ref{sec:ising-formulation}, it suffices to prove Theorem~\ref{thm:general-lower}. 
The first step is to find a suitable upper bound for the KL-divergence between two Ising models in high temperature. Ideally,
we would like this bound to depend on the
difference in $\beta$ and the difference in $h$. 
This is achieved with the following Lemma.

\begin{lemma}\label{l:klbound}
Suppose $h_0,h_1,\beta_0,\beta_1$ are such that $\|h_0\|_\infty,\|h_1\|_\infty \le M$ and $|\beta_0|\|A\|_\infty, |\beta_1|\|A\|_\infty \le 1-\alpha$ for some $\alpha \in (0,1)$.  Then, there is a constant $C = C(\alpha,M)$ such that
$$
D_{KL}(P_{h_0,\beta_0}\|P_{h_1,\beta_1}) \leq C (\beta_1-\beta_0)^2\|A\|_F^2 + \|(\beta_1-\beta_0)\E_{h_0,\beta_0}[A\sigma] + h_1 - h_0\|_2^2.
$$
\end{lemma}
\begin{proof}
In the following computations, the concept of the log partition function will be useful. We thus define 
$$
F(h,\beta) = \ln \lp(\frac{1}{2^n}\sum_{y \in \{-1,+1\}^n}\exp\lp(\frac{\beta}{2} y^\top Ay\rp)\rp)
$$
which is the log partition function of a model with interaction matrix $\beta A$ and external field $h$. Also, 
we also define for $t \in [0,1]$,
\[
h_t = (1-t)h_0 + t h_1; \quad
\beta_t = (1-t) \beta_0 + t \beta_1
\]
and notice that for $t \in \{0,1\}$, this definition coincides with $h_0,h_1,\beta_0,\beta_1$.
Some simple calculations show that
\begin{align*}
&\frac{d F(h_t,\beta_t)}{dt} = \E_{h_t,\beta_t}\lp[\frac{1}{2} (\beta_1-\beta_0) \sigma^\top A \sigma + \sigma^\top (h_1-h_0) \rp]\\
&\frac{d^2 F(h_t,\beta_t)}{dt^2} = \Var_{h_t,\beta_t}\lp[\frac{1}{2} (\beta_1-\beta_0) \sigma^\top A \sigma + \sigma^\top (h_1-h_0)\rp]
\end{align*}
where the expectation and variance are over $\sigma\sim \Pr_{h_t,\beta_t}$
Now, we do some simple calculations with the KL-divergence
\begin{equation}\label{eq:32}
\begin{split}
D_{KL}(P_{h_0,\beta_0}\|P_{h_1,\beta_1}) &= 
\E_{h_0,\beta_0} 
\ln \frac{P_{h_0,\beta_0}(\sigma)}{P_{h_1,\beta_1}(\sigma)}\\
&= \E_{h_0,\beta_0}
\lp[\frac{1}{2}(\beta_0-\beta_1)\sigma^\top A y + \sigma^\top (h_0-h_1) - F(h_0,\beta_0) + F(h_1,\beta_1)\rp]\\
&=F(h_1,\beta_1) - F(h_0,\beta_0) -  \frac{d F(h_t,\beta_t)}{d_t}\bigg|_{t=0} \\
&= \frac{1}{2} \frac{d^2 F(h_t,\beta_t)}{d_t^2}\bigg|_{t=\xi}\\
&= \Var_{h_\xi,\beta_\xi}\lp[\frac{1}{2} (\beta_1-\beta_0) \sigma^\top A \sigma + \sigma^\top (h_1-h_0)\rp]
\end{split}
\end{equation}
for some $\xi \in [0, 1]$, using Taylor's approximation on $F(h_t,\beta_t)$.
Hence, it suffice to bound this variance in order to obtain a bound on the KL-divergence. This is a variance of a quadratic polynomial of the Ising model, which can be bounded by the following theorem:
\begin{theorem}[\cite{adamczak2019note}, Theorem~2.2]\label{thm:var-bound}
Let $A$ be an interaction matrix of dimension $n\times n$, let $\beta\in \mathbb{R}$ such that $\|\beta A\|_\infty \le 1-\alpha$ for some $\alpha \in (0,1)$ and let $h \in [-M,M]^n$ for some $M>0$. Let $A'$ be an $n\times n$ real matrix and let $h' \in \mathbb{R}^n$. Then, there exists $C=C(\alpha,M)>0$ such that
\[
\Var_{h,\beta}\lp[\frac{1}{2}\sigma^\top A'\sigma + \sigma^\top h\rp]
\le C \|A'\|_F^2 + C \|\E_{h,\beta}[A\sigma + h]\|_2^2.
\]
\end{theorem}

We can apply this theorem to bound the right hand side of \eqref{eq:32}, since the assumptions of this lemma guarantee that $\|\beta_t A\|_\infty \le \max(\|\beta_0 A\|_\infty, \|\beta_1 A\|_\infty) \le 1-\alpha$, and the result follows.
\end{proof}
Now we will try to connect this bound on the KL-divergence with the rate of the upper bound. In the proof of Theorem~\ref{t:general}, in order to obtain a good upper bound on the rate, we lower bounded $\|h^* - h - (\beta- \beta^*)A\sigma\|_2^2$ by a nonrandom quantity(Lemma~\ref{l:substitute}. 
In the following Lemma, we essentially show that this substitution is tight in the high temperature regime, which means there is a corresponding upper bound of the
same order. The proof involves a similar trick as in Lemma~\ref{l:substitute}, by bounding each term of the triangle inequality separately.

\begin{lemma}
Suppose we have an Ising model with $\beta^*,h^*$ satisfying $\beta^*\|A\|_\infty \le 1-\alpha$ for some $\alpha \ge 0$ and let $h,\beta$ such that $|\beta|,\|h\|_\infty \le M$. Then, there exists some $C=C(\alpha,M)>0$, such that
$$
\|h^* - h - (\beta-\beta^*) \Exp_{h^*,\beta^*} A\sigma\| \le C \lp\|h^* - h - (\beta-\beta^*) A \tanh\lp(\frac{\beta^*}{\beta-\beta^*}(h^*-h) + h^*\rp)\rp\|
+ C\|(\beta-\beta^*) A\|_F.
$$
\end{lemma}

\begin{proof}

Denote $\tilde{\beta} = \beta-\beta^*$. We have
\begin{multline}\label{eq:21}
\|h^* - h - (\beta-\beta^*) \Exp A\sigma\| 
 \le
\\
\lp\|h^* - h - \tilde{\beta} A \tanh\lp(\frac{\beta^*}{\tilde{\beta}}(h^* - h) + h^*\rp)\rp\|
+ \lp\| \tilde{\beta}A\tanh(\beta^*\Exp[A\sigma] + h^*)
- \tilde{\beta}A\tanh\lp(\frac{\beta^*}{\tilde{\beta}}(h^* - h) + h^*\rp)\rp\|
\\
+ \lp\|\tilde{\beta} \Exp\lp[A\tanh(\beta^*A\sigma + h^*)\rp]
- \tilde{\beta}A\tanh(\beta^*\Exp[A\sigma] + h^*)\rp\|
\end{multline}
We bound the right hand side of \eqref{eq:21}. The second term is bounded as follows, using the fact that $\tanh$ is a 1-Lipschitz function:
\begin{multline*}
\lp\| \tilde{\beta}A\tanh(\beta^*\Exp[A\sigma] + h^*)
- \tilde{\beta}A\tanh\lp(\frac{\beta^*}{\tilde{\beta}}(h^* - h) + h^*\rp)\rp\|_2\\
\le |\tilde{\beta}|\|A\|_2 \lp\| \tanh(\beta^*\Exp[A\sigma] + h^*)
- \tanh\lp(\frac{\beta^*}{\tilde{\beta}}(h^* - h) + h^*\rp)\rp\|_2\\
\le |\tilde{\beta}|\|A\|_2 \lp\| \beta^*\Exp[A\sigma] + h^*
- \lp(\frac{\beta^*}{\tilde{\beta}}(h^* - h) + h^*\rp)\rp\|_2\\
= |\beta^*|\|A\|_2 \lp\| \tilde{\beta}\Exp[A\sigma] -(h^* -h) \rp\|_2
= |\beta^*|\|A\|_2 \lp\| h^*-h- \tilde{\beta}\Exp[A\sigma] \rp\|_2.
\end{multline*}
Further, the last term of \eqref{eq:21} can be bounded as follows, using Jenssen's inequality and the Lipschitzness of $\tanh$:
\begin{multline}\label{eq:38}
    \lp\|\tilde{\beta} \Exp\lp[A\tanh(\beta^*A\sigma + h^*)\rp]
- \tilde{\beta}A\tanh(\beta^*\Exp[A\sigma] + h^*)\rp\|\\
\le |\tilde{\beta}|\|A\|_2
\lp\|\Exp\lp[\tanh(\beta^*A\sigma + h^*)\rp]
- \tanh(\beta^*\Exp[A\sigma] + h^*)\rp\|_2\\
=|\tilde{\beta}|\|A\|_2 \sqrt{\lp\|\Exp\lp[\tanh(\beta^*A\sigma + h^*)
- \tanh(\beta^*\Exp[A\sigma] + h^*)\rp]\rp\|_2^2}\\
\le |\tilde{\beta}|\|A\|_2\sqrt{\Exp\lp[\lp\|\tanh(\beta^*A\sigma + h^*)
- \tanh(\beta^*\Exp[A\sigma] + h^*)\rp\|_2^2\rp]}\\
\le|\tilde{\beta}|\|A\|_2\sqrt{\Exp\lp[\lp\|\beta^*A\sigma + h^*
- (\beta^*\Exp[A\sigma] + h^*)\rp\|_2^2\rp]}\\
= |\tilde{\beta}|\|A\|_2|\beta^*|\sqrt{\Exp\lp[\lp\|A\sigma -\Exp[A\sigma]\rp\|_2^2\rp]}
= |\tilde{\beta}|\|A\|_2|\beta^*|\sqrt{\sum_{i=1}^n \mathrm{Var}(A_i\sigma)},
\end{multline}
where $A_i$ is row $i$ of $A$. Using Theorem~\ref{thm:var-bound}, each variance is bounded by $C\|A_i\|_2^2$, and the right hand side of \eqref{eq:38} is bounded by 
\[
C|\tilde{\beta}|\|A\|_2|\beta^*|\sqrt{\sum_i \|A_i\|_2^2} \le C|\tilde{\beta}|\|A\|_2|\beta^*| \|A\|_F \le C|\tilde{\beta}|\|A\|_F.
\]
since $\|A_2\||\beta^*| \le 1$. Combining the above calculations, we derive that
\[
\|h^* - h - \tilde{\beta} \Exp A\sigma\| 
 \le
\lp\|h^* - h - \tilde{\beta} A \tanh\lp(\frac{\beta^*}{\tilde{\beta}}(h^* - h) + h^*\rp)\rp\|
+ |\beta^*|\|A\|_2 \|h^* - h - \tilde{\beta} \Exp A\sigma\| + C|\tilde{\beta}|\|A\|_F.
\]
Since $|\beta^*|\|A\|_2 < 1$, we derive that
\[
\|h^* - h - \tilde{\beta} \Exp A\sigma\| \le \frac{C}{1-|\beta^*|\|A\|_2} \lp(\lp\|h^* - h - \tilde{\beta} A \tanh\lp(\frac{\beta^*}{\tilde{\beta}}(h^* - h) + h^*\rp)\rp\| + |\tilde{\beta}|\|A\|_F \rp)
\]
\end{proof}

As a corollary of the two lemmas above, we derive a bound for the KL-divergence between two Ising models in terms of a familiar quantity that was used in the upper bound of the rate.
\begin{lemma}\label{lem:kl-lb}
Suppose we have an Ising model with $\beta_0,h_0$ satisfying $\beta_0\|A\|_\infty \le 1-\alpha$ for some $\alpha \ge 0$ and let $h_1,\beta_1$ such that $|\beta_1|,\|h_1\|_\infty \le M$. Then, there exists some $C=C(\alpha,M)>0$ such that
\begin{align*}
D_{KL}(P_{h_0,\beta_0}\|P_{h_1,\beta_1})
&\le C \lp\|h_0 - h_1 - (\beta_1-\beta_0) A \tanh\lp(\frac{\beta_0}{\beta_1-\beta_0}(h_0-h_1) + h_0\rp)\rp\|_2^2
+ C\|(\beta_1-\beta_0) A\|_F^2\\
&=C \psi(h_1,\beta_1;h_0,\beta_0).
\end{align*}
\end{lemma}

Next, we add another auxiliary lemma:
\begin{lemma}\label{lem:scale-psi}
Let $\beta_0,\beta_1 \in \mathbb{R}$ and $h_0,h_1 \in \mathbb{R}^n$. Define for any $t \in \mathbb{R}$,
\[
(h_t,\beta_t) = (h_0,\beta_0) + t (h_1-h_0,\beta_1-\beta_0).
\]
Then, for any $t \in \mathbb{R}$,
\[
\psi(h_t,\beta_t;h_0,\beta_0)
= t^2 \psi(h_1,\beta_1;h_0,\beta_0).
\]
Consequently, for any $t \ne 0$,
\[
\frac{\|h_t-h_0\|^2}{\psi(h_t,\beta_t;h_0,\beta_0)}
= \frac{\|h_1-h_0\|^2}{\psi(h_1,\beta_1;h_0,\beta_0)}.
\]
\end{lemma}
\begin{proof}
This follows from the fact that
\[
h_t-h_0
= t (h_1-h_0); \quad
\beta_t-\beta_0 = t(\beta_1-\beta_0).
\]
\end{proof}
To prove the lower bound, we use the following version of Le-Cam's method for binary hypothesis testing:
\begin{lemma}[Le-Cam]\label{lem:la-cam}
Let $(\Theta,d)$ be a metric space, let $\{ P_\theta \colon \theta \in \Theta\}$ be a family of distributions over the domain $\mathcal{X}$ let $\theta_0,\theta_1 \in \Theta$ and let $\hat{\theta} \colon \mathcal{X} \to \Theta$ be some estimator. Then, there exists $i \in \{0,1\}$ such that
\[
\Pr_{x\sim P_{\theta_i}}[d(\hat{\theta}(x), \theta_i) \ge d(\theta_0,\theta_1)/2]
\ge \frac{1- D_{TV}(P_{\theta_0}, P_{\theta_1})}{2}.
\]
\end{lemma}

We are now ready to present the proof of the Lower bound in Theorem~\ref{thm:costis:general upper}, which is also the proof of Theorem~\ref{thm:general-lower}. The trick is to use the bound of Lemma~\ref{lem:kl-lb} to select to Ising models that are as far as possible, while having $KL$-divergence bounded by a constant. Using Pinsker's inequality, this translates into a constant bound for the TV-distance between them, which implies that we cannot distinguish between the two models.
\begin{proof}
Let $(h_0,\beta_0)$ be the maximizer of $\mathcal{C}_1(\mathcal{H},h,\beta)$ over $(h,\beta) \in \mathcal{H} \times [-1/2,1/2]$. Let $(h_1,\beta_1)$ be
the maximizer in the definition of $\mathcal{C}_1(\mathcal{H},h_0,\beta_0)$.
Let $c_0\in (0,1)$ be a constant to be defined later. Define $(h_t,\beta_t)$ as in Lemma~\ref{lem:scale-psi}. And let $\xi$ be the maximal number, $t \in [0,1]$ such that the following holds: both
\begin{equation}\label{cond1}
t^2 \psi(h_1,\beta_1;h_0,\beta_0)
= \psi(h_t,\beta_t;h_0,\beta_0)
\le c_0,
\end{equation}
and
\begin{equation}\label{cond2}
|\beta_t - \beta_0| \le 1/2.
\end{equation}
Notice that there exists $\zeta \in \{-\xi,\xi\}$ such that $|\beta_{\zeta}| \le 1/2$: indeed, $|\beta_0| \le 1/2$, and $-1/2 \le \beta_0-\beta_\xi = -(\beta_0 - \beta_{-\xi}) \le 1/2$. Fix such $\zeta$ and notice that due to the assumption that $\mathcal{F}$ is convex and closed under complements, we have that $h_\zeta \in \mathcal{F}$.
Furthermore, by Lemma~\ref{lem:kl-lb},
\[
D_{KL}(\Pr_{h_\zeta,\beta_\zeta}~\|~\Pr_{h_0,\beta_0})
\le C \psi(h_\zeta,\beta_\zeta; h_0,\beta_0)
\le C c_0
\]
which implies by Pinsker's inequality, that
\[
D_{TV}(\Pr_{h_\zeta,\beta_\zeta},~\Pr_{h_0,\beta_0}) \le \sqrt{D_{KL}(\Pr_{h_\zeta,\beta_\zeta}~\|~\Pr_{h_0,\beta_0})/2}
\le \sqrt{Cc_0/2}.
\]
By Le-Cam's method, if $c_0$ is sufficiently small, then either
\[
\Pr_{h_\zeta,\beta_\zeta}[\|\hat h - h_\zeta\| \ge \|h_\zeta-h_0\|/2] \ge 0.49
\]
or
\[
\Pr_{h_0,\beta_0}[\|\hat h - h_0\| \ge \|h_\zeta-h_0\|/2] \ge 0.49.
\]
It remains to prove that
\[
\|h_\zeta-h_0\|^2/n \ge c\mathcal{C}_1(\mathcal{H},h_0,\beta_0) \ge c\mathcal{C}_1(\mathcal{H},h_\zeta,\beta_\zeta)
\]
for some constant $c>0$.
Notice that the second inequality follows from the fact that $h_0,\beta_0$ was chosen to maximize $\mathcal{C}_1$, hence it remains to prove the first inequality. By Lemma~\ref{lem:scale-psi}, one has
\[
\frac{\|h_\zeta-h_0\|^2}{\psi(h_\zeta,\beta_\zeta;h_0,\beta_0)}
= \frac{\|h_1-h_0\|^2}{\psi(h_1,\beta_1;h_0,\beta_0)}.
\]
If $\psi(h_1,\beta_1;h_0,\beta_0) \le 1$, then $\mathcal{C}_1(\mathcal{H},h_0,\beta_0) = \|h_0-h_1\|_2^2/n$, by definition of $\mathcal{C}_1$ and by the fact $(h_1,\beta_1)$ was chosen as the maximizer in its definition. Since $\psi(h_1,\beta_1;h_0,\beta_0) \le 1$, we can choose $\zeta$ to be at least a constant $c$, without violating the conditions\eqref{cond1}, \eqref{cond2}.  Hence
\[
\|h_\zeta-h_0\|^2/n
\ge c^2 \|h_1-h_0\|^2/n
= c^2 \mathcal{C}_1(\mathcal{H},h_0,\beta_0)
\]
as required.
Next, assume that $\psi(h_1,\beta_1;h_0,\beta_0) > 1$, which implies that 
\[
\mathcal{C}_1(\mathcal{H},h_0,\beta_0) = \frac{\|h_0-h_1\|_2^2/n}{\psi(h_1,\beta_1;h_0,\beta_0)}
= \frac{\|h_0-h_\zeta\|_2^2/n}{\psi(h_\zeta,\beta_\zeta;h_0,\beta_0)}.
\]
By definition of $\zeta$, we derive that $\psi(h_\zeta,\beta_\zeta;h_0,\beta_0) \ge c$ for some constant $c \ge 0$. This happens because we can take $\zeta$ at least a constant without violating \eqref{cond2}, which will result in a greater than constant value for $\psi(h_\zeta,\beta_\zeta;h_0,\beta_0) $. In particular, $\|h_0-h_\zeta\|_2^2/n \ge c \mathcal{C}_1(\mathcal{H},h_0,\beta_0)$ as required. This concludes the proof.
\end{proof}

\subsection{Proof of Theorem~\ref{thm:costis:lowerbound}} \label{sec:pr:lower-Forb}
It suffices to prove the statement assuming that $n$ and $r$ are integer powers of $2$, and we make this assumption here. Further, we can replace the assumption that $|x_i|\le M$ and $|\theta| \le 1$ with $|x_i|\le 1$ and $|\theta|\le M$. Here, we will prove the result with $M=2$.

Define $a=0.8952...$ as the non-negative solution to $\tanh(1+a/2) = a$. Consider the logistic regression setting where $x = \mathbf{1} = (1,\dots,1)$. Let $\theta_0 = 1$, $\beta_0 = 1/2$, $\theta_1 = 1+a$ and $\beta_1 = -1/2$. Let $A$ be the following matrix: partition the indices to $r$ sets of size $n/r$, and set $A_{ij} = r/n$ if $i$ and $j$ are in the same partition and $A_{ij} = 0$ otherwise. Notice that the sum of entries in each row of $A$ is $1$, and further, that $\|A\|_F^2 = r$. Indeed, for any partition, $\sum_{ij} A_{ij}^2= 1$, where the sum is over $i,j$ in this partition. Summing over all partitions, we obtain that $\|A\|_F^2 = r$.

In this case, $f_\theta(\mathbf{1}) = (\theta,\dots,\theta)$, hence we have
\begin{multline*}
\psi(\theta_1\mathbf{1},\beta_1;\theta_0\mathbf{1},\beta_0)
= \|A\|_F^2 + \lp\|(1+a)\mathbf{1} - \mathbf{1} + (-1) A \vtanh\lp(\frac{1/2}{-1/2-1/2}(-a)\mathbf{1} + \mathbf{1}\rp)\rp\|_2^2\\
= \|A\|_F^2 + \|a \mathbf{1} - A \vtanh((a/2+1) \mathbf{1})\|_2^2
= \|A\|_F^2 + \|a \mathbf{1} - \vtanh((a/2+1) \mathbf{1})\|_2^2 = \|A\|_F^2,
\end{multline*}
using the fact that the sum of elements in each row of $A$ equals $1$, and the fact that $a = \tanh(a/2+1)$.

Fix a constant $c_0 > 0$, let $\zeta = c_0/\|A\|_F$, and let $\theta_\zeta = 1 + \zeta a$. By Lemma~\ref{lem:scale-psi}, we have
\[
\psi(\theta_\zeta\mathbf{1},\beta_\zeta;\theta_0\mathbf{1},\beta_0)
= \zeta^2\psi(\theta_1\mathbf{1},\beta_1;\theta_0\mathbf{1},\beta_0)
= \frac{c_0^2}{\|A\|_F^2} \psi(\theta_1\mathbf{1},\beta_1;\theta_0\mathbf{1},\beta_0)
= c_0^2.
\]
By Lemma~\ref{lem:kl-lb} one has
\[
D_{KL}(\Pr_{\theta_\zeta,\beta_\zeta}\| \Pr_{\theta_0,\beta_0})\le C\psi(\theta_\zeta\mathbf{1},\beta_\zeta;\theta_0\mathbf{1},\beta_0)
\le C c_0^2.
\]
By Pinsker's inequality,
\[
D_{TV}(\Pr_{\theta_\zeta,\beta_\zeta}\| \Pr_{\theta_0,\beta_0})
\le \sqrt{D_{KL}(\Pr_{\theta_\zeta,\beta_\zeta}\| \Pr_{\theta_0,\beta_0})/2}
\le \sqrt{C c_0^2/2}.
\]
If $c_0$ is sufficiently small, then by Le-Cam's inequality (Lemma~\ref{lem:la-cam}), we have that either
\[
\Pr_{h_\zeta,\beta_\zeta}[\|\hat h - h_\zeta\| \ge \|h_\zeta-h_0\|/2] \ge 0.49
\]
or
\[
\Pr_{h_0,\beta_0}[\|\hat h - h_0\| \ge \|h_\zeta-h_0\|/2] \ge 0.49.
\]
Further, notice that $\|h_\zeta-h_0\|^2/n = (a\zeta)^2 = a^2 c_0^2/\|A\|_F^2$, as required, which concludes the proof.
\section{Applications}
In this Section, we describe various applications of Theorem~\ref{thm:costis:general upper} in machine learning problems. As we shall see, in order to get the estimation rates in each application, we will often use the weaker bound with respect to $\|A\|_F$. The challenge is thus to compute the covering numbers in each specific case.

\subsection{Linear class} \label{sec:pr:linear}
In this section, we consider the general setting of logistic regression from dependent observations, first studied in \cite{daskalakis2019regression}. 
This model has already been described in Setting~\ref{setting:our maing setting}. 
We are given a feature matrix $\mathbf{X} \in \R^{n\times d}$ and the goal if to find estimation rates for the parameter vector $\theta \in \R^d$ and $\beta \in \R$, using MPLE.
We start with Theorem~\ref{thm:costis:linear}. 
Clearly, the only obstacle to applying Theorem~\ref{thm:costis:general upper} is to calculate the covering numbers. 
This is exactly what will be done in the proof. The following is just a restatement of Theorem~\ref{thm:costis:linear}.

\begin{theorem}[Generalization of \cite{daskalakis2019regression}] \label{cor:logistic}
Let $\sigma$ be sampled 
according to Setting~\ref{setting:our maing setting} and suppose $\hat{\theta} \in \R^d, \hat{\beta}\in \R$ are the MPLE estimates for this problem.  
Then, with probability at least $1 - \delta$ 
\[
\|\hat{\theta}-\theta^*\|_2^2 + |\hat{\beta}-\beta^*|^2
\lesssim \frac{d \log n + \log(1/\delta)}{\|A\|_F^2}.
\]
\end{theorem}

\begin{proof}
We will prove the bound for $\theta$, the proof for $\beta$ is completely analogous. 
First of all, we notice that to use the result of the general Theorem~\ref{thm:costis:general upper}, we need to specify what $f_\theta$ is in this case.
By the way Setting~\ref{setting:our maing setting} is defined, is is clear that $f_\theta(x) = x^\top \theta$. Hence, we have that
\begin{align*}
\frac{1}{n} \sum_{i=1}^n (f_{\hat \theta}(x_i) - f_{\theta^*}(x_i))^2 &=
\frac{1}{n} \sum_{i=1}^n  (x_i^\top(\hat\theta - \theta^*))^2\\
&= \frac{1}{n} \|\mathbf{X}\theta\|^2\\
&= \frac{(\hat\theta - \theta^*)^\top \mathbf{X}^\top \mathbf{X}(\hat\theta - \theta^*)}{n}
\end{align*}
By the lower bound $\kappa$ on the eigenvalues of $\mathbb{X}^\top \mathbf{X}/n$, this implies that 
$$
\|\hat \theta - \theta^*\|^2 \leq \frac{1}{n\kappa} \sum_{i=1}^n (f_{\hat \theta}(x_i) - f_{\theta^*}(x_i))^2
$$
Thus, by applying Theorem~\ref{thm:costis:general upper}, we immediately obtain
that with probability $\geq 1 - \delta$
$$
\|\hat \theta - \theta^*\|^2 \lesssim 
\frac{\inf_{\epsilon \geq 0}\lp(\log\frac{n}{\delta} + \epsilon n + \log N(\mathcal{F},X,\epsilon)\rp)}{\|A\|_F^2}
$$
To conclude the proof, we need to bound $N(\mathcal{F},X,\epsilon)$ and choose a suitable $\epsilon$. 
As we know, $\mathcal{F}$ is a set of functions $f_\theta$ indexed by $\theta$, together with a metric $d$. Since for all $\|x_i\|\leq M$, we have
\begin{align*}
d(f_{\theta_1}, f_{\theta_2}) &= \sqrt{\frac{\sum_{i=1}^n (x_i^\top(\theta_1 - \theta_2))^2}{n}}
\leq M \|\theta_1 - \theta_2\| 
\end{align*}
Hence, finding an $\epsilon$-net for $\mathcal{F}$ with respect to $d$ is equivalent to finding an $\epsilon/M$-net of the space of $\theta$'s with respect to the $l_2$-norm. By the assumptions of Setting~\ref{setting:our maing setting}, we know that $\theta$ lies in a ball of radius $1$. Hence, the size of an $\epsilon$-net for the unit ball is known(\cite{vershynin2018high}) to be at most $(1/\epsilon)^d$, since the space of $\theta$'s lies in $\R^d$. To conclude, we have
$$
\log N(\mathcal{F},X,\epsilon) \lesssim d\log\frac{1}{\epsilon}
$$
Setting $\epsilon = 1/n$, we have that
$$
\log\frac{n}{\delta} + \epsilon n + \log N(\mathcal{F},X,\epsilon) \lesssim d\log n + \log\frac{1}{\delta},
$$
from which the rate for $\theta$ follows.
For $\beta$, we can use the same approach, first using the bound of Theorem~\ref{thm:costis:general upper} and then 
bounding the covering number in exactly the same way. 
\end{proof}

However, improving the rate of \cite{daskalakis2019regression} is just one of the applications of the general Theorem for logistic regression. After all, the assumption $\|A\|_F = \Omega(\sqrt{n})$ seems too restrictive when we care about estimating $\theta$. If anything, having weaker interactions among the spins should help us recover $\theta$ more easily, since we are closer to the vanilla logistic regression setting. In Theorem~\ref{thm:costis:linear theta only}, we show that if we pick the matrix $\mathbf{X}$ randomly, sampling each coordinate i.i.d. from the Gaussian distribution, then with high probability we get a $1/\sqrt{n}$ rate of consistency, regardless of the matrix $A$. 
To do this, we will use the general guarantee provided by Theorem~\ref{thm:costis:general upper}. 
First, we state and prove a Lemma that brings this general rate to a more convenient form.

\begin{lemma}\label{l:aux_for_log}
Let $\mathcal{C}_1'$ be the quantity defined in \eqref{eq:def-C1prime}. Then,
\[
\mathcal{C}_1'(\mathcal{H},h^*,\beta^*)
\le \frac{1}{n\cdot \inf_{\lambda \in \R, h \in \mathcal{H}}\lp(\|\lambda A\|_F^2 + \lp\|\frac{h^* - h}{\|h^* - h\|} - \lambda A \tanh\lp(\frac{\beta^*}{\lambda}\frac{h^* - h}{\|h^* - h\|} + h^*\rp)\rp\|^2\rp)}\enspace.
\]

\end{lemma}
\begin{proof}
Let $\beta \in [-M,M]$ and $h \in \mathcal{H}$. Then,
\begin{equation}
\begin{split}
\frac{\|h-h^*\|_2^2}{\psi(h,\beta)}
&=\frac{\|h-h^*\|_2^2}{(\beta-\beta^*)^2\|A\|_F^2 + 
\lp\|h^* - \hat{h} + (\beta^*-\hat{\beta}) A \tanh\lp(\frac{\beta^*}{\hat{\beta}-\beta^*}(h^* - \hat{h}) + h^*\rp) \rp\|_2^2}\\
&= \frac{1}{\lp(\frac{\hat{\beta} - \beta^*}{\|\hat{h} - h^*\|}\rp)^2\|A\|_F^2 + 
\lp\|\frac{h^* - \hat{h}}{\|h^* - \hat{h}\|} + \frac{\beta^*-\hat{\beta}}{\|h^* - \hat{h}\|} A \tanh\lp(\frac{\beta^*}{\frac{\hat{\beta} - \beta^*}{\|h^* - \hat{h}\|}}\frac{h^* - \hat{h}}{\|h^* - \hat{h}\|} + h^*\rp) \rp\|^2} \\
&= \frac{1}{\inf_{\lambda \in \mathbb{R}}\|\lambda A\|_F^2 +
\lp\|\frac{h^* - \hat{h}}{\|h^* - \hat{h}\|} + \lambda A \tanh\lp(\frac{\beta^*}{\lambda}\frac{h^* - \hat{h}}{\|h^* - \hat{h}\|} + h^*\rp) \rp\|^2}.
\end{split}
\end{equation}
Dividing both sides of the inequality by $n$ and taking supremum over $\beta \in [-M,M]$ and $h \in \mathcal{H}$, the result follows.
\end{proof}

Now, we are ready to prove Theorem~\ref{thm:costis:linear theta only}. We present the statement in a slightly more detailed way below as Theorem~\ref{cor:random_logistic}.

\begin{theorem}\label{cor:random_logistic}
Suppose $\sigma$ is sampled according to Setting~\ref{setting:our maing setting}. We drop all assumptions about $\mathbf{X}$ and instead assume that we sample
 $\mathbf{X}_{ij} \sim N(0,1)$ independently for all $i \in [n],j \in [d]$. Let $\hat \theta,\hat \beta$ be the estimates of the MPLE. Then, with probability $1 - \delta$
$$
\|\hat \theta - \theta^*\| \leq C(M)\me^{\sqrt{\log n}} \frac{\sqrt{d \log n + \log(1/\delta)}}{\sqrt{n}} \max\lp(\frac{\|A\|_2\sqrt{ \log(1/\delta)+ \log n + d\log\log n + d\log( \log(1/\delta)) }}{\|A\|_F}, 1\rp)
$$
Consequently, $\hat \theta$ is an $\alpha_n$-consistent estimator, where 
$$
\alpha_n = \me^{\sqrt{\log n}}\sqrt{\frac{d\log n}{n}}\max\lp(\frac{\sqrt{d \log\log n + \log n}\|A\|_2}{\|A\|_F},1\rp)
$$
\end{theorem}
\begin{proof}
We use the rate for $\theta$ provided in Theorem~\ref{thm:costis:general upper}, combined with Lemma~\ref{l:aux_for_log}, which simplifies the rate.
This gives us that with probability $\geq 1 - \delta$:

\begin{equation}\label{eq:squared}
\|\mathbf{X}\hat\theta - \mathbf{X}\theta^*\|^2 \lesssim \frac{ \log \frac{n}{\delta}  + \inf_{\epsilon\geq 0} \lp(n \epsilon + \log\ N\lp(\mathcal{F}, X,\epsilon\rp)\rp) }{\inf_{\lambda \in \R, \|\theta\| \leq M}\lp(\|\lambda A\|_F^2 + \lp\|\frac{\mathbf{X}\theta^* - \mathbf{X}\theta}{\|\mathbf{X}\theta^* - \mathbf{X}\theta\|} - \lambda A \tanh\lp(\frac{\beta^*}{\lambda}\frac{\mathbf{X}\theta^* - \mathbf{X}\theta}{\|\mathbf{X}\theta^* - \mathbf{X}\theta\|} + \mathbf{X}\theta^*\rp)\rp\|^2\rp)}
\end{equation}
We will first show that the denominator is bounded with high probability. For convenience, since we want a bound for $\|\mathbf{X}\hat\theta - \mathbf{X}\theta^*\|$, we will work with the denominator without the squares. This is without loss of generality, since for $a,b>0$ we have $a^2 + b^2 = \Theta((a+b)^2)$.  
First, a simple scale invariance of the expression yields
\begin{align*}
&\inf_{\lambda \in \R, \|\theta\| \leq 1}\lp(\|\lambda A\|_F + \lp\|\frac{\mathbf{X}\theta^* - \mathbf{X}\theta}{\|\mathbf{X}\theta^* - \mathbf{X}\theta\|} - \lambda A \tanh\lp(\frac{\beta^*}{\lambda}\frac{\mathbf{X}\theta^* - \mathbf{X}\theta}{\|\mathbf{X}\theta^* - \mathbf{X}\theta\|} + \mathbf{X}\theta^*\rp)\rp\|\rp)\\
&\quad\quad\quad =  \frac{1}{\|\mathbf{X}\theta^* - \mathbf{X}\theta\|} \inf_{\lambda \in \R, \|\theta\| \leq 1}\lp(\|\lambda A\|_F + \lp\|\mathbf{X}\theta^* - \mathbf{X}\theta - \lambda A \tanh\lp(\frac{\beta^*}{\lambda}(\mathbf{X}\theta^* - \mathbf{X}\theta) + \mathbf{X}\theta^*\rp)\rp\|\rp)\numberthis \label{eq:scale}
\end{align*}
Let's focus on the infimum of the latter expression. In particular, we fix $\theta$ with bounded norm. 
Also, notice that we can replace $A$ by $A/\|A\|_2$ and $\beta^*$ by $\beta^* \|A\|$ and the probability model remains the same. 
This is done without loss of generality to ensure that $A$ has unit spectral norm, which will simplify the computations later on. Notice that since $\|A\|$ is bounded, $\beta^*$ remains bounded after the transformation. Also, notice that after this transformation $\|A\|_F = \Omega(1)$.

We will show that for all values of $\lambda$, at least one of the two terms is large.
First of all, by triangle inequality 
$$
\lp\|\mathbf{X}\theta^* - \mathbf{X}\theta - \lambda A \tanh\lp(\frac{\beta^*}{\lambda}(\mathbf{X}\theta^* - \mathbf{X}\theta) + \mathbf{X}\theta^*\rp)\rp\| \geq \|\mathbf{X}\theta^* - \mathbf{X}\theta\| - 
\lp\|\lambda A \tanh\lp(\frac{\beta^*}{\lambda}(\mathbf{X}\theta^* - \mathbf{X}\theta) + \mathbf{X}\theta^*\rp)\rp\|
$$
We will show that with high probability over the choice of $\mathbf{X}$, this difference is large.
Let's consider $\|\mathbf{X}(\theta^* - \theta)\|$ first. Notice that this each coordinate of the vector $\mathbf{X}(\theta^* - \theta)$ is independent from the rest and distributed as a gaussian $N(0,\|\theta - \theta^*\|^2)$. Thus, we can apply Theorem 3.1.1 of \cite{vershynin2018high}, which gives us that
\begin{equation}\label{eq:subg}
\lp\|\|\mathbf{X}(\theta^* - \theta)\| - \|\theta - \theta^*\|\sqrt{n}\rp\|_{\psi_2} \leq K \|\theta - \theta^*\|
\end{equation}
for some constant $K$. This means that the norm of the vector is well concentrated around $ \|\theta - \theta^*\|\sqrt{n}$.
Now, let's consider the second term. For simplicity, denote by $\mathbf{u} = \frac{\beta^*}{\lambda}(\mathbf{X}\theta^* - \mathbf{X}\theta) + \mathbf{X}\theta^*$. In the computations below, $\Exp$ refers to expectation with respect to the distribution of $\mathbf{X}$. 
\begin{align*}
\Exp \lp\| A \tanh\lp(\frac{\beta^*}{\lambda}(\mathbf{X}\theta^* - \mathbf{X}\theta) + \mathbf{X}\theta^*\rp)\rp\|^2 &= 
\Exp \sum_i \lp(\sum_j A_{ij} \tanh(\mathbf{u}_j)\rp)^2 \\
&= \Exp \sum_i \sum_j A_{ij}^2 \tanh(\mathbf{u}_j)^2 \leq \|A\|_F^2 
\end{align*}
We used the fact that the coordinates of $\mathbf{u}$ are all independent from each other and have mean $0$. This means that the coordinates of $\tanh(\mathbf{u})$ are also independent and have mean $0$. We conclude that
$$
\Exp \lp\| A \tanh\lp(\frac{\beta^*}{\lambda}(\mathbf{X}\theta^* - \mathbf{X}\theta) + \mathbf{X}\theta^*\rp)\rp\| \leq \sqrt{\Exp \lp\| A \tanh\lp(\frac{\beta^*}{\lambda}(\mathbf{X}\theta^* - \mathbf{X}\theta) + \mathbf{X}\theta^*\rp)\rp\|^2} \leq \|A\|_F
$$
We just bounded the mean of the function $f(\mathbf{X}) = \lp\| A \tanh\lp(\frac{\beta^*}{\lambda}(\mathbf{X}\theta^* - \mathbf{X}\theta) + \mathbf{X}\theta^*\rp)\rp\| $. The next step is to show that it is concentrated around that value. 
To do that, we will use the gaussian isoperimetric inequality, see e.g. Theorem 5.1 in \cite{boucheron2013concentration}.
We can view $\mathbf{X}$ as a $n\times d$ vector of independent normal gaussian variables. The $2$ norm of this vector is the frobenius norm of $\mathbf{X}$. Hence, to use the isoperimetric inequality, we need to bound the Lipschitzness of $f$ w.r.t. the frobenius norm of $\mathbf{X}$. Let $\mathbf{X},\mathbf{X}' \in \R^{n\times d}$ be two matrices. We denote by $\mathbf{X}_i$ as the $i$-th row of matrix $\mathbf{X}$.  Then, the difference in value is
\begin{align*}
|f(\mathbf{X}) - f(\mathbf{X}')| &\leq 
\lp\| A \tanh\lp(\frac{\beta^*}{\lambda}(\mathbf{X}\theta^* - \mathbf{X}\theta) + \mathbf{X}\theta^*\rp) - A \tanh\lp(\frac{\beta^*}{\lambda}(\mathbf{X}'\theta^* - \mathbf{X}'\theta) + \mathbf{X}'\theta^*\rp)\rp\|\\
&\leq \|A\| \lp\|  \tanh\lp(\frac{\beta^*}{\lambda}(\mathbf{X}\theta^* - \mathbf{X}\theta) + \mathbf{X}\theta^*\rp) -  \tanh\lp(\frac{\beta^*}{\lambda}(\mathbf{X}'\theta^* - \mathbf{X}'\theta) + \mathbf{X}'\theta^*\rp)\rp\|\\
&\leq \|A\| \lp\|  \frac{\beta^*}{\lambda}(\mathbf{X}\theta^* - \mathbf{X}\theta) + \mathbf{X}\theta^* -  \frac{\beta^*}{\lambda}(\mathbf{X}'\theta^* - \mathbf{X}'\theta) + \mathbf{X}'\theta^*\rp\|\\
&\leq \|A\|\lp(\frac{\beta^*}{|\lambda|} \sqrt{\sum_{i=1}^n ((\mathbf{X}_i - \mathbf{X}'_i)^\top(\theta^* - \theta))^2} +
\sqrt{\sum_{i=1}^n ((\mathbf{X}_i - \mathbf{X}'_i)^\top\theta^*)^2} \rp)\\
&\leq \|A\|\lp(\frac{\beta^*}{|\lambda|} \|\theta^* - \theta\| \|\mathbf{X} - \mathbf{X}'\|_F + \|\theta^*\|\|\mathbf{X} -  \mathbf{X}'\|_F\rp) \\
&= \|A\|\lp(\frac{\beta^*}{|\lambda|} \|\theta^* - \theta\|+ \|\theta^*\|\rp)\|\mathbf{X} -  \mathbf{X}'\|_F
\end{align*}

Now, since the entries of $\mathbf{X}$ are i.i.d. gaussians and $\|A\|,\|\theta^*\|$ are bounded by constants, from the gaussian isoperimetric inequality (\cite{vershynin2018high,boucheron2013concentration}) we get that 
with probability $1 - \me^{- t^2/2}$
\begin{align*}
\lp\| \lambda A \tanh\lp(\frac{\beta^*}{\lambda}(\mathbf{X}\theta^* - \mathbf{X}\theta) + \mathbf{X}\theta^*\rp)\rp\| &\leq \Exp \lp\|\lambda A \tanh\lp(\frac{\beta^*}{\lambda}(\mathbf{X}\theta^* - \mathbf{X}\theta) + \mathbf{X}\theta^*\rp)\rp\|  + ct (\|\theta^* - \theta\| + \lambda)\\
&\leq \|\lambda A\|_F +  ct (\|\theta^* - \theta\| + \lambda)
\end{align*}

Also, by \eqref{eq:subg} we have that with probability $1 - \me^{-t^2/2}$ 
$$
\|\mathbf{X}\theta^* - \mathbf{X}\theta\| \geq \|\theta^* - \theta\| \sqrt{n} - t K \|\theta - \theta^*\|
$$
Hence, we get that with probability $1 - \me^{-t^2}$
$$
\lp\|\mathbf{X}\theta^* - \mathbf{X}\theta - \lambda A \tanh\lp(\frac{\beta^*}{\lambda}(\mathbf{X}\theta^* - \mathbf{X}\theta) + \mathbf{X}\theta^*\rp)\rp\| \geq c_1\|\theta^* - \theta\| \sqrt{n} - \|\lambda A\|_F - t C  \|\theta - \theta^*\| - tc\lambda
$$
Now we distinguish cases for $\lambda$. First of all, suppose that 
$$
\lambda \leq C \sqrt{n}\|\theta- \theta^*\| \min\lp(\frac{1}{\sqrt{ \log(1/\delta') + \frac{\log n}{2}}}, \frac{1}{\|A\|_F}\rp) := \kappa
$$
for a suitable constant $C$. Then, applying the previous result, with probability at least $1 - (\delta'/\sqrt{n}) $ we have that
\begin{equation}\label{eq:second}
\lp\|\mathbf{X}\theta^* - \mathbf{X}\theta - \lambda A \tanh\lp(\frac{\beta^*}{\lambda}(\mathbf{X}\theta^* - \mathbf{X}\theta) + \mathbf{X}\theta^*\rp)\rp\| \geq C'\|\theta^* - \theta\| \sqrt{n} 
\end{equation}
On the other hand, if 
$$
\lambda \geq \kappa
$$
then the first term of the expression is 
$$
\|\lambda A\|_F \geq  C \sqrt{n}\|\theta- \theta^*\| \min\lp(\frac{\|A\|_F}{\sqrt{ \log(1/\delta')+ \frac{\log n}{2}}}, 1\rp)
$$
Note that this bound holds with probability $1$. Now denote 
$$
g(\lambda) = \lp\|\mathbf{X}\theta^* - \mathbf{X}\theta - \lambda A \tanh\lp(\frac{\beta^*}{\lambda}(\mathbf{X}\theta^* - \mathbf{X}\theta) + \mathbf{X}\theta^*\rp)\rp\|
$$
To get a bound for the infimum, we need the bound of \eqref{eq:second} to hold for all $|\lambda| \leq \kappa$. To do that, we notice that the expression of the infimum is Lipschitz with respect to $\lambda$ and the Lipschitz constant is at most $C''\sqrt{n}$ for some constant $C''$ that depends on $\|A\|_F,\|A\|_2$. Now suppose we pick an $\epsilon$-net for the set of $|\lambda| \leq \kappa$. Clearly, we can pick such a net that has at most $2\kappa/\epsilon$ elements. Suppose that for a point in the net $\lambda_1$ we have $g(\lambda_1) \geq C'\|\theta^* - \theta\| \sqrt{n}$ and consider another point $\lambda_2$ with $|\lambda_2 - \lambda_1|<\epsilon$. We then have
$$
f(\lambda_2) \geq f(\lambda_1) - C''\sqrt{n}\epsilon \geq C'\|\theta^* - \theta\| \sqrt{n} - C''\sqrt{n}\epsilon
$$
Hence, in order for to obtain a similar bound for $f(\lambda_2)$, we need to choose $\epsilon = \Theta(\|\theta^* - \theta\| )$. The size of this net will be $O(\kappa/\epsilon) = O(\sqrt{n})$. Hence, if \eqref{eq:second} holds for all points in this net, it will hold for all $|\lambda| \leq \kappa$. The probability that \eqref{eq:second} holds for all points in the net is at least $1-\delta'$, which easily follows from a union bound. 
Hence, with probability at least $1 - \delta'$ we have that

\begin{align*}
&\inf_{\lambda \in \R}\lp(\|\lambda A\|_F + \lp\|\frac{\mathbf{X}\theta^* - \mathbf{X}\theta}{\|\mathbf{X}\theta^* - \mathbf{X}\theta\|} - \lambda A \tanh\lp(\frac{\beta^*}{\lambda}\frac{\mathbf{X}\theta^* - \mathbf{X}\theta}{\|\mathbf{X}\theta^* - \mathbf{X}\theta\|} + \mathbf{X}\theta^*\rp)\rp\|\rp)\\
&\quad\quad\quad= \Omega\lp(\sqrt{n}\|\theta- \theta^*\| \min\lp(\frac{\|A\|_F}{\sqrt{ \log(1/\delta')+ \log n}}, 1\rp)\rp)\numberthis \label{eq:unscaled}
\end{align*}
Now, from \eqref{eq:scale} we also need to upper bound $\|\mathbf{X}\theta - \mathbf{X}\theta^*\|$. By the concentration of the norm, which we have already used, we get that with probability $1 - \me^{-t^2}$
$$
\|\mathbf{X}\theta^* - \mathbf{X}\theta\| \leq \|\theta^* - \theta\| \sqrt{n} + t K \|\theta - \theta^*\|
$$
Hence, with probability $1 - \delta'$
$$
\|\mathbf{X}\theta^* - \mathbf{X}\theta\| \leq C'''\|\theta^* - \theta\| (\sqrt{n} + \sqrt{ \log (1/\delta')})
$$
Combining this with \eqref{eq:unscaled}, we get that for any $\delta' \geq \me^{-n}$, with probability $1 - \delta'$
we have 
\begin{align*}
&\inf_{\lambda \in \R}\lp(\|\lambda A\|_F + \lp\|\frac{\mathbf{X}\theta^* - \mathbf{X}\theta}{\|\mathbf{X}\theta^* - \mathbf{X}\theta\|} - \lambda A \tanh\lp(\frac{\beta^*}{\lambda}\frac{\mathbf{X}\theta^* - \mathbf{X}\theta}{\|\mathbf{X}\theta^* - \mathbf{X}\theta\|} + \mathbf{X}\theta^*\rp)\rp\|\rp)\\ 
&\quad\quad\quad= \Omega\lp( \min\lp(\frac{\|A\|_F}{\sqrt{ \log(1/\delta')+ \log n}}, 1\rp)\rp)
\end{align*}
Now, to conclude the analysis of the denominator of the rate, we need to take the infimum over all $\theta$ with bounded norm.
First of all, consider the set
$$
\mathcal{U} = \lp\{\frac{\mathbf{X}\theta^* - \mathbf{X}\theta}{\|\mathbf{X}\theta^* - \mathbf{X}\theta\|} : \|\theta\| \leq 1\rp\}
$$
We will follow a covering argument for $\mathcal{U}$ similar to what we did previously. First of all, we prove that for $\mathbf{a}, \mathbf{a}' \in \mathcal{U}$, if these two vectors are close, then the value of the infimum does not change much. 
Notice that to prove that, it is enough to prove that for a fixed $\lambda$, the expression doesn't change much. For a fixed $\lambda$, we have
\begin{align*}
&\lp|\lp\|\mathbf{a} - \lambda A \tanh\lp(\frac{\beta^*}{\lambda}\mathbf{a} + \mathbf{X}\theta^*\rp)\rp\| - \lp\|\mathbf{a}' - \lambda A \tanh\lp(\frac{\beta^*}{\lambda}\mathbf{a}' + \mathbf{X}\theta^*\rp)\rp\| \rp| \\
&\quad\quad\quad\leq\|\mathbf{a} - \mathbf{a}'\| + |\lambda|\|A\|\lp\|\tanh\lp(\frac{\beta^*}{\lambda}\mathbf{a} + \mathbf{X}\theta^*\rp) - \tanh\lp(\frac{\beta^*}{\lambda}\mathbf{a}' + \mathbf{X}\theta^*\rp)\rp\|\\
&\quad\quad\quad\leq (1 + \beta^* \|A\|)\|\mathbf{a} - \mathbf{a}'\| 
\end{align*}
Thus, the infimum is a Lipschitz function of $\mathbf{a}$ and the Lipschitzness is a constant. Now let's denote 
$$
\gamma = \min\lp(\frac{\|A\|_F}{\sqrt{ \log(1/\delta')+ \log n}}, 1\rp)
$$
Consider an $\epsilon$-net on the set $\mathcal{U}$ with respect to the $2$-norm. 
If the value of the infimum is at least $\gamma$ for all the points in the net, then any other point has value at least $\gamma -(1 + \beta^*\|A\|) \epsilon$. Thus, if we choose $\epsilon = \Theta(\gamma)$, then we get that for all $\mathbf{a}\in \mathcal{U}$, the infimum is $\Omega(\gamma)$. To determine the probability of this event, we can simply take a
union bound over all $\mathcal{N}(\mathcal{U},\|\cdot\|_2,\gamma)$ points in the net. Thus, our task reduces to finding a bound for this covering number.
To do that, we need a few facts about the singular values of the matrix $\mathbf{X}$.
We use a result from \cite{vershynin2010introduction}. According to this, with probability at least $1-\me^{-ct^2}$ we have
$$
\lp\|\frac{1}{n}\mathbf{X}^\top \mathbf{X} - I\rp\| \leq \sqrt{\frac{d}{n}} + \frac{t}{\sqrt{n}}
$$
From now on, we will assume this is the case without explicitly mentioning the probability. 
This implies that with probability at least $1 - \me^{-cn}$, the eigenvalues of $\mathbf{X}^\top \mathbf{X}/n$ are upper and lower bounded by constants $\lambda_{min},\lambda_{max}$. 

By definition of the set $\mathcal{U}$, it is enough to assume without loss of generality that $\theta$ with $\|\theta^* - \theta\| = 1$ for all $\theta$ of interest(simply divide by the norm $\|\theta^* - \theta\|$. 
Now, consider $\theta_1,\theta_2$ with this property. We have that
\begin{align*}
\lp|\frac{\mathbf{X}\theta^* - \mathbf{X}\theta_1}{\|\mathbf{X}\theta^* - \mathbf{X}\theta_1\|} - \frac{\mathbf{X}\theta^* - \mathbf{X}\theta_2}{\|\mathbf{X}\theta^* - \mathbf{X}\theta_2\|}\rp| &\leq \frac{\|\mathbf{X} \theta_2 - \mathbf{X}\theta_1\|}{\min(\|\mathbf{X}\theta^* - \mathbf{X}\theta_1\|,\|\mathbf{X}\theta^* - \mathbf{X}\theta_2\|)}\\
&\leq \frac{\sqrt{n\lambda_{max}} \|\theta_2 - \theta_1\|}{\sqrt{n\lambda_{min}}}\\
&= O(\|\theta_2 - \theta_1\|)
\end{align*}
Thus, to get an $\epsilon$-net for $\mathcal{U}$, it suffices to construct an $O(\epsilon)$-net for $\theta^*+ B(0,1)$, where $B(0,1)$ is the unit sphere in $\R^d$. However, it is known that for this set, we can construct an $\epsilon$-net with $O((1/\epsilon)^d)$ elements. From our previous analysis, we should choose $\epsilon = O(\gamma)$. Thus, by applying a union bound over the points in the net and then using the Lipschitzness, we get that 
with probability at least $1 -  (1/\gamma)^d \delta'$, 
$$
\inf_{\lambda \in \R, \|\theta\| = O(1)}\lp(\|\lambda A\|_F + \lp\|\frac{\mathbf{X}\theta^* - \mathbf{X}\theta}{\|\mathbf{X}\theta^* - \mathbf{X}\theta\|} - \lambda A \tanh\lp(\frac{\beta^*}{\lambda}\frac{\mathbf{X}\theta^* - \mathbf{X}\theta}{\|\mathbf{X}\theta^* - \mathbf{X}\theta\|} + \mathbf{X}\theta^*\rp)\rp\|\rp) = \Omega(\gamma)
$$
We would like to choose $\delta'$ so that the probability of error is at most $\delta$. This gives us
\begin{align*}
(1/\gamma)^d\delta' < \delta \implies 
d \log (\gamma) + \log(1/ \delta') > \log (1/\delta) 
\end{align*}
If $\gamma = 1$, this inequality reduces to $\delta' < \delta$. Otherwise, since $\|A\|_F \geq 1$, it suffices to satisfy  the inequality
$$
\log(1/\delta') -\frac{d}{2} \log( \log(1/\delta') + \log n) > \log (1/\delta)
$$
We notice that the preceding inequality is satisfied if we set 
$$
\log(1/\delta') = \log(1/\delta) + d\log( \log(1/\delta) + \log n)
$$
Hence, we conclude that with probability at least $1 - \delta$, the infimum is at least
\begin{align*}
&\min\lp(\frac{\|A\|_F}{\sqrt{ \log(1/\delta)+ \log n + d\log( \log(1/\delta) + \log n) }}, 1\rp)\\
&\quad\quad\quad= \Omega\lp(\min\lp(\frac{\|A\|_F}{\sqrt{ \log(1/\delta)+ \log n + d\log\log n + d\log( \log(1/\delta)) }}, 1\rp)\rp)
\end{align*}
Now, all that remains is to compute the numerator of the rate. Notice that we are in the event where the eigenvalues of $\mathbb{X}^\top \mathbf{X}/n$ are bounded by positive constants. Hence, suppose $M$ is an upper bound for the eigenvalues. We have
\begin{align*}
d(f_{\theta_1}, f_{\theta_2}) &= \sqrt{\frac{\|\mathbf{X}\theta\|^2}{n}}\\
&\leq M\|\theta_2 - \theta_1\|
\end{align*}
This means that it is enough to find an $\epsilon$-net for the space of $\theta$'s with respect to the $l_2$-norm.
This computation was also done in Theorem~\ref{cor:logistic} and the value we found was
$$
\sqrt{d \log n + \log(1/\delta)}
$$
if we want error probability $\delta$.

Hence, we get that for any $\delta > \me^{-n}$, with probability $1-\delta$ we have:
$$
\|\mathbf{X}\hat\theta - \mathbf{X}\theta^*\| \lesssim  \sqrt{d \log n + \log(1/\delta)} \max\lp(\frac{\sqrt{ \log(1/\delta)+ \log n + d\log\log n + d\log( \log(1/\delta)) }}{\|A\|_F}, 1\rp)
$$
Now, we need to consider the constant $C(M)$ that is hidden with $\lesssim$. The result of Theorem~\ref{thm:costis:general upper} states that $C(M)$ depends single-exponentially on the bound $M$. In our case, the parameters are bounded. However, $\mathbf{X}\theta$ is a random quantity, hence it's maximum entry is not necessarily bounded. We address this issue now. We immediately notice that each 
coordinate of $\mathbf{X}\theta$ is a gaussian with $0$ mean and variance at most $M$. Also, the coordinate values are all independent, since the rows of $\mathbf{X}$ are all independent. Hence, the maximum coordinate is distributed as the maximum of $n$ independent gaussian variables, so it's expectation is at most $M\sqrt{2\log n}$.
Also, we know that the maximum is also well concentrated, which means that with probability at least $1-\delta$,
$$
\|\mathbf{X}\theta\|_\infty \leq M\sqrt{2\log n} + \sqrt{\log (1/\delta)}
$$
Now, if $\delta >1/ \text{poly}(n)$, we have that 
$$
\|\mathbf{X}\theta\|_\infty  = O(\sqrt{\log n})
$$
This implies that the constant in the rate is at most $O\lp(\me^{\sqrt{\log n}}\rp)$,
which is smaller than any polynomial of $n$.
The result now follows by recalling that we normalized $A$ by $\|A\|_2$ in the beginning.

\end{proof}

\begin{remark}
Since $\|A\|_F \geq \|A\|_2$, Corollary~\ref{cor:random_logistic} always guarantees a $\tilde O(d/\sqrt{n})$ consistency rate, regardless of the norms of the matrix $A$. We have ignored the $\exp(\sqrt{\log n})$ part of the rate, since it is smaller than any polynomial of $n$.
\end{remark}

\subsection{Sparse Logistic Regression}\label{sec:pr:sparse}

Another interesting application of the general Theorem~\ref{thm:costis:general upper} involves the case of Sparse Logistic Regression. In this case, we have the familiar Setting~\ref{setting:our maing setting} of logistic regression, with an unknown vector $\theta\in \R^d$ and a fixed matrix $\mathbf{X} \in \R^{n\times d}$. The extra assumption comes in the form of sparsity of the vector $\theta$. We will assume, as is standard in the literature, that $\theta$ is $l_1$-sparse. This also enables us to run MPLE efficiently. Clearly, we expect 
the rate to depend on $s$ in that case. The quantity that determines the rate is the metric entropy for sparse subsets of $\R^d$. For convenience, we restate it to include a more detailed description of the assumptions. 

\begin{theorem}
Let $\sigma$ be sampled according to setting ~\ref{setting:our maing setting}. Additionally, assume that $\|\theta^*\|_1 \leq s$. We also assume the restricted eigenvalue property that is stated in Setting~\ref{setting:our maing setting}, which means that for all $\|\theta\|_1 \leq s$ we have
$$
\|\mathbf{X} \theta\| \geq \kappa \sqrt{n}\|\theta\|
$$
where $\kappa$ is a constant.
Let $\hat\theta,\hat\beta$ be the estimates we get for MPLE constrained in the sparse region.
Then, w.pr. $\ge 1-\delta$,

\[
\|\hat{\theta}-\theta^*\|_2^2 + |\hat{\beta}-\beta^*|^2
\lesssim  \frac{(n^2 s\log(d))^{1/3} + \log(1/\delta)}{\|A\|_F^2}.
\]
\end{theorem}

\begin{proof}

Notice that since $\hat\theta,\theta^*$ are sparse, by the restricted eigenvalue property we get
\begin{align*}
\frac{\sum_{i=1}^n (f_{\hat\theta} - f_{\theta^*})^2}{n} &= \frac{\|\mathbf{X}(\hat{\theta} - \theta^*)\|^2}{n} \\
&\geq \kappa^2 \|\hat{\theta}- \theta^*\|
\end{align*}
Thus, we can immediately apply Theorem~\ref{thm:costis:general upper} to obtain with probability $1 - \delta$
$$
\|\hat{\theta}-\theta^*\|_2^2 + |\hat{\beta}-\beta^*|^2
\lesssim  \frac{\inf_{\epsilon \ge 0} \lp(\log \frac{n}{\delta} + \epsilon n + \log N(\mathcal{F},X,\epsilon)\rp)}{\|A\|_F^2}
$$

All that remains now is to bound the covering number $N(\mathcal{F},X,\epsilon)$. 
Since we are in Setting~\ref{setting:our maing setting}, the norms $\|x_i\|$ are uniformly bounded. 
This means that, similarly to Theorems~\ref{cor:logistic},~\ref{cor:random_logistic}, we have that

$$
d(f_{\theta_1},f_{\theta_2}) \leq M \|\theta_1 - \theta_2\|
$$
Thus, is suffices to compute the covering number of the space of $\theta$'s with respect to the $l_2$-norm. The space of $\theta$'s is $\mathcal{U} = \{\theta\in \R^d: \|\theta\|_1 \leq s\}$.
Thus, we need the covering numbers for this set. A result from \cite{raskutti2011minimax} states that 
$$
\log N(\mathcal{U}, \|\cdot\|_2,\epsilon) \leq C\frac{s \log d}{\epsilon^2}
$$
Thus, to find the optimal $\epsilon$, we should minimize the expression
$$
\epsilon n + \frac{s \log d}{\epsilon^2}.
$$
This is clearly minimized when the two terms are equal, which means
$$
\epsilon n = \frac{s \log d}{\epsilon^2} \implies \epsilon = \lp(\frac{s\log d}{n}\rp)^{1/3}
$$
Hence, the optimal value for the numerator is 
$$
\log (1/\delta) + (n^2s\log (d))^{1/3}
$$
The result follows. 
\end{proof}

\subsection{Neural Network Regression} \label{sec:pr:neural}
Another application of Theorem~\ref{thm:costis:general upper} involves the case where the regression function $f_\theta$ is not linear, but a neural network. Since neural networks are often overparametrized, we cannot hope to recover the parameter $\theta$ is this case. However, we can still apply Theorem~\ref{thm:costis:general upper}, 
which will yield guarantees for estimating the output of the network. 
We assume we are in Setting~\ref{set2}, which means that all the parameters of the network are bounded. Again, the crucial quantity is the metric entropy for the image of these neural networks, for which we have sufficiently strong guarantees from the literature. 

\neural*
\begin{proof}

To begin with, we can directly apply Theorem~\ref{thm:costis:general upper} to obtain with probability $\geq 1 - \delta$:
$$
\frac{\sum_{i=1}^n (f_{\hat\theta} - f_{\theta^*})^2}{n} + |\hat\beta - \beta^*|^2 \lesssim \frac{\inf_{\epsilon \ge 0} \lp(\log \frac{n}{\delta} + \epsilon n + \log N(\mathcal{F},X,\epsilon)\rp)}{\|A\|_F^2}
$$
Hence, we only have to bound the covering numbers of the set $\mathcal{F}$. By the way we have defined distance $d$, it is essentially the $l_2$-norm of the difference in the outputs of two networks, divided by $\sqrt{n}$. Hence, we define $\mathcal{H}_X$ to be the set of outputs of neural networks satisfying the bounds of Setting~\ref{set2}, when the input is the matrix $X$, composed of features $x_1,\ldots,x_n$. Then, it suffices to find an $\epsilon \sqrt{n}$-net for $\mathcal{H}_X$ with respect to the $l_2$-norm.
Such a bound is given in Theorem 3.3 of \cite{bartlett2017spectrally}, which gives us (using the fact that we have width bounded by $d$)
$$
\log N(\mathcal{F},X,\epsilon) \leq \frac{\log d \sum_{i=1}^n \|x_i\|^2}{(\sqrt{n}\epsilon)^2} R^2 = \frac{n K^2R^2 \log d}{n\epsilon^2}
$$
Clearly, to find the $\epsilon$ that minimizes the numerator, we set
$$
n\epsilon = \frac{K^2R^2 \log d}{\epsilon^2} \implies 
\epsilon = \lp(\frac{K^2R^2 \log d}{n}\rp)^{1/3}
$$
Using this value of $\epsilon$, the numerator becomes 
$$
\log(1/\delta) + (n^2K^2R^2\log d)^{1/3},
$$
which concludes the proof.
\end{proof}
\subsection{The Curie-Weiss model}
The power of obtaining separate rates for estimation of parameters $\theta,\beta$ can be easily showcased when considering the Curie Weiss model with external field. 
In this model, the external field has a fixed direction  $h \in \R^n$. Without loss of generality, we can assume that $A$ is the all $1/n$'s matrix, since the error introduced  by this modification is $O(1/n)$. The probability distribution becomes
\begin{equation}\label{eq:curie}
\Pr[\sigma = y] \propto \exp\lp(\theta^* \sum_{i=1}^n y_i h_i
+ \frac{\beta^*}{n} \sum_{i,j}  y_i y_j\rp)
\end{equation}
We will identify cases where it is possible to estimate $\theta$ or $\beta$. For simplicity, we will assume that the coordinates of $h$ lie in $\{-1,1\}$.
The result of \cite{daskalakis2019regression} does not give us any rate, since $\|A\|_F = 1$. 
However, as we will see, the property that determines the rate has to do with how much $h$ is close to being an eigenvector of $A$. 

\begin{corollary}
Let $\sigma$ be sampled according to \eqref{eq:curie} with $|\beta^*|,|\theta^*| \leq M$. Suppose the external field $h$ has a fraction $\alpha$ of coordinates equal to $1$ and a fraction $1-\alpha$ equal to $-1$. Let $\hat \theta,\hat \beta$ be the estimates when we run MPLE on $[-M,M]^2$. Then, with probability at least $1-\delta$
$$
|\hat{\theta} - \theta^*| \leq C\frac{\sqrt{\log n + \log(1/\delta)}}{\sqrt{4n(1 - (2\alpha -1)^2)}}
$$
\end{corollary}
\begin{proof}
 Let $U,V\setminus U$ be the subsets of vertices that have $h_i = 1$ and $h_i = -1$ respectively. 
Let's start with acquiring a rate for $\theta^*$. 
 Since each parameter is one dimensional and bounded, according to Theorem REF, with probability $1-\delta$
 $$
 |\hat \theta - \theta^*|\|h\| \leq C\frac{\sqrt{\log n + \log (1/\delta)}}{\inf_{\lambda \in \R} \lp(|\lambda| + \lp\|\frac{h}{\|h\|} - \lambda A \tanh\lp(\frac{\beta^*}{\lambda} \frac{h}{\|h\|} + \theta^* h\rp)\rp\|\rp)}
 $$
 which implies that 
 $$
  |\hat \theta - \theta^*| \leq C\frac{\sqrt{\log n + \log (1/\delta)}}{\inf_{\lambda \in \R} \lp(|\lambda| + \lp\|h - \lambda A \tanh\lp(\frac{\beta^*}{\lambda} h + \theta^* h\rp)\rp\|\rp)}
 $$
 We have that
 $$
\inf_{\lambda \in \R} \lp(|\lambda| + \lp\|h - \lambda A \tanh\lp(\frac{\beta^*}{\lambda} h + \theta^* h\rp)\rp\|\rp)\geq \inf_{\lambda \in \R} \lp\|h - \lambda A \tanh\lp(\frac{\beta^*}{\lambda} h + \theta^* h\rp)\rp\| 
 $$

The $\tanh$ vector has a special structure. 
Specifically, coordinates corresponding to nodes in $U$ have a common value $\mu_1(\lambda)$ and  nodes in $V \setminus A$ have a common value $\mu_2(\lambda)$. Hence, if we denote by $\mathbf{1}$ the all ones vector, we have
\begin{align*}
\inf_{\lambda \in \R} \lp\|h - \lambda A \tanh\lp(\frac{\beta^*}{\lambda} h + \theta^* h\rp)\rp\|  &=
\inf_{\lambda \in \R}\lp( \lp\|h - \lambda \lp(\frac{\alpha n \mu_1(\lambda) + (1 - \alpha)n \mu_2(\lambda)}{n} \mathbf{1}\rp)\rp\|_2\rp)\\
&\geq
\inf_{\lambda \in \R}\lp( \|\lambda  \mathbf{1} +  h\|_2\rp)
\end{align*}
The latter quantity is the distance of $h$ from the subspace spanned by $\mathbf{1}$. Hence, the optimal $\lambda^*$ corresponds to the projection of $h$ onto this subspace. By standard calculations, we have that
$$
\lambda^* = - \frac{\langle h, \mathbf{1}\rangle}{\|\mathbf{1}\|^2} = - \frac{2\alpha n - n}{n} = 1 - 2\alpha
$$
and
$$
\inf_{\lambda \in \R}\lp( \|\lambda  \mathbf{1} +  h\|_2\rp) = 
\sqrt{\|h\|_2^2  - \frac{\langle h, \mathbf{1}\rangle^2}{\|\mathbf{1}\|^2}} =
\sqrt{4n(\alpha - \alpha^2)} = \sqrt{4n(1 - (2\alpha - 1)^2)}
$$
Therefore, we get that $\hat \theta$ is a
$\sqrt{\log n}/\sqrt{4n(1 - (2\alpha - 1)^2)S} $-consistent estimator. 
\end{proof}

\printbibliography

\end{document}